\documentclass{article}

\usepackage[preprint]{neurips_2026}

\usepackage[utf8]{inputenc} 
\usepackage[T1]{fontenc}    
\usepackage{hyperref}       
\usepackage{url}            
\usepackage{booktabs}       
\usepackage{amsfonts}       
\usepackage{nicefrac}       
\usepackage{microtype}      
\usepackage{xcolor}         

\usepackage{graphicx}
\usepackage{subcaption}
\usepackage{amsmath}
\usepackage{amssymb}
\usepackage{mathtools}
\usepackage{amsthm}
\usepackage{thmtools, thm-restate}
\usepackage{multirow}
\usepackage{bm}
\usepackage{soul}
\usepackage[shortlabels]{enumitem}
\usepackage{tikz}
\usetikzlibrary{calc,positioning}
\usepackage{babel}
\usepackage{algorithm,algorithmic}

\usepackage[capitalize,noabbrev]{cleveref}

\allowdisplaybreaks

\newcommand{\ab}{\mathbf{a}}

\newcommand{\eb}{\mathbf{e}}

\newcommand{\gb}{\mathbf{g}}

\newcommand{\ub}{\mathbf{u}}

\newcommand{\wb}{\mathbf{w}}
\newcommand{\xb}{\mathbf{x}}
\newcommand{\yb}{\mathbf{y}}
\newcommand{\zb}{\mathbf{z}}
\newcommand{\Ab}{\mathbf{A}}
\newcommand{\Bb}{\mathbf{B}}

\newcommand{\Db}{\mathbf{D}}

\newcommand{\Fb}{\mathbf{F}}

\newcommand{\Hb}{\mathbf{H}}
\newcommand{\Ib}{\mathbf{I}}

\newcommand{\Mb}{\mathbf{M}}

\newcommand{\Ob}{\mathbf{O}}

\newcommand{\Ub}{\mathbf{U}}
\newcommand{\Vb}{\mathbf{V}}

\newcommand{\Xb}{\mathbf{X}}

\newcommand{\Acal}{\mathcal{A}}

\newcommand{\Ccal}{\mathcal{C}}
\newcommand{\Dcal}{\mathcal{D}}

\newcommand{\Fcal}{\mathcal{F}}

\newcommand{\Mcal}{\mathcal{M}}

\newcommand{\Scal}{{\mathcal{S}}}

\newcommand{\Ucal}{\mathcal{U}}

\newcommand{\EE}{\mathbb{E}} 
\newcommand{\VV}{\mathbb{V}} 
\newcommand{\II}{\mathbb{I}} 
\newcommand{\PP}{\mathbb{P}} 
\newcommand{\RR}{\mathbb{R}} 
\newcommand{\Tr}{\mathop{\mathrm{tr}}}

\newcommand*{\Thetab}{\boldsymbol{\Theta}}
\newcommand*{\thetab}{\boldsymbol{\theta}}

\newcommand*{\etab}{\boldsymbol{\eta}}
\newcommand*{\xib}{\boldsymbol{\xi}}
\newcommand*{\gammab}{\boldsymbol{\gamma}}

\newcommand*{\vPhib}{\boldsymbol{\varPhi}}

\global\long\def\vPhib{\boldsymbol{\varPhi}}%
\global\long\def\argmax#1{\underset{#1}{\text{argmax }}}%
\global\long\def\abs#1{\left|#1\right|}%
\global\long\def\norm#1{\big\Vert #1\big\Vert }%
\global\long\def\Reg{\text{{R}}}%
\global\long\def\Tr{\text{Tr}}%
\global\long\def\Sgn{\text{sign}}%
\global\long\def\Ubrace#1#2{\underset{#2}{\underbrace{#1}}}%

\def\UCB{\texttt{VA-MNL}}

\theoremstyle{plain}
\newtheorem{theorem}{Theorem}
\newtheorem{proposition}[theorem]{Proposition}
\newtheorem{lemma}[theorem]{Lemma}
\newtheorem{corollary}[theorem]{Corollary}
\theoremstyle{definition}
\newtheorem{definition}{Definition}

\theoremstyle{remark}
\newtheorem{remark}[theorem]{Remark}

\title{Variance-Adaptive Optimal Algorithm for Reinforcement Learning with Multinomial Logit Function Approximation}

\author{%
  Wonyoung Kim \\
  Chung-Ang University \\
  \texttt{wyk7@cau.ac.kr} \\
  \And
  Min-Hwan Oh \\
  Seoul National University \\
  \And
  Garud Iyengar \\
  Columbia University \\
  \And
  Assaf Zeevi \\
  Columbia University \\
}

\begin{document}

\maketitle

\begin{abstract}
Reinforcement learning with multinomial logistic (MNL) function approximation has become an important framework due to its flexibility and broad applicability. 
While existing studies have established regret guarantees under worst-case analysis, they do not capture how performance depends on the variability of the interaction between the learner and the environment. 
In this paper, we develop a new theoretical analysis for MNL-based Markov decision processes that yields explicit variance-adaptive regret bounds.
Our algorithm is computationally efficient and achieves the instance-wise optimal rate of regret, narrowing the gap between upper and lower bounds. 
Our numerical experiments validate that our method learns optimal policies more efficiently than conventional approaches.
\end{abstract}

\section{Introduction}
\label{sec:Intro}

Reinforcement learning (RL) studies how an agent can learn to make sequential decisions by interacting with an environment so as to maximize cumulative rewards \citep{puterman2014markov, sutton2018reinforcement}.  
In the classical \emph{tabular} setting--- Markov decision processes (MDPs) with finite state and action spaces, decades of work have developed a rich theoretical understanding: algorithms achieve near-minimax optimal regret and sample-efficiency guarantees that scale with the number of states/actions and the horizon~\citep{kearns2002near,
brafman2002r,
azar2017minimax,dann2017unifying,
zhang2020almost, li2021breaking, zhang2024settling, lee2025minimax}.  
However, tabular methods become infeasible as the state--action space grows. 
\emph{feature-based} RL is a way to extend tabular RL to large-scale problems, where either the transition model and/or value functions are approximated in a low-dimensional function class. 
In particular, \textit{linear} function approximation~\citep{ayoub2020model,ishfaq2021randomized,zanette2020frequentist,jin2023provably} represents 
the transition model as linear combinations of known feature vectors, enabling theoretical tractability that scales with the feature dimension and generalization beyond visited states

However, the linear MDP paradigm can be fundamentally limited for modeling \emph{transition probabilities}. 
In many environments, the next-state distribution is required to lie on a probability simplex, whereas a generic linear parametrization does not naturally preserve non-negativity and normalization, and may therefore impose an unnatural structure on the dynamics \citep{hwang2023model,lee2024demystifying}.  
This motivates parametric models that retain the benefits of low-dimensional feature representations while respecting the geometry of probability distributions. To that end, the \emph{multinomial logistic} (MNL) MDP model has been proposed as a principled alternative~\citep{hwang2023model, cho2024randomized}: by parameterizing transition probabilities through a softmax/logistic structure, MNL-MDPs produce valid distributions by construction and can capture richer, non-linear dependence on features using a compact parameterization.

Despite its promises, the existing theoretical results are still dominated by \emph{worst-case} regret analyses. 
That is, existing bounds do not reflect the \emph{realized variance} induced by the actual interaction trajectory between an algorithm and a specific MDP instance \citep{hwang2023model,cho2024randomized,li2024provably}.  
From the viewpoint of modern learning theory, this is unsatisfying: refined, data-dependent guarantees—motivated by ideas such as local complexity and variance-sensitive error bounds—often provide a sharper characterization of performance when the observed sample path is “easy” \citep{bartlett2006local,xu2025towards}. Moreover, even at the level of worst-case regret, optimality in MNL-MDPs has not been fully settled. Table~\ref{tab:MNL-comparison} summarizes the state of the art and highlights the missing piece: regret bounds that are simultaneously (i) \emph{variance-adaptive} and (ii) \emph{optimal}.

\begin{table*}[t]
\caption{Regret bound comparison of MNL-based RL algorithms. Here, $d$ denotes the feature dimension, $H$ is the episode length, and $K$ is the total number of episodes. The parameter $\kappa$ represents the problem-dependent lower bound on the variance over the bounded parameter set as defined in~\eqref{eq:kappa}, while $\sigma_{k,h}^2$ denotes the realized variance along an interaction sample as defined in~\eqref{eq:sigma}.}
\label{tab:MNL-comparison}
\begin{center}
\begin{small}
\begin{tabular}{lllccc}
\toprule
& Paper & Regret Bound & Variance-Adaptive & Optimal  \\
\midrule
\multirow{5}{*}{\begin{tabular}[c]{@{}c@{}}Upper \\ Bound \end{tabular}} 
& \citet{hwang2023model} & $\tilde{O}(\kappa^{-1/2} dH^{3/2}\sqrt{K} )$ & No & No  \\
& \citet{cho2024randomized} & $\tilde{O}(dH^{2}\sqrt{K} )$ & No & No  \\
& \citet{li2024provably} & $\tilde{O}(dH^{2}\sqrt{K} )$ & No & No  \\
& \textbf{This work (Corollary~\ref{cor:regret_bound})} & $\tilde{O}(dH^{3/2}\sqrt{K})$ & No & \textbf{Yes}  \\
& \textbf{This work (Theorem~\ref{thm:regret_bound}}) & $\tilde{O}\left(dH\sqrt{ \sum_{k=1}^{K}\sum_{h=1}^{H} \sigma^2_{k,h}} \right)$ & \textbf{Yes} & \textbf{Yes}  \\
\midrule
\multirow{2}{*}{\begin{tabular}[c]{@{}c@{}}Lower \\ Bound\end{tabular}} 
& \citet{park2025infinite} & $\Omega(dH^{3/2}\sqrt{K})$ & No & - \\
& \textbf{This work (Theorem~\ref{thm:lower_bound})} & $\Omega \left(dH\sqrt{ \sum_{k=1}^{K}\sum_{h=1}^{H} \sigma^2_{k,h}} \right)$ & \textbf{Yes} & -  \\
\bottomrule
\end{tabular}
\end{small}
\end{center}
\vskip -0.1in
\end{table*}

In this work, we establish the first variance-adaptive regret upper and lower bounds for RL with MNL function approximation, significantly bridging the gap between pessimistic worst-case theory and variance-adaptive empirical performance. 
Our technical contribution is centered on a novel integration of the Online Newton Step (ONS) algorithm with an estimation framework to derive an exact ellipsoid confidence set for the parameter estimator. 
To achieve this, we introduce three key theoretical innovations: (i) a tighter regret analysis for ONS applied specifically to log-sum-exp functions; (ii) a new Hessian approximation technique that stabilizes curvature estimation in high-dimensional discrete spaces; and (iii) a refined second-order expansion of the confidence bound that captures local variance more accurately than traditional first-order methods. 
Collectively, these ingredients yield a variance-adaptive regret guarantee that reflects the intrinsic difficulty of the MDP instance, providing a sharper and more robust theoretical foundation for MNL-based reinforcement learning.
The main contribution of this paper is as follows:

\begin{itemize}
\item 
\textbf{Variance-adaptive optimal regret bound:} 
We establish the first variance-dependent regret upper bound (Theorem~\ref{thm:regret_bound}) for MNL-MDPs.
Unlike previous minimax analyses, our bound is \emph{adaptive to the variance on the realized sample path}, capturing the intrinsic difficulty of the MDP instance by adaptively scaling with the local variance and information geometry of the observed data sequence.

\item
\textbf{First optimal regret for for MNL-MDPs:}
We show that our established bounds are provably optimal.
To complement the upper bound, we prove a lower bound (Theorem~\ref{thm:lower_bound}) that matches the variance-dependent regret upper bound up to polylogarithmic factors.
We also provide the minimax upper bound that matches the lower bound for the first time in MNL-MDPs.

\item
\textbf{Efficient Hessian estimation and confidence ellipsoids:} On the technical side, we propose a new Hessian estimation technique that significantly improves upon existing exponential error dependencies (Lemma~\ref{lem:Hessian_bound} and Lemma~\ref{lem:Information_matrix_bound}). 
Furthermore, we develop a specialized estimation scheme that transforms online curvature updates into a \textit{self-normalized confidence ellipsoid} (Theorem~\ref{thm:self}), providing a rigorous theoretical foundation for parameter estimation, which can be of independent interest beyond MNL-MDPs.

\item 
\textbf{Variance-adaptive algorithm for MNL-MDP:} Our proposed algorithm uses \textit{second-order approximated upper confidence bound exploration strategy} to capture the realized variance and achieve variance-adaptive regret bound. 
This algorithm achieves the first regret guarantees in this setting that are simultaneously variance-adaptive and optimal up to logarithmic factors, effectively bridging the gap between pessimistic regret bound and empirical efficiency. Our numerical experiments show that the proposed algorithm is computationally efficient with superior performance.

\end{itemize}

\section{Related Works}
\label{sec:related_work}

Table~\ref{tab:MNL-comparison} summarizes the regret comparison for RL algorithms with MNL function approximation. \citet{hwang2023model} first established a provably efficient algorithm for this setting. Subsequent works by \citet{cho2024randomized} and \citet{li2024provably} provided $\kappa$-free regret bounds by employing randomized exploration and online mirror descent (OMD), respectively, where $\kappa$ is a problem-dependent constant that serves as a lower bound on the minimum eigenvalue of the Hessian across the feasible parameter space (Definition~\ref{def:kappa}).
However, these results are not instance-dependent and fail to adapt to the realized variance along the sample path. In contrast, our work is the first to establish regret bounds for MNL-MDPs that are both variance-adaptive and strictly tighter than existing theoretical guarantees.

While the MNL bandit model has been extensively studied within the context of assortment optimization and discrete choice modeling \citep{agrawal2017thompsonMNL, agrawal2019mnllbandit, oh2021multinomial, chen2024robust, wang2018near, perivier2022dynamic, lee2024improved}, it remains distinct from the MNL-MDP framework. Notably, neither is a special case of the other. In the MNL bandit setting, the agent selects an assortment of items where the resulting rewards (e.g., a customer's selection) follow a multinomial distribution. Conversely, in an MNL-MDP, the agent selects a single action that determines state transition probabilities via an MNL function.

\section{Multinomial Logit Markov Decision Process}
\label{sec:MNL_MDP}

We consider episodic Markov decision processes (MDPs) denoted by $\Mcal(\Scal,\Acal,H,\PP = \{\PP_{h}: h \in [H]\},r)$, where $\Scal$ is the state space, $\Acal$ is the action space, $H$ is the length of the episode, $K$ is the total number of episodes, $\{\PP_{h}: h \in [H]\}$ is the collection of state transition probability distributions for each $h \in [H]$, and $r$ is a reward function.
In each episode $k\in[K]$, the environment chooses an initial state $s_{k,1}\in\Scal$ and the state evolves for $H$ steps. At each step
$h\in[H]$ in every episode $k$, the decision maker observes the state $s_{k,h}\in\Scal$, selects an action $a_{k,h}\in\Acal$, and receives
an immediate reward $r(s_{k,h},a_{k,h})\in[0,1]$, and the next state is generated from the transition probability distribution
$\PP_{h}(\cdot|s_{k,h},a_{k,h})$. This process repeats until the end of the episode. 
A policy $\pi:\Scal\times[H]\to\Acal$ is a (potentially stochastic) function that determines which action the decision maker chooses in state $s$ and step $h\in[H]$ , i.e., $a \sim\pi(s,h)$.
The value $V_h^{\pi}(s)$ of policy $\pi$ in state $s \in \Scal$ and step $h\in [H]$ is defined as the expected sum of rewards under the policy $\pi$ until the end of the episode when starting from $s_{h}=s$, i.e.,  
\[
V_h^{\pi}(s)=\EE_{\pi}\big[\sum_{h^{\prime}=h}^{H}r(s_{h^{\prime}},a_{h^{\prime}})\big|s_{h}=s\big],
\]
where $\EE_{\pi}$ is the expectation over the trajectory generated by the policy-environment interaction measure.
The action-value function $Q_{h}(s,a;\pi)$ of the policy $\pi$ is defined as
\[
Q^{\pi}_{h}(s,a):=\EE_{\pi}\Big[\sum_{h^{\prime}=h}^{H}r\big(s_{h^{\prime}},a_{h^{\prime}}\big)\Big|s_{h}=s,a_{h}=a\Big].
\]
We define the optimal policy $\pi^{\star}:=(\pi_{1}^{\star},\ldots,\pi_{H}^{\star})$ to be a policy that achieves the highest possible value at each step $h\in[H]$ and for every state $s\in\Scal$. 
We denote the value of the optimal policy $V_{h}^{\star}(s):=V_{h}^{\pi^{\star}}(s)$ and the action-value function of the optimal policy $Q_{h}^{\star}(s,a):=Q^{\pi^{\star}}_{h}(s,a)$, respectively. 
Thus, $V_{h}^{\pi^\star}(s) = \max_{a \in \Acal} Q_h^{\pi^\star}(s,a)$.
For simplicity, we define $\PP_{h}V_{h+1}(s,a):=\EE_{s^{\prime}\sim\PP_{h}(\cdot|s,a)}[V_{h+1}(s^{\prime})]$.
Recall that the action-value functions satisfy the Bellman recursion:
\[
Q_{h}^{\pi}(s,a) 
= r(s,a) 
+ \PP_{h} V^{\pi}_{h+1}(s,a), \quad
Q_{h}^{\star}(s,a) 
= r(s,a) 
+ \PP_{h} V_{h+1}^{\star}(s,a).
\]
where $V_{H+1}^{\pi}(s)=V_{H+1}^{\star}(s)=0$.
Let $\Pi_K = \big\{\pi_{k}:k\in[K]\big\}$ denote a collection of policies, where $\pi_k$ is the policy of the decision maker for the $k\in[K]$-th episode. 
Then the cumulative regret of $\Pi_K$ for MDP $\Mcal:=\Mcal(\Scal,\Acal,H,\PP,r)$ is defined as
\[
\Reg_{\Pi_K}(\Mcal) \!:=\!\sum_{k=1}^{K} V_{1}^{\star}(s_{k,1}) \!-\! V^{\pi_k}_{1}(s_{k,1}).
\]
The goal of the decision maker is to choose $\Pi_K$ that minimizes regret. Note that this is equivalent to choosing $\Pi_K$ that maximizes cumulative rewards over $K$ episodes.
We assume that the transition kernel of the MDP $\Mcal(\Scal,\Acal,H,\PP,r)$ is given by an MNL model.
\begin{definition}[Multinomial Logit MDP] 
Let $\eb_s$ denote the $s$-th Euclidean basis. 
The state transition function when an action $a_h\in\Acal$ is taken at the state $s_h$ is given by
\begin{equation}
\PP_h\!(s_{h+1} \!\!\mid\!\! s_h, a_h, \thetab_h^\star) 
\!=\! 
\frac{
  \exp\!\big(
    \eb_{s_{h+1}}^\top 
    \vPhib(s_h, a_h) 
    \thetab_h^\star
  \big)
}{\underset{\tilde{s} \in \Scal_h\!(s_h, a_h)}{\sum}\!\!
  \exp\!\big(
    \eb_{\tilde{s}}^\top\! 
    \vPhib(s_h, a_h) 
    \thetab_h^\star
  \big)
},
\label{eq:transit_kernel}
\end{equation}
where $\Scal_h(s,a)$ is the set of states that are reachable from $s\in \Scal$ in step $h\in[H]$ using action $a \in \Acal$. 
The feature map $\vPhib:\Scal\times\Acal\to\RR^{|\Scal|\times d}$  is known, and the transition parameter $\Thetab^{\star} = [\thetab_{h}^{\star}]_{h \in [H]}$ is unknown. 
The feature map $\vPhib$ and $\thetab^{\star}$ are bounded: $\|\vPhib(s,a)^\top\eb_{s^{\prime}}\|_{2}\le B_{\vPhib}$ for all $s', s\in\Scal$, and $a\in\Acal$, and $\max_{h\in[H]}\|\thetab_h^{\star}\|_{2}\le B_{\thetab}$.
\end{definition}
Note that each step $h$ has a different parameter $\thetab^{\star}_h$, which makes the problem harder to learn than the homogeneous transition probability studied in \citet{hwang2023model}.
While the size of the state space $|\Scal|$ may be exponentially large, it does not affect the regret bound.
The cardinality of the reachable set $\Scal_h(s,a)$ also does not affect the regret bound but affects the computational complexity for estimating the transition kernel.
According to \citet{hwang2023model}, there are many examples where the cardinality of the reachable set is often much smaller than the full state space.

MNL-MDPs model sequential decision problems in which a single action leads to one of several reachable next states, and the relative transition probabilities among those states follow a MNL structure. 
For example, in recommendation or online engagement systems, one displayed policy/action may lead a user to several possible future states (e.g., click, continue browsing, switch category, or exit), and the long-term objective depends on how these state transitions unfold over time.

\section{Multinomial Logit Parameter Estimation Procedure}
\label{sec:proposed_method}
In this section, we present our proposed online estimation method for the MNL-MDP. 
We first propose an online algorithm for maximizing the likelihood of a collection of multinomial random variables (Section~\ref{sec:ONS}).
Next, we show how to convert a regret bound of the negative log-likelihood into a confidence region for the parameter $\Thetab^{\star}$ that is compatible with the UCB-style algorithms (Section~\ref{subsec:online_to_confidence}).

\begin{algorithm}[t]
\caption{Online-to-Confidence-Ellipsoid Estimator Update (\texttt{OCEE})}
\label{alg:ONS-MST}
\begin{algorithmic}[1]
\REQUIRE online estimator $\boldsymbol{\theta}$, information matrix $\mathbf{H}$, learning rate $\eta$, transition sample $(s,a,s^{\prime})$, parameter bound $B_{\boldsymbol{\theta}}$, moment vector $\gammab$.

\STATE Compute the gradient and update the information matrix:
\[
\gb = \boldsymbol{\Phi}(s,a)^{\top}\big(-\mathbf{e}_{s^{\prime}} + \nabla L_{\mathcal{S}(s,a)}(\boldsymbol{\Phi}(s,a)\boldsymbol{\theta})\big), \quad
\mathbf{H} \gets \mathbf{H} + \gb \gb^{\top}
\]

\STATE \textbf{Newton step}: 
$\tilde{\boldsymbol{\theta}} \gets \boldsymbol{\theta} - \eta \mathbf{H}^{-1} \gb$

\STATE \textbf{Projection step}:
$\thetab \gets \text{argmin}_{\|\boldsymbol{\theta}\|_2 \le B_{\boldsymbol{\theta}}} \|\boldsymbol{\theta} - \tilde{\boldsymbol{\theta}}\|_{\mathbf{H}}^2$

\STATE Update the moment vector: $\gammab \gets \gammab +\gb \gb^\top \thetab$

\STATE Compute the estimator: $\widehat{\thetab}=\Hb^{-1}\gammab$

\ENSURE Estimator $\widehat{\boldsymbol{\theta}}$, updated online estimator $\thetab$, updated Hessian $\mathbf{H}$, updated moment vector $\gammab$.

\end{algorithmic}
\end{algorithm}

\subsection{Tight Regret Bound for Online Newton Step}
\label{sec:ONS}
Let $L_{\Scal_h(s,a)}(\etab):=\log\sum_{\bar{s}\in\Scal_{h}(s,a)}\exp(\eb^\top_{\bar{s}}\etab)$  denote the log-sum function.
For step $h\in[H]$ in episode $k\in[K]$, the negative log-likelihood is
\begin{equation}
\begin{aligned}
\ell_{k,h}(\thetab) 
&:= -\log \PP_h(s_{k,h+1} \mid s_{k,h}, a_{k,h}, \thetab) \\
&= -\eb_{s_{k,h+1}}^{\top} \vPhib(s_{k,h}, a_{k,h}) \thetab  + L_{\Scal_h(s_{k,h}, a_{k,h})} \big(\vPhib(s_{k,h}, a_{k,h}) \thetab \big).
\end{aligned}
\label{eq:log_likelihood}  
\end{equation}
For each $h\in[H]$, let $\thetab_{0,h}:=\boldsymbol{0}$ and given $\thetab_{k-1,h}$ we define
\begin{equation}
\thetab_{k,h} := \thetab_{k-1,h} - \eta \widehat{\Hb}_{k,h}^{-1} \nabla \ell_{k,h}(\thetab_{k-1,h})
\label{eq:ONS}
\end{equation}
the Online Newton Step (ONS) estimate of $\thetab_h^\star$, where
\begin{equation}
\widehat{\Hb}_{k,h}:=\sum_{\nu=1}^{k}\nabla\ell_{\nu,h}(\thetab_{\nu-1,h})\nabla\ell_{\nu,h}(\thetab_{\nu-1,h})^{\top}+\epsilon\Ib_{d}
\label{eq:Hessian}
\end{equation}
denote the \emph{observed information matrix} over the $k$ transitions from step $h$ to $h+1$ with transition samples $\{(s_{\nu,h},a_{\nu,h},s_{\nu,h+1})\}_{\nu=1}^{k}$ and initial parameter $\epsilon>0$.
Our analysis utilizes the Fisher Information Identity:  
\[
  \mathbb{E}[\nabla\ell_{\nu,h}(\theta^{\star}_{h})\nabla\ell_{\nu,h}(\theta^{\star}_{h})^\top \mid \mathcal{F}_{\nu,h}] = \nabla^2\ell_{\nu,h}(\theta^{\star}_h),
\]
where $\Fcal_{\nu,h}$ is the sigma-algebra generated by
$\cup_{h^{\prime}=1}^{H}\cup_{\nu^{\prime}=1}^{\nu-1}\{s_{\nu^{\prime},h^{\prime}},a_{\nu^{\prime},h^{\prime}}\}$
and $\cup_{h^{\prime}=1}^{h}\{s_{\nu,h^{\prime}},a_{\nu,h^{\prime}}\}$.
This identity allows us to relate the second-order curvature directly to the variance of the gradients at the true parameter $\theta^{\star}_h$.

If the gradient $\nabla\ell_{\nu,h}(\theta^{\star}_{h})$ is approximated by the gradient at the current iterate $\theta_{\nu,h}$, the resulting error is quadratic in the difference $\|\nabla\ell_{\nu,h}(\theta_{\nu-1,h}) - \nabla\ell_{\nu,h}(\theta_h^{\star})\|$. 
\begin{restatable}[Quadratic Error Approximation for the True Hessian]{lemma}{QuadError}
\label{lem:exp_to_quad}
Let $\gb_{\nu,h}:=\nabla\ell_{\nu,h}(\thetab_{\nu-1,h})$. 
Then,
\begin{equation}
 \EE\big[ \gb_{\nu,h} \gb_{\nu,h}^{\top} \,\big|\, \Fcal_{\nu,h} \big]
- \nabla^2 \ell_{\nu,h}(\thetab_h^{\star}) \!=\! 
\big(\nabla\ell_{\nu,h}(\thetab_{\nu-1,h})\!-\!\nabla\ell_{\nu,h}(\thetab_h^{\star})\big)\big(\nabla\ell_{\nu,h}(\thetab_{\nu-1,h})\!-\!\nabla\ell_{\nu,h}(\thetab_h^{\star})\big)^{\top}.
\label{eq:diff}
\end{equation}
\end{restatable}
The cumulative rank-one update of $\gb_{\nu,h}\gb_{\nu,h}^\top$ effectively captures the curvature in the direction of the gradient. 
As these gradients accumulate over iterations, the resulting matrix provides a robust approximation of the second-order structure without relying on global Hessian stability constants.

\subsection{Online to Confidence Set Conversion}
\label{subsec:online_to_confidence}

While the ONS estimate $\thetab_{k,h}$ method is computationally efficient and gives a reasonable estimator, the error bound of the estimator is not guaranteed, i.e., the confidence ellipsoid is not established.
Instead of using the online estimate $\thetab_{\nu,h}$ we incorporate all online estimates to the proposed estimator that achieves an error bound normalized by the Hessian matrix.
Given a regularization parameter $\epsilon>0$, we define our estimator $\widehat{\thetab}_{k,h}$ for $\thetab^{\star}_h$ in episode $k$ as the solution of the optimization problem:
\begin{equation}
\widehat{\thetab}_{k,h} := \underset{\thetab_{h} \in \RR^{d}}{\text{argmin}}
\bigg\{ \!\!
  \sum_{\nu=1}^{k} 
  \big(
    \gb_{\nu,h}^{\top}\!
    (\thetab_{\nu-1,h} - \thetab_{h}) 
  \big)^{2}
  \!+\! \epsilon \|\thetab_h\|^2 \!
\bigg\}
\Rightarrow 
\widehat{\thetab}_{k,h}=\widehat{\Hb}^{-1}_{k,h}
  \big(\sum_{\nu=1}^{k}
  \gb_{\nu,h}
  \gb_{\nu,h}^{\top}
  \!\thetab_{\nu-1,h}\big).
\label{eq:est}
\end{equation}
This estimator can be updated in a constant time with rank-1 update.
The following theorem presents the confidence ellipsoid of the proposed estimator $\widehat{\thetab}_{k,h}$.

\begin{restatable}[Self-Normalized Bound for the Estimator]{theorem}{self}
Define $C_{\vPhib,\thetab}:=(e-1)(6+8B_{\vPhib}B_{\thetab}+2B_{\vPhib}^{2}B_{\thetab}^{2})$.
Set the learning rate $\eta_{\vPhib,\thetab}:=(e-1)(3+4B_{\vPhib}^{2}B_{\thetab}^{2})$.
For the estimator $\widehat{\thetab}_{k,h}$ defined in \eqref{eq:est}, with probability at least $1-3\delta$,
\(
\norm{\widehat{\thetab}_{k,h}-\thetab_{h}^{\star}}_{\widehat{\Hb}_{k,h}}\le\beta_{k},
\)
for all $k\in[K]$ where
$\beta_{k}=\sqrt{\epsilon}B_{\thetab}+\sqrt{\epsilon B_{\thetab}^{2}+4\gamma_{k}}$ 
and 
\(\gamma_k = \tilde{O}(d B_{\vPhib}^4 B_{\thetab}^4)\).
\label{thm:self}
\end{restatable}

The proof of Theorem~\ref{thm:self} is in Appendix~\ref{sec:proof_self}.
The rate of radius $\beta_k$ is $O(B_{\thetab}^2 B_{\vPhib}^2 \sqrt{d \log dk})$.
Our self-normalized bound involves the observed information matrix $\widehat{\Hb}_{k,h}$ instead of the Hessian estimate $\tilde{\Hb}_{k,h}:=\sum_{\nu=1}^{k}\nabla^{2}\ell_{\nu,h}(\thetab_{\nu-1,h})+\epsilon\Ib_{d}$.
\citet{zhang2024online} proposed an online mirror descent algorithm
that establishes a similar self-normalized bound. 
Their self-normalized bound is normalized by the Hessian estimate
$\tilde{\Hb}_{k,h}$, which involves the second order derivative.

Following the standard self-concordant analysis (see e.g., \citet{sun2019generalized}), there exists a constant $C>0$ such that
\begin{equation}
\nabla^{2}\ell_{\nu,h}(\thetab_{\nu-1,h})\succeq 
\exp\left(-C
  \left\|
    \vPhib(s_{\nu,h},a_{\nu,h})
    (\thetab_{\nu-1,h} - \thetab_{h}^{\star})
  \right\|_{2}
\right) \nabla^2 \ell_{\nu,h}(\thetab_h^{\star}).
\label{eq:exp_ratio}
\end{equation}
Therefore, it suffers from the exponential difference $\exp{B_{\vPhib}B_{\thetab}}$ when obtaining the self-normalized bound that involves the true Hessian, $\Hb^{\star}_{k,h}:=\sum_{\nu=1}^{k} \nabla^2\ell_{\nu,h}(\thetab^{\star})+B^2_{\vPhib} \Ib_d$.
In contrast, our proposed Hessian estimate
$\widehat{\Hb}_{k,h}
=\sum_{\nu=1}^{k}\nabla\ell_{\nu,h}(\thetab_{\nu-1,h})
\nabla\ell_{\nu,h}(\thetab_{\nu-1,h})^\top +\epsilon\Ib_{d}$ yields the self-normalized bound with $\Hb_{k,h}^{\star}$.
The following lemma shows that we can approximate $\Hb^{\star}_{k,h}$ with our proposed information matrix $\widehat{\Hb}_{k,h}$ and without using the second order derivatives of the log-likelihood.

\begin{restatable}[Information Matrix Dominates the True Hessian]{lemma}{HessianBound}
\label{lem:Hessian_bound} 
For any step $h \in [H]$, if the regularization parameter is set to the required value $\epsilon = B_{\boldsymbol{\Phi}}^{2} + 4B_{\boldsymbol{\Phi}}^{2}\log(d/\delta)$, then with probability at least $1-\delta$, we have $\widehat{\mathbf{H}}_{k,h} \succeq (1-e^{-1})\mathbf{H}_{k,h}^{\star}$ for all $k \in [K]$
\end{restatable}
Lemma \ref{lem:Hessian_bound} shows $\widehat{\Hb}_{k,h}$ dominates $\Hb^{\star}_{k,h}$ and directly implies 
$\|\widehat{\thetab}_{k,h}-\thetab_h^{\star}\|_{\Hb^{\star}_{k,h}} = O(\sqrt{d\log dk})$.
Thus, our estimator has a $\Hb_{k,h}^{\star}$-normalized bound without involving the exponential approximation of the second derivative in~\eqref{eq:exp_ratio}.
This paves the way for the improvement over the previous regret upper bounds by the factor of at most $\exp(B_{\vPhib}B_{\thetab})$

\section{Variance-Adaptive MNL-MDP Algorithm}
\label{sec:alg}
With our proposed estimator $\widehat{\thetab}_{k,h}$, let
\(
\Ccal_{k,h} := \big\{ \thetab \in \RR^d :
\|\thetab - \widehat{\thetab}_{k,h}\|^2_{\widehat{\Hb}_{k,h}} \le \beta_{k} \big\},
\)
denote the confidence ellipsoid.
The optimism in the face of uncertainty (OFU) principle requires the decision maker to compute
\begin{equation}
\argmax{\thetab \in \Ccal_{k-1,h}} \!
 \EE\left[\widehat{V}_{k,h+1}(s^{\prime})\big| s, a, \thetab \right]
\!=\! \argmax{\thetab \in \Ccal_{k-1,h}}\!\! \bigg\{
\frac{\sum_{s^{\prime} \in \Scal_{h}(s,a)} \!\widehat{V}_{k,h+1}(s^{\prime}) 
\exp\big( \eb_{s^\prime}^\top\vPhib(s,a)\thetab \big)}
{\sum_{s^{\prime} \in \Scal_{h}(s,a)} \!
\exp\big( \eb_{s^\prime}^\top\vPhib(s,a) \thetab \big)}
\!\bigg\},
\label{eq:OFUL-exactopt}
\end{equation}
where $\widehat{V}_{k,h+1}(s)=\max_{a\in\Acal}\widehat{Q}_{k,h}(s,a)$ is the estimated value function. 
However, computing the solution for \eqref{eq:OFUL-exactopt} is intractable. 
\citet{hwang2023model} use the upper bound:
\[
\EE[\widehat{V}_{k,h+1}(s^{\prime}) \!\!\mid\!\! s, a, \thetab]
\le 
\EE[\widehat{V}_{k,h+1}(s^{\prime}) \!\!\mid\!\! s, a, \widehat{\thetab}]
+ 2H\kappa^{-1/2} \norm{\widehat{\thetab} - \thetab_{h}^{\star}}_{\tilde{\Ab}_{k-1,h}}\! 
\max_{s^{\prime} \in \Scal_{h}(s,a)} \|\eb_{s^{\prime}}^\top \vPhib(s,a)\|_{\tilde{\Ab}_{k-1,h}^{-1}}.
\]
where $\kappa$ is the lower bound of the Hessian defined in~\eqref{eq:kappa} and $\tilde{\Ab}_{k-1,h}\!:=\!\sum_{\nu=1}^{k-1}\vPhib(s_{\nu,h},a_{\nu,h})^{\top}\vPhib(s_{\nu,h},a_{\nu,h})+\Ib_{d}$
is the Gram matrix without the variance term. 
Note that the $\kappa^{-1/2}$ term propagates to the cumulative regret bound.
\citet{cho2024randomized} and \citet{li2024provably} refined this approximation in terms of Hessian $\tilde{\Hb}_{k,h}$ to remove $\sqrt{\kappa}$ in the leading term of the regret bound.
However, they use the first-order Taylor expansion and this overlooks the realized variance of the state-transition and the interaction.

In contrast, we approximate the expected value function with the second-order Taylor expansion. 
Recall that $L_{\Scal_h(s,a)}(\etab):=\log\sum_{\bar{s}\in\Scal_{h}(s,a)}\exp(\eb^\top_{\bar{s}}\etab)$ is the log-sum function and let $\Vb_{k,h}(s,a):=\nabla^2L_{\Scal_h(s,a)}(\vPhib(s,a)\widehat{\thetab}_{k-1,h})$.
Let $P_H(x):=\max\big\{0,\min\{x,H\}\big\}$ denote the projection function onto $[0,H]$.
The proposed second-order upper bound for $Q$-value is
\begin{equation}
\begin{aligned}
\widehat{Q}_{k,h}(s,a) := P_H\Big(&
  r(s,a) + \!\!\!\sum_{s^{\prime} \in \Scal_{k,h}(s,a)}\! \widehat{V}_{k,h+1}(s^{\prime}) \, \PP(\eb_{s^{\prime}} \!\mid s,a,\widehat{\thetab}_{k-1,h})\\
  &+ \!\beta_{k-1} \Big\| \!\!\!\sum_{s^{\prime} \in \Scal_{k,h}(s,a)} \widehat{V}_{k,h+1}(s^{\prime}) \eb_{s^{\prime}}^{\top} \Vb_{k,h}(s,a) \vPhib(s,a) \Big\|_{\widehat{\Hb}_{k-1,h}^{-1}} \\
  &+ \!\beta_{k-1}^{2}\! \max_{s \in \Scal_{k,h}(s,a)} \widehat{V}_{k,h+1}(s) \left\| \eb_{s^{\prime}}^{\top} \vPhib(s,a) \right\|_{\widehat{\Hb}_{k-1,h}^{-1}}^{2}\Big).
\label{eq:Q_hat}
\end{aligned}
\end{equation}
In Lemma~\ref{lem:optimism}, we establish that the $\widehat{Q}_h$ is an optimistic $Q$ function with high probability.

\begin{restatable}[Optimism of the Second Order UCB]{lemma}{Optimism}
\label{lem:optimism} For any $(s,a)\in\Scal\times\Acal$, with probability at least $1-3\delta$,
\(
Q_{\star,h}(s,a)\le\widehat{Q}_{k,h}(s,a),
\)
for all $k\in[K]$ and $h\in[H]$.
\end{restatable}


\begin{algorithm}[t]
\caption{Variance adaptive algorithm for MNL-MDP (\UCB)}
\label{alg:rein}
\begin{algorithmic}[1]
\REQUIRE Confidence level $\delta \in (0,1)$, Initial observed information matrix $\mathbf{H}_{0,h} = B_{\boldsymbol{\Phi}}^{2} (1+ 4 \log\frac{d}{\delta}) \mathbf{I}_d$, Initial online estimate $\thetab_{0,h}=\boldsymbol{0}$, 
Learning rate $\eta_{\vPhib,\thetab}:=(e-1)(3+4B_{\vPhib}^{2}B_{\thetab}^{2})$, Initial moment vector $\gammab_{0,h}=\boldsymbol{0}$.

\FOR{$k = 1$ to $K$}
    \FOR{$h = H$ to $1$}
        \STATE Compute $\widehat{Q}_{k,h}(s, a)$ as in~\eqref{eq:Q_hat} for all $s$ and $a$.
    \ENDFOR
    \FOR{$h=1$ to $H$}
    \STATE Select action $a_{k,h} = \text{argmax}_{a \in \mathcal{A}} \widehat{Q}_{k,h}(s_{k,h}, a)$
    \STATE Observe $s_{k,h+1}$ and update parameters: 
    \begin{align*}
    \widehat{\thetab}_{k,h}, \thetab_{k,h}, \Hb_{k,h}, \gammab_{k,h} \gets \texttt{OCEE}(\thetab_{k-1,h},\Hb_{k-1,h},\eta_{\vPhib,\thetab}, (s_{k,h},a_{k,h},s_{k,h+1}), B_{\thetab}, \gammab_{k-1,h})
    \end{align*}
    \ENDFOR
\ENDFOR
\end{algorithmic}
\end{algorithm}

Algorithm~\ref{alg:rein}, termed the \emph{Variance-Adaptive algorithm for MNL-MDP} (\UCB\ ), combines the online-to-confidence set conversion and a second-order approximate upper confidence bound.
At the start of each episode $k$, the algorithm maintains for every stage $h$ a parameter estimate $\widehat{\thetab}_{k-1,h}$ together with the associated observed information matrix ${\Hb}_{k-1,h}$, which defines a confidence ellipsoid. 
The $Q$-value estimates $\widehat{Q}_{k,h}$ are computed by incorporating both the current estimate and an exploration bonus derived from the confidence set and the second-order Taylor expansion. 
At each decision point $h$, the action $a_{k,h}$ is selected as the maximizer of $\widehat{Q}_{k,h}$. 
After observing the next state $s_{k,h+1}$ and the feature vector $\vPhib(s_{k,h},a_{k,h})$, the online estimate $\thetab_{k,h}$ and the observed information matrix $\Hb_{k,h}$ are updated with the outer product of the stochastic gradient as in Algorithm~\ref{alg:ONS-MST}
The algorithm updates $\widehat{\thetab}_{k,h}$ as in~\eqref{eq:est} by integrating the previous online estimates $\{\thetab_{\nu,h}:\nu\in[k]\}$.

The computational time per step arises from computing the stochastic gradient, updating the inverse of $\Hb_{k-1,h}$ using the Sherman--Morrison rank-one formula, and projecting the updated parameter back into the feasible set. 
The gradient computation scales linearly with $d$, whereas the Hessian inverse update and projection each require $O(d^2)$ operations. 
In addition, the Q-value computation involves a log-sum-exp evaluation over the set of reachable states $\mathcal{S}_h(s,a)$, which costs $O(|\mathcal{S}_h(s,a)|\,d)$ per state--action pair. 
When the number of reachable states is constant, the total cost of processing one decision point is $O(d^2)$.

\section{Regret Analysis}
\label{sec:regret_analysis}
We define a problem-dependent constant that significantly affects the regret bound.
\begin{definition}[Fisher Information Lower Bound]
\label{def:kappa}
Recall that $L_{\Scal_h(s,a)}(\etab):=\log\sum_{\bar{s}\in\Scal_{h}(s,a)}\exp(\eb^\top_{\bar{s}}\etab)$ is the log-sum-exp function.
We define the Fisher information lower bound as $\kappa>0$ such that
\begin{equation}
\inf_{\thetab:\|\thetab\|_{2}\le
B_{\thetab}}\lambda_{\min}\Big(\nabla^{2}L_{\Scal_h(s,a)}\big(\vPhib(s,a)\thetab\big)\Big)\ge\kappa,
\label{eq:kappa}
\end{equation}
for all $s \in \Scal$ and $a \in \Acal$. 
\end{definition}

Because $L_{\Scal_h(s,a)}$ is strictly convex, the Fisher information lower bound $\kappa$ is always positive.
Definition~\ref{def:kappa} is essential in learning the multinomial model \citep{jezequel2021mixability,oh2021multinomial,oh2019thompson,amani2021ucb,lee2024nearly,zhang2024online,perivier2022dynamic}.
The constant $\kappa^{-1}$ can be exponential in $B_{\vPhib}$ and
$B_{\thetab}$~\citep{faury2020improved}; eliminating the $\kappa^{-1}$ in the dominating order of the regret bound has been a crucial task in MNL-MDP literature \citep{cho2024randomized,li2024provably}.

We define the variance over a specific action-state pair
\begin{equation}
\sigma^2(s,a):=\max_{\xb:\|\xb\|_{\infty}\le1}\xb^\top\nabla^2 L_{\Scal_h(s,a)}(\vPhib(s,a)\thetab_h^\star)\xb
\label{eq:sigma}
\end{equation}
as the maximum quadratic value of the Hessian over $\ell_{\infty}$-ball.
It is easy to show that the range of $\sigma^2(s,a)$ is in $[4\kappa,1]$.
The maximum over the hypercube is achieved at the endpoints $\{-1,1\}^{\Scal_h(s,a)}$. 
Then, with simple algebra,
\(
\sigma^2(s,a) = 1- 4\min_{\tilde{\Scal}\subseteq\Scal_h(s,a)} \big(1/2 - \PP(s_{h+1}\in\tilde{\Scal}\mid s,a)\big)^2
\)
For given $(s,a)$, if there is a subset $\tilde{\Scal}\subseteq\Scal_{h}(s,a)$ such that $\PP(s_{h+1}\in\tilde{\Scal})=1/2$,
then $\sigma^2(s,a)=1$, where the variance reaches the maximum. 
On the other hand, there are other cases with $\sigma^2(s,a)=O(\kappa)$.
For example, in the logistic case $\Scal_h(s,a)=\{0,1\}$, if $\PP(s_{h+1}=0|s,a)=\frac{1-\sqrt{1-4\kappa}}{2}$ then $\sigma^2(s,a)=4\kappa$.
In this case, replacing $HK$ with $\sum_{k=1}^{K}\sum_{h=1}^{H}\sigma^2(s_{k,h},a_{k,h})$ in the regret bound improves by a factor of $\kappa^{-1/2}$, or equivalently by $\exp(B_{\vPhib}B_{\thetab}/2)$ factor.

\begin{theorem}[Regret Upper Bound of \UCB\ ]
\label{thm:regret_bound}
Let $a_{k,h}$ denote the action selected by \UCB\ at step $h\in[H]$ in episode $k\in[K]$ and $s_{k,h}$ denote the realized state.
Let $\sigma^2_{k,h}:=\sigma^2(s_{k,h},a_{k,h})$.
If the algorithm $\Pi_K$ is set to \UCB\, then for any MNL-MDP $\Mcal$, with probability at least $1-5\delta$, 
\[
\Reg_{\Pi_K}(\Mcal_{\Thetab^{\star}}) = O\bigg( dH\log K\!\!\sqrt{\sum_{k=1}^{K}\sum_{h=1}^{H}\sigma_{k,h}^2 \log\frac{KH}{\delta}}\!  +\!\kappa^{-1}d^{\frac{5}{2}}H^{\frac{5}{2}}\log^{\frac{3}{2}}\! K\bigg).
\]
\end{theorem}

The explicit expression of the regret bound is presented in Section~\ref{subsec:regret_proof}.
Note that the variances $\sigma_{k,h}^2$ are random variables which depend on the action chosen by the algorithm and states evolving accordingly and involves the parameters $\{\thetab^{\star}_{h}:h\in[H]\}$; thus, our regret bound is variance-adaptive, that is the regret bound depends on the state-action transition interaction samples.
This is the first regret upper bound that depends on the variance of the multinomial transition.

When $K =\Omega (\kappa^{-3} d^3 H^2 \log K / \log\frac{KH}{\delta})$ is sufficiently large, $\kappa^{-1} \ll \sum_{k,h}\sigma^2_{k,h}$, and the $\tilde{O}\big(\kappa^{-1}d^{5/2}H^{5/2}\log^{3/2}K\big)$ term is dominated by $\tilde{O}(dH\sqrt{\sum_{k,h}\sigma^2_{k,h}})$.
The best known regret bound for MNL-MDP is $\tilde{O}(dH^{3/2}\sqrt{HK}+\kappa^{-1}d^{2}H^{2})$ \citep{cho2024randomized}. 
Our regret bound has the sum of the variance $\sum_{k=1}^{K}\sum_{h=1}^{H}\sigma^{2}_{k,h}$ of the multinomial distribution, and improves by at most a factor of $\kappa^{-1/2}$.
Further, our upper bound is optimal in that it matches the lower bound in Theorem~\ref{thm:lower_bound} up to logarithmic factors.
Because $\sigma_{k,h}^2\le 1$, Theorem~\ref{thm:regret_bound} directly implies $O(dH^{3/2} \sqrt{K} \log KH)$ regret bound (Corollary~\ref{cor:regret_bound}).
Further, by developing a novel bound for the sum of variance of the UCB $\widehat{Q}_{k,h}$ (Lemma~\ref{lem:variance_bound}), we establish an improved regret bound by the factor of $\sqrt{H}$ in the main order term. 

Deferring the standard computations to Appendix~\ref{subsec:regret_proof}, we provide a novel optimistic law of total variance that enables us to derive the variance-adaptive optimal regret bound. 

\begin{restatable}[The Optimistic Law of Total Variance]{lemma}{TotalVariance}
\label{lem:variance_bound} The sum of variance of the value function $\hat{V}_{k,h+1}$
\[
\begin{aligned}
&\sum_{k=1}^{K}\sum_{h=1}^{H}\VV\Big[\widehat{V}_{k,h+1}(s_{k,h+1})\Big|s_{k,h},a_{k,h}\Big]\\
&\le
12eH\sum_{k=1}^{K}\sum_{h=1}^{H}\sigma_{k,h}^2+ 1200eH^{4}\kappa^{-2}\beta_{K-1}^{4} \big(d\log \frac{K+1}{d} + 3\log\frac{1}{\delta}\big).
\end{aligned}
\]
\end{restatable}
Lemma \ref{lem:variance_bound} paves the way to improving the $\sqrt{H}$ term in the regret bound to achieve the optimal regret rate. 
This is the first variance bound for the \textit{estimated} value function $\widehat{V}_{k,h}$; previous results obtain a bound for the true value function $V_{h}^{\star}$.


We provide the variance-dependent lower bound for the MNL-MDP that matches our upper bound with its proof in Appendix~\ref{subsec:lower_bound_proof}.
\begin{restatable}[Variance-Adaptive Lower Bound for MNL-MDP]{theorem}{LowerBound}
\label{thm:lower_bound} 
Assume that $d\ge2$, $H\ge 4$, and $K\ge \max\{144H(d-1)^2,\kappa^{-2}H^{-1}\}$. 
Given any parameter $\Thetab^{\star}=(\thetab_1^{\star},\ldots,\thetab_{H}^{\star}) \in \RR^{d \times H}$, there exists a MNL-MDP instance $\Mcal_{\Thetab^{\star}}:=\Mcal(\Scal,\Acal,H,\Thetab^{\star},\vPhib,r)$ such that with probability at least $1/e$, the regret of any algorithm $\Pi_K$,
\[
\Reg_{\Pi_K}\!(\Mcal_{\Thetab^{\star}}) \!\ge\!\frac{2Hd\sqrt{\sum_{h=1}^{H}\!\sum_{k=1}^{K}\!\sigma^2(s_{k,h},a_{k,h})}}{17},
\]
where $\{s_{k,h},a_{k,h}:k\in[K],h\in[H]\}$ denote the state and action sample generated by the interaction between $\Pi_K$ and $\Mcal_{\Thetab^{\star}}$.
\end{restatable}

Theorem~\ref{thm:lower_bound} implies that no algorithm can achieve better rate than the specified term with high probability as well as in expectation.
This is the first variance-adaptive lower bound for the MNL-MDP for any given parameter $\Thetab^{\star}$.
The lower bound in Theorem~\ref{thm:lower_bound} involves the sum of variances $\sum_{h=1}^{H}\sum_{k=1}^{K} \sigma_{k,h}^2$ that depends on the choice of the algorithm $\Pi_K$.
This captures that the difficulty of the problem arises from the variance of the interaction between the MDP and the algorithm.

\citet{li2024provably} proved an $\Omega(dH\sqrt{\kappa K} )$ lower bound for the MNL-MDP by reducing the instance into $H/2$ independent logistic bandits and applied the lower bound in \citet{abeille2021instance}.
Recently, \citet{park2025infinite} proved an $\Omega(d H^{3/2}\sqrt{K})$ lower bound.
However, both bounds represents the worst case over a set of MDPs, which overly pessimistic.
In contrast, Theorem~\ref{thm:lower_bound} provides an instance-dependent lower bound where the parameter $\Thetab^{\star}$ is fixed and the choice of MNL-MDP is restricted.

The novelty in establishing the lower bound is in obtaining a variance-dependent bound for the KL-divergence between the two distinct multinomial distributions for the change-of-measure technique (Lemma~\ref{lem:KL}).
To derive the lower bound, we first involve the regret bound maximized over a local neighborhood of a specified instance parametrized by $\Thetab^{\star}$.
Based on the fact that the regret $\Reg_{\Pi_K}(\Mcal_{\Thetab})$ is Lipchitz continuous on $\Thetab$, we shrink the radius of the neighborhood to zero to obtain a lower bound for the specific instance parametrized by $\Thetab^{\star}$.
The details of the proof in Appendix~\ref{subsec:lower_bound_proof}.

\begin{figure*}[t]
    \centering
    \begin{subfigure}[b]{0.245\textwidth}
        \centering
        \includegraphics[width=\textwidth]{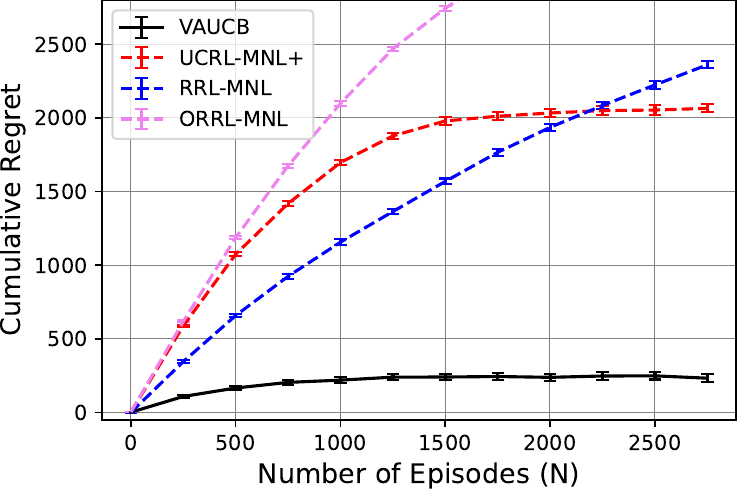}
        \caption{$S=4, H=12$}
        \label{fig:riverswim4}
    \end{subfigure}
    \begin{subfigure}[b]{0.245\textwidth}
        \centering
        \includegraphics[width=\textwidth]{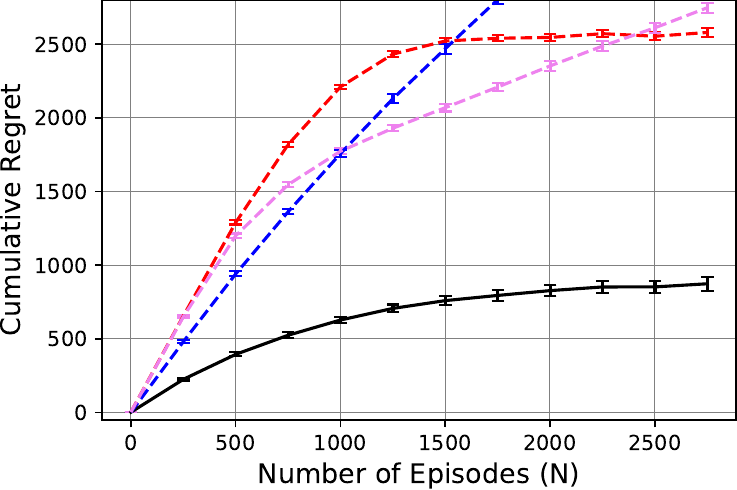}
        \caption{$S=8, H=24$}
        \label{fig:riverswim8}
    \end{subfigure}
    \begin{subfigure}[b]{0.245\textwidth}
        \centering
        \includegraphics[width=\textwidth]{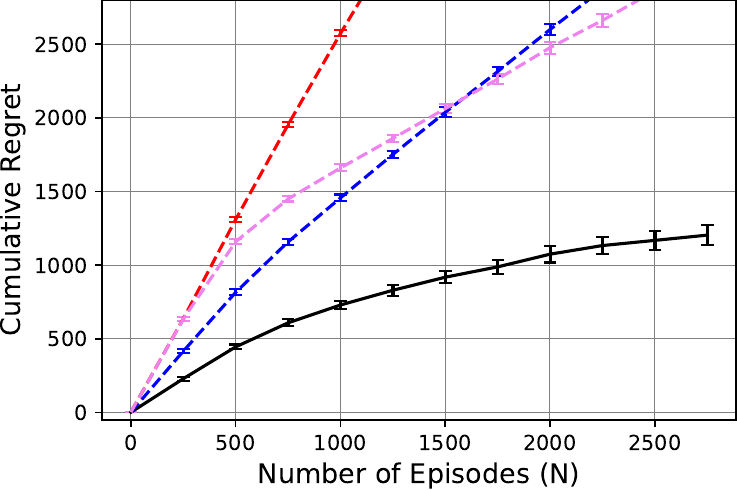}
        \caption{$S=12, H=36$}
        \label{fig:riverswim12}
    \end{subfigure}
    \begin{subfigure}[b]{0.245\textwidth}
        \centering
        \includegraphics[width=\textwidth]{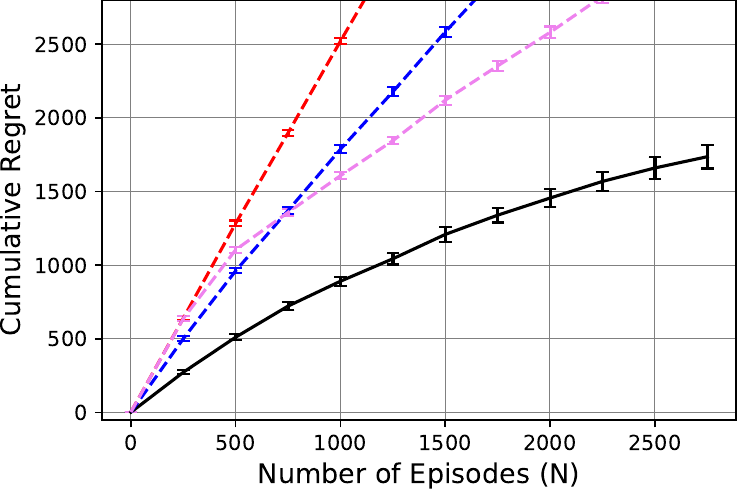}
        \caption{$S=16, H=48$}
        \label{fig:riverswim16}
    \end{subfigure}
    
    \caption{Empirical evaluation of the cumulative regret on the RiverSwim environment across four pairs of state $S$ and episode length $H$. We compare the proposed \UCB\  against \texttt{UCRL-MNL+}, \texttt{RRL-MNL}, and \texttt{ORRL-MNL}. Results are averaged over 20 repeated experiments  where the lines represent mean reward and shaded regions indicate standard deviation. The proposed \UCB\ yields less cumulative regret and finds the optimal policy faster than other methods.}
    \label{fig:riverswim_comparison}
\end{figure*}

\section{Experiments}
\label{sec:experiments}

To evaluate the empirical efficiency of the proposed \UCB\  algorithm, we conduct a series of experiments on the RiverSwim environment, a standard benchmark specifically designed to validate the exploration--exploitation capabilities of the algorithms for MNL-MDP \citep{ayoub2020model, hwang2023model, cho2024randomized}.
The environment isolates the exploration challenge in a simple sequential setting: the agent must choose between a locally attractive action with a small reward and a more difficult direction that is necessary to reach the larger long-term reward. 
This makes RiverSwim a natural testbed for assessing whether the algorithm performs the better long-horizon decision-making.
We defer the descriptions of the RiverSwim environment in Appendix~\ref{sec:Riverswim}.

Figure~\ref{fig:riverswim_comparison} illustrates the cumulative regret across varying problem scales, ranging from $S=4, H=12$ to $S=16, H=48$.
We compare the performance of \UCB\  against the state-of-the-art online algorithms: \texttt{UCRL-MNL+}, \texttt{RRL-MNL}, and \texttt{ORRL-MNL} \citep{cho2024randomized}.
The proposed \UCB\  demonstrates significantly superior convergence rates and stability across all scales. 
Notably, as the state space complexity increases, the performance gap between \UCB\  and the baseline algorithms widens. 
This performance gain is attributed to our novel parameter estimation with the observed information matrix and the second order upper confidence bound that captures the variance to avoid excessive exploration, leading to the rapid stabilization of the regret curves.

\bibliography{ref}

@article{amani2021ucb,
  title={Ucb-based algorithms for multinomial logistic regression bandits},
  author={Amani, Sanae and Thrampoulidis, Christos},
  journal={Advances in Neural Information Processing Systems},
  volume={34},
  pages={2913--2924},
  year={2021}
}

@article{xu2025towards,
  title={Towards optimal problem dependent generalization error bounds in statistical learning theory},
  author={Xu, Yunbei and Zeevi, Assaf},
  journal={Mathematics of Operations Research},
  volume={50},
  number={1},
  pages={40--67},
  year={2025},
  publisher={INFORMS}
}

@article{agrawal2019mnllbandit,
  title={MNL-Bandit: A dynamic learning approach to assortment selection},
  author={Agrawal, Shipra and Avadhanula, Vashist and Goyal, Vineet and Zeevi, Assaf},
  journal={Operations Research},
  volume={67},
  number={5},
  pages={1453--1485},
  year={2019},
  publisher={INFORMS}
}

@article{bartlett2006local,
  title={Local Rademacher complexities and empirical minimization},
  author={Bartlett, Peter L and Mendelson, Shahar},
  journal={Annals of Statistics},
  volume={34},
  year={2006}
}

@inproceedings{agrawal2017thompsonMNL,
  title={Thompson sampling for the MNL-Bandit},
  author={Agrawal, Shipra and Avadhanula, Vashist and Goyal, Vineet and Zeevi, Assaf},
  booktitle={Conference on Learning Theory (COLT)},
  pages={76--78},
  year={2017}
}

@inproceedings{wang2018near,
  title     = {Near-Optimal Policies for Dynamic Multinomial Logit Assortment Selection Models},
  author    = {Wang, Yining and Chen, Xi and Zhou, Yuan},
  booktitle = {Advances in Neural Information Processing Systems (NeurIPS)},
  year      = {2018},
  pages     = {1235--1244}
}

@article{chen2024robust,
  title={Robust dynamic assortment optimization in the presence of outlier customers},
  author={Chen, Xi and Krishnamurthy, Akshay and Wang, Yining},
  journal={Operations Research},
  volume={72},
  number={3},
  pages={999--1015},
  year={2024},
  publisher={Informs}
}

@book{sutton2018reinforcement,
  title={Reinforcement Learning: An Introduction},
  author={Sutton, Richard S and Barto, Andrew G},
  year={2018},
  publisher={MIT press}
}

@inproceedings{park2025infinite,
  title={Infinite-Horizon Reinforcement Learning with Multinomial Logit Function Approximation},
  author={Park, Jaehyun and Kwon, Junyeop and Lee, Dabeen},
  booktitle={International Conference on Artificial Intelligence and Statistics},
  pages={361--369},
  year={2025},
  organization={PMLR}
}

@article{li2024provably,
  title={Provably efficient reinforcement learning with multinomial logit function approximation},
  author={Li, Long-Fei and Zhang, Yu-Jie and Zhao, Peng and Zhou, Zhi-Hua},
  journal={Advances in Neural Information Processing Systems},
  volume={37},
  pages={58539--58573},
  year={2024}
}

@article{jun2017scalable,
  title={Scalable generalized linear bandits: Online computation and hashing},
  author={Jun, Kwang-Sung and Bhargava, Aniruddha and Nowak, Robert and Willett, Rebecca},
  journal={Advances in Neural Information Processing Systems},
  volume={30},
  year={2017}
}

@inproceedings{abeille2021instance,
  title={Instance-wise minimax-optimal algorithms for logistic bandits},
  author={Abeille, Marc and Faury, Louis and Calauz{\`e}nes, Cl{\'e}ment},
  booktitle={International Conference on Artificial Intelligence and Statistics},
  pages={3691--3699},
  year={2021},
  organization={PMLR}
}

@inproceedings{faury2020improved,
  title={Improved optimistic algorithms for logistic bandits},
  author={Faury, Louis and Abeille, Marc and Calauz{\`e}nes, Cl{\'e}ment and Fercoq, Olivier},
  booktitle={International Conference on Machine Learning},
  pages={3052--3060},
  year={2020},
  organization={PMLR}
}

@article{oh2019thompson,
  title={Thompson sampling for multinomial logit contextual bandits},
  author={Oh, Min-hwan and Iyengar, Garud},
  journal={Advances in Neural Information Processing Systems},
  volume={32},
  year={2019}
}

@article{zhang2024online,
  title={Online (multinomial) logistic bandit: Improved regret and constant computation cost},
  author={Zhang, Yu-Jie and Sugiyama, Masashi},
  journal={Advances in Neural Information Processing Systems},
  volume={36},
  year={2024}
}

@article{perivier2022dynamic,
  title={Dynamic pricing and assortment under a contextual MNL demand},
  author={Perivier, Noemie and Goyal, Vineet},
  journal={Advances in Neural Information Processing Systems},
  volume={35},
  pages={3461--3474},
  year={2022}
}

@inproceedings{oh2021multinomial,
  title={Multinomial logit contextual bandits: Provable optimality and practicality},
  author={Oh, Min-hwan and Iyengar, Garud},
  booktitle={Proceedings of the AAAI conference on artificial intelligence},
  volume={35},
  pages={9205--9213},
  year={2021}
}

@article{cho2024randomized,
  title={Randomized exploration for reinforcement learning with multinomial logistic function approximation},
  author={Cho, Wooseong and Hwang, Taehyun and Lee, Joongkyu and Oh, Min-hwan},
  journal={Advances in Neural Information Processing Systems},
  volume={37},
  pages={76643--76720},
  year={2024}
}

@inproceedings{lee2024improved,
  title={Improved regret bounds of (multinomial) logistic bandits via regret-to-confidence-set conversion},
  author={Lee, Junghyun and Yun, Se-Young and Jun, Kwang-Sung},
  booktitle={International Conference on Artificial Intelligence and Statistics},
  pages={4474--4482},
  year={2024},
  organization={PMLR}
}

@inproceedings{tran2015composite,
  title={Composite convex minimization involving self-concordant-like cost functions},
  author={Tran-Dinh, Quoc and Li, Yen-Huan and Cevher, Volkan},
  booktitle={Modelling, Computation and Optimization in Information Systems and Management Sciences: Proceedings of the 3rd International Conference on Modelling, Computation and Optimization in Information Systems and Management Sciences-MCO 2015-Part I},
  pages={155--168},
  year={2015},
  organization={Springer}
}

@article{lee2024nearly,
  title={Nearly minimax optimal regret for multinomial logistic bandit},
  author={Lee, Joongkyu and Oh, Min-hwan},
  journal={Advances in Neural Information Processing Systems},
  volume={37},
  pages={109003--109065},
  year={2024}
}

@article{jezequel2021mixability,
  title={Mixability made efficient: Fast online multiclass logistic regression},
  author={J{\'e}z{\'e}quel, R{\'e}mi and Gaillard, Pierre and Rudi, Alessandro},
  journal={Advances in Neural Information Processing Systems},
  volume={34},
  pages={23692--23702},
  year={2021}
}

@article{hazan2007logarithmic,
  title={Logarithmic regret algorithms for online convex optimization},
  author={Hazan, Elad and Agarwal, Amit and Kale, Satyen},
  journal={Machine Learning},
  volume={69},
  number={2},
  pages={169--192},
  year={2007},
  publisher={Springer}
}

@article{sun2019generalized,
  title={Generalized self-concordant functions: a recipe for newton-type methods},
  author={Sun, Tianxiao and Tran-Dinh, Quoc},
  journal={Mathematical Programming},
  volume={178},
  number={1},
  pages={145--213},
  year={2019},
  publisher={Springer}
}

@inproceedings{hwang2023model,
  title={Model-based reinforcement learning with multinomial logistic function approximation},
  author={Hwang, Taehyun and Oh, Min-hwan},
  booktitle={Proceedings of the AAAI conference on artificial intelligence},
  volume={37},
  pages={7971--7979},
  year={2023}
}

@article{jin2023provably,
  title={Provably Efficient Reinforcement Learning with Linear Function Approximation},
  author={Jin, Chi and Yang, Zhuoran and Wang, Zhaoran and Jordan, Michael I},
  journal={Mathematics of Operations Research},
  year={2023},
  publisher={INFORMS}
}

@inproceedings{ayoub2020model,
  title={Model-based reinforcement learning with value-targeted regression},
  author={Ayoub, Alex and Jia, Zeyu and Szepesvari, Csaba and Wang, Mengdi and Yang, Lin},
  booktitle={International Conference on Machine Learning},
  pages={463--474},
  year={2020},
  organization={PMLR}
}

@inproceedings{zanette2020frequentist,
  title={Frequentist regret bounds for randomized least-squares value iteration},
  author={Zanette, Andrea and Brandfonbrener, David and Brunskill, Emma and Pirotta, Matteo and Lazaric, Alessandro},
  booktitle={International Conference on Artificial Intelligence and Statistics},
  pages={1954--1964},
  year={2020},
  organization={PMLR}
}

@article{dann2017unifying,
  title={Unifying PAC and regret: Uniform PAC bounds for episodic reinforcement learning},
  author={Dann, Christoph and Lattimore, Tor and Brunskill, Emma},
  journal={Advances in Neural Information Processing Systems},
  volume={30},
  year={2017}
}

@inproceedings{azar2017minimax,
  title={Minimax regret bounds for reinforcement learning},
  author={Azar, Mohammad Gheshlaghi and Osband, Ian and Munos, R{\'e}mi},
  booktitle={International Conference on Machine Learning},
  pages={263--272},
  year={2017},
  organization={PMLR}
}

@article{kearns2002near,
  title={Near-optimal reinforcement learning in polynomial time},
  author={Kearns, Michael and Singh, Satinder},
  journal={Machine learning},
  volume={49},
  pages={209--232},
  year={2002},
  publisher={Springer}
}

@book{puterman2014markov,
  title={Markov decision processes: discrete stochastic dynamic programming},
  author={Puterman, Martin L},
  year={2014},
  publisher={John Wiley \& Sons}
}

@article{tropp2012user,
  title={User-friendly tail bounds for sums of random matrices},
  author={Tropp, Joel A},
  journal={Foundations of computational mathematics},
  volume={12},
  number={4},
  pages={389--434},
  year={2012},
  publisher={Springer}
}

@inproceedings{abbasi2011improved,
  title={Improved algorithms for linear stochastic bandits},
  author={Abbasi-Yadkori, Yasin and P{\'a}l, D{\'a}vid and Szepesv{\'a}ri, Csaba},
  booktitle={Advances in Neural Information Processing Systems},
  pages={2312--2320},
  year={2011}
}

@inproceedings{ishfaq2021randomized,
  title={Randomized exploration in reinforcement learning with general value function approximation},
  author={Ishfaq, Haque and Cui, Qiwen and Nguyen, Viet and Ayoub, Alex and Yang, Zhuoran and Wang, Zhaoran and Precup, Doina and Yang, Lin},
  booktitle={International Conference on Machine Learning},
  pages={4607--4616},
  year={2021},
  organization={PMLR}
}

@article{zhang2020almost,
  title={Almost optimal model-free reinforcement learningvia reference-advantage decomposition},
  author={Zhang, Zihan and Zhou, Yuan and Ji, Xiangyang},
  journal={Advances in Neural Information Processing Systems},
  volume={33},
  pages={15198--15207},
  year={2020}
}

@article{li2021breaking,
  title={Breaking the sample complexity barrier to regret-optimal model-free reinforcement learning},
  author={Li, Gen and Shi, Laixi and Chen, Yuxin and Gu, Yuantao and Chi, Yuejie},
  journal={Advances in Neural Information Processing Systems},
  volume={34},
  pages={17762--17776},
  year={2021}
}

@inproceedings{zhang2024settling,
  title={Settling the sample complexity of online reinforcement learning},
  author={Zhang, Zihan and Chen, Yuxin and Lee, Jason D and Du, Simon S},
  booktitle={The Thirty Seventh Annual Conference on Learning Theory},
  pages={5213--5219},
  year={2024},
  organization={PMLR}
}

@inproceedings{lee2025minimax,
  title={Minimax Optimal Reinforcement Learning with Quasi-Optimism},
  author={Lee, Harin and Oh, Min-hwan},
  booktitle={The Thirteenth International Conference on Learning Representations},
  year={2025}
}

@article{brafman2002r,
  title={R-max-a general polynomial time algorithm for near-optimal reinforcement learning},
  author={Brafman, Ronen I and Tennenholtz, Moshe},
  journal={Journal of Machine Learning Research},
  volume={3},
  number={Oct},
  pages={213--231},
  year={2002}
}

@inproceedings{lee2024demystifying,
  title={Demystifying Linear MDPs and Novel Dynamics Aggregation Framework},
  author={Lee, Joongkyu and Oh, Min-hwan},
  booktitle={The Twelfth International Conference on Learning Representations},
  year={2024}
}
\bibliographystyle{abbrvnat}

\newpage
\appendix

\section{Experimental Settings}
\label{sec:Riverswim}

The RiverSwim environment consists of a linear chain of $S$ states, $\mathcal{S}=\{1,2,\dots,S\}$, where the agent begins at the leftmost state. At each state $i$, the agent can choose between two actions: \textsf{Left} ($L$) and \textsf{Right} ($R$). The environment is characterized by a "current" that favors movement toward state $1$, creating a challenging exploration problem:

\textbf{Dynamics:} At interior states $i \in \{2, \dots, S-1\}$, taking action $L$ moves the agent deterministically to $i-1$. Conversely, action $R$ involves a stochastic transition: the agent succeeds in moving to $i+1$ with probability $p_{R\to}=0.35$, remains in $i$ with probability $p_{R\circ}=0.35$, or is pushed back to $i-1$ with probability $p_{R\leftarrow}=0.30$.

\textbf{Rewards:} A small reward $r_{\text{left}}=0.005$ is granted for taking $L$ at the leftmost boundary ($i=1$), providing a "local" trap for sub-optimal policies. A significantly larger reward $r_{\text{right}}=1.0$ is only obtainable at the rightmost boundary ($i=S$) by taking $R$, necessitating deep exploration against the current.

\textbf{Computation Resources:} All experiments were conducted on an EC2 c7i.xlarge instance featuring custom 4th Generation Intel Xeon Scalable processors with a dedicated x86\_64 architecture to ensure high-performance execution of the numerical simulations.

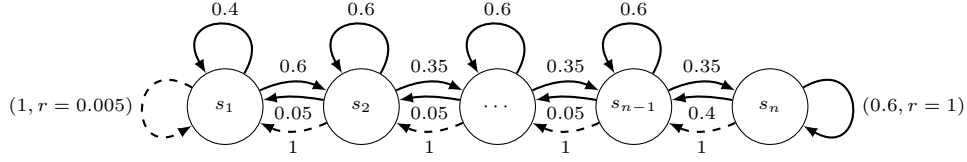
\begin{figure}[t] 
\centering
\begin{tikzpicture}[
  >=latex,
  x=12mm, y=10mm,           
  state/.style={circle,draw,minimum width=10mm,inner sep=0.8pt, font=\scriptsize},
  Redge/.style={->,thick},  
  Ledge/.style={->,thick,dashed},
  every node/.style={font=\scriptsize} 
]

\node[state] (s1)   at (0,0) {$s_1$};
\node[state] (s2)   at (1.5,0) {$s_2$};
\node[state] (dots) at (3,0) {$\cdots$};   
\node[state] (snm) at (4.5,0) {$s_{n-1}$};
\node[state] (sn)   at (6,0) {$s_n$};

\draw[Redge] (s1) edge[out=60,in=120,looseness=5] node[above] {$0.4$} (s1);
\draw[Redge] (s1) to[bend left=25] node[above,sloped] {$0.6$} (s2);

\draw[Redge] (s2) edge[out=60,in=120,looseness=5] node[above] {$0.6$} (s2);
\draw[Redge] (s2) to[bend left=25] node[above,sloped] {$0.35$} (dots);
\draw[Redge] (s2) to[bend right=10] node[below,sloped] {$0.05$} (s1);

\draw[Redge] (dots) to[bend left=25] node[above,sloped] {$0.35$} (snm);
\draw[Redge] (dots) edge[out=60,in=120,looseness=5] node[above] {$0.6$} (dots);
\draw[Redge] (dots) to[bend right=10] node[below,sloped] {$0.05$} (s2);

\draw[Redge] (snm) edge[out=60,in=120,looseness=5] node[above] {$0.6$} (snm);
\draw[Redge] (snm) to[bend left=25] node[above,sloped] {$0.35$} (sn);
\draw[Redge] (snm) to[bend right=10] node[below,sloped] {$0.05$} (dots);

\draw[Redge] (sn) to[bend right=10] node[below,sloped] {$0.4$} (snm);
\draw[Redge] (sn) edge[out=30,in=-30,looseness=5] node[right] {$(0.6,r=1)$} (sn);

\draw[Ledge] (s2)   to[bend left=25] node[below,sloped] {$1$} (s1);
\draw[Ledge] (dots) to[bend left=25] node[below,sloped] {$1$} (s2);
\draw[Ledge] (snm) to[bend left=25] node[below,sloped] {$1$} (dots);
\draw[Ledge] (sn)   to[bend left=25] node[below,sloped] {$1$} (snm);
\draw[Ledge] (s1) edge[out=150,in=210,looseness=5] node[left] {$(1, r=0.005)$} (s1);

\end{tikzpicture}
\caption{The \emph{RiverSwim} MDP environment with $n$ states. This environment is commonly used to evaluate exploration in RL. Solid arrows represent the swim right action ($R$); dashed arrows represent the go left action ($L$).}
\label{fig:riverswim-figure}
\end{figure}

\section{Missing Proofs}

\subsection{Proof of Lemma~\ref{lem:exp_to_quad}}
\QuadError*

\label{sec:proof_exp_to_quad} Let $\xib_{\nu,h}:=\EE[\eb_{s_{\nu,h+1}}\mid\Fcal_{\nu,h}]-\eb_{s_{\nu,h+1}}$
denote the residual. By Proposition~\ref{prop:mean_var}, $\xib_{\nu,h}=\eb_{s_{\nu,h+1}}-\nabla L_{\Scal_{h}(s_{h},a_{h})}\big(\vPhib(s_{k,h},a_{k,h})\thetab_{h}^{\star}\big)$
and $\xib_{\nu,h}{}^{\top}\vPhib(s_{\nu,h},a_{\nu,h})=\nabla\ell_{\nu,h}(\thetab_{h}^{\star})$.
Thus,
\[
\nabla\ell_{\nu,h}(\thetab_{\nu-1,h})=\nabla\ell_{\nu,h}(\thetab_{\nu-1,h})-\nabla\ell_{\nu,h}(\thetab_{h}^{\star})+\xib_{\nu,h}^{\top}\vPhib(s_{\nu,h},a_{\nu,h}).
\]
Taking the conditional expectation given $\Fcal_{\nu,h}$, and using
$\EE[\xib_{\nu,h}\mid\Fcal_{\nu,h}]=\boldsymbol{0}$ together with
$\vPhib(s_{\nu,h},a_{\nu,h})^{\top}\EE[\xib_{\nu,h}\xib_{\nu,h}^{\top}\mid\Fcal_{\nu,h}]\vPhib(s_{\nu,h},a_{\nu,h})=\nabla^{2}\ell_{\nu,h}(\thetab_{h}^{\star})$
(Proposition~\ref{prop:mean_var}) completes the proof.

\subsection{Proof of Theorem \ref{thm:self}}
\self*
\begin{proof}
\medspace{}By the triangular inequality, for each $\nu\in[k]$, 
\[
\begin{aligned} & \|\widehat{\thetab}_{k,h}-\thetab_{h}^{\star}\|_{\widehat{\Hb}_{\nu,h}-\widehat{\Hb}_{\nu-1,h}}^{2}\\
 & =\big(\nabla\ell_{\nu,h}(\thetab_{\nu-1,h})(\widehat{\thetab}_{k,h}-\thetab_{h}^{\star})\big)^{2}\\
 & \le2\big(\nabla\ell_{\nu,h}(\thetab_{\nu-1,h})(\widehat{\thetab}_{k,h}-\thetab_{\nu-1,h})\big)^{2}+2\big(\nabla\ell_{\nu,h}(\thetab_{\nu-1,h})(\thetab_{\nu-1,h}-\thetab_{h}^{\star})\big)^{2}.
\end{aligned}
\]
Summing up over $\nu\in[k]$, 
\[
\begin{aligned} & \norm{\widehat{\thetab}_{k,h}-\thetab_{h}^{\star}}_{\widehat{\Hb}_{k,h}}^{2}-\epsilon\norm{\widehat{\thetab}_{k,h}-\thetab_{h}^{\star}}_{2}^{2}\\
 & \le2\sum_{\nu=1}^{k}\big(\nabla\ell_{\nu,h}(\thetab_{\nu-1,h})(\widehat{\thetab}_{k,h}-\thetab_{\nu-1,h})\big)^{2}+2\sum_{\nu=1}^{k}\big(\nabla\ell_{\nu,h}(\thetab_{\nu-1,h})(\thetab_{\nu-1,h}-\thetab_{h}^{\star})\big)^{2}.
\end{aligned}
\]
Because $\widehat{\thetab}_{k,h}$ is the minimizer, 
\begin{align*}
&\sum_{\nu=1}^{k}\big(\nabla\ell_{\nu,h}(\thetab_{\nu-1,h})(\widehat{\thetab}_{k,h}-\thetab_{\nu-1,h})\big)^{2}+\frac{\epsilon}{2}\norm{\widehat{\thetab}_{k,h}}_{2}^{2}\\
&\le\sum_{\nu=1}^{k}\big(\nabla\ell_{\nu,h}(\thetab_{\nu-1,h})(\thetab_{h}^{\star}-\thetab_{\nu-1,h})\big)^{2}+\frac{\epsilon}{2}\norm{\thetab_{h}^{\star}}_{2}^{2}.
\end{align*}
Thus, 
\[
\norm{\widehat{\thetab}_{k,h}-\thetab_{h}^{\star}}_{\widehat{\Hb}_{k,h}}^{2}\le\epsilon\big(\|\widehat{\thetab}_{k,h}-\thetab_{h}^{\star}\|_{2}^{2}-\|\widehat{\thetab}_{k,h}\|_{2}^{2}+\|\thetab_{h}^{\star}\|_{2}^{2}\big)+4\sum_{\nu=1}^{k}\big(\nabla\ell_{\nu,h}(\thetab_{\nu-1,h})(\thetab_{\nu-1,h}-\thetab_{h}^{\star})\big)^{2}.
\]
By Cauchy Schwartz inequality, 
\[
\begin{split}\|\widehat{\thetab}_{k,h}-\thetab_{h}^{\star}\|_{2}^{2}-\|\widehat{\thetab}_{k,h}\|_{2}^{2}+\|\thetab_{h}^{\star}\|_{2}^{2}= & -2\thetab_{h}^{\star}\widehat{\thetab}_{k,h}+2\|\thetab_{h}^{\star}\|_{2}^{2}\\
= & 2\big(\thetab_{h}^{\star}-\widehat{\thetab}_{k,h}\big)^{\top}\thetab_{h}^{\star}\\
\le & 2\norm{\thetab_{h}^{\star}-\widehat{\thetab}_{k,h}}_{\widehat{\Hb}_{k,h}}\norm{\thetab_{h}^{\star}}_{\widehat{\Hb}_{k,h}^{-1}}\\
\le & 2\epsilon^{-1/2}B_{\thetab}\norm{\thetab_{h}^{\star}-\widehat{\thetab}_{k,h}}_{\widehat{\Hb}_{k,h}}.
\end{split}
\]
Thus, we obtain, 
\[
\norm{\widehat{\thetab}_{k,h}-\thetab_{h}^{\star}}_{\widehat{\Hb}_{k,h}}^{2}\le2B_{\thetab}\sqrt{\epsilon}\norm{\widehat{\thetab}_{k,h}-\thetab_{h}^{\star}}_{\widehat{\Hb}_{k,h}}+4\sum_{\nu=1}^{k}\big(\nabla\ell_{\nu,h}(\thetab_{\nu-1,h})(\thetab_{\nu-1,h}-\thetab_{h}^{\star})\big)^{2}.
\]
For $a,b>0$, 
\[
x^{2}\le2ax+b,\;\;x\ge0\Longrightarrow x\le a+\sqrt{a^{2}+b}.
\]
Thus, 
\[
\norm{\widehat{\thetab}_{k,h}-\thetab_{h}^{\star}}_{\widehat{\Hb}_{k,h}}\le\sqrt{\epsilon}B_{\thetab}+\sqrt{\epsilon B_{\thetab}^{2}+4\sum_{\nu=1}^{k}\big(\nabla\ell_{\nu,h}(\thetab_{\nu-1,h})(\thetab_{\nu-1,h}-\thetab_{h}^{\star})\big)^{2}},
\]
Applying Lemma \ref{lem:regret_to_self},
\begin{align*}
 & \sum_{\nu=1}^{k}\big(\nabla\ell_{\nu,h}(\thetab_{\nu-1,h})(\thetab_{\nu-1,h}-\thetab_{h}^{\star})\big)^{2}\\
 & \le2C_{\vPhib,\thetab}\bigg(\sum_{\nu=1}^{k}\big(\ell_{\nu,h}(\thetab_{\nu-1,h})-\ell_{\nu,h}(\thetab_{h}^{\star})\big)+\frac{16B_{\vPhib}^{2}B_{\thetab}^{2}}{C_{\vPhib,\thetab}}\log\frac{d}{\delta}+2(e-1)(6+8B_{\vPhib}B_{\thetab})\log\frac{1}{\delta}\bigg)\\
 & \le2C_{\vPhib,\thetab}\bigg(\sum_{\nu=1}^{k}\big(\ell_{\nu,h}(\thetab_{\nu-1,h})-\ell_{\nu,h}(\thetab_{h}^{\star})\big)\bigg)+32B_{\vPhib}^{2}B_{\thetab}^{2}\log\frac{d}{\delta}+4C_{\vPhib,\thetab}^{2}\log\frac{1}{\delta}
\end{align*}
By Lemma \ref{lem:online_regret}, we obtain 
\begin{align*}
 & \sum_{\nu=1}^{k}\big(\nabla\ell_{\nu,h}(\thetab_{\nu-1,h})(\thetab_{\nu-1,h}-\thetab_{h}^{\star})\big)^{2}\\
 & \le2C_{\vPhib,\thetab}\bigg(\frac{2B_{\thetab}^{2}\epsilon}{\eta_{\vPhib,\thetab}}+d\eta_{\vPhib,\thetab}\log\frac{k+1}{d}+\frac{8B_{\vPhib}^{2}B_{\thetab}^{2}}{\eta_{\vPhib,\thetab}}\log\frac{1}{\delta}\bigg)+32B_{\vPhib}^{2}B_{\thetab}^{2}\log\frac{d}{\delta}+4C_{\vPhib,\thetab}^{2}\log\frac{1}{\delta}\\
 & \le\frac{4C_{\vPhib,\thetab}B_{\thetab}^{2}\epsilon}{\eta_{\vPhib,\thetab}}+2dC_{\vPhib,\thetab}\eta_{\vPhib,\thetab}\log\frac{k+1}{d}+\bigg(\frac{16C_{\vPhib,\thetab}B_{\vPhib}^{2}B_{\thetab}^{2}}{\eta_{\vPhib,\thetab}}+4C_{\vPhib,\thetab}^{2}\bigg)\log\frac{1}{\delta}+32B_{\vPhib}^{2}B_{\thetab}^{2}\log\frac{d}{\delta},
\end{align*}
where the last term is defined as $\gamma_{k}$. Thus, setting $\beta_{k}=\sqrt{\epsilon}B_{\thetab}+\sqrt{\epsilon B_{\thetab}^{2}+4\gamma_{k}},$
completes the proof. 
\end{proof}

\subsection{Proof of Lemma~\ref{lem:Hessian_bound}}
\HessianBound*
\begin{proof}
\medspace{} For each $h\in[H]$, let $\Fcal_{\nu,h}$ denote the
$\sigma$-algebra generated by $\{s_{\nu^{\prime},h},a_{\nu^{\prime},h}:h\in[H],\nu^{\prime}\in[\nu-1]\}$
and $\{s_{h^{\prime}}^{(\nu)},a_{h^{\prime}}^{(\nu)}:h^{\prime}\in[h]\}$.
For $\nu\in[k]$ and $h\in[H]$ define $\etab_{\nu,h}^{\star}:=\vPhib(s_{\nu,h},a_{\nu,h})\thetab_{h}^{\star}$
and $\etab_{\nu,h}:=\vPhib(s_{\nu,h},a_{\nu,h})\thetab_{k-1,h}$.
Following the argument in \eqref{eq:ell_norm_bound_1} and \eqref{eq:ell_norm_bound_2},
\[
\lambda_{\max}\Big(\nabla\ell_{\nu,h}(\thetab_{\nu-1,h})\nabla\ell_{\nu,h}(\thetab_{\nu-1,h})^{\top}\Big)\le2B_{\vPhib}^{2}.
\]
Since 
\[
\min_{\xb:\|\xb\|_{2}\le1}\xb^{\top}\nabla\ell_{\nu,h}(\thetab_{\nu-1,h})\nabla\ell_{\nu,h}(\thetab_{\nu-1,h})^{\top}\xb=\min_{\xb:\|\xb\|_{2}\le1}\Big(\nabla\ell_{\nu,h}(\thetab_{\nu-1,h})^{\top}\xb\Big)^{2}\ge0,
\]
we can use Matrix Chernoff bound (Lemma \ref{lem:chernoff_matrix-1})
to derive with probability at least $1-\delta$, 
\begin{align*}
&\sum_{\nu=1}^{k}\nabla\ell_{\nu,h}(\thetab_{\nu-1,h})\nabla\ell_{\nu,h}(\thetab_{\nu-1,h})^{\top}\\
&\succeq(e-1)\sum_{\nu=1}^{k}\EE\big[\nabla\ell_{\nu,h}(\thetab_{\nu-1,h})\nabla\ell_{\nu,h}(\thetab_{\nu-1,h})^{\top}\big|\Fcal_{\nu,h}\big]-\log\frac{d}{\delta}\Ib_{d}.
\end{align*}
for all $k\in[K]$. By the covariance inequality (Proposition \ref{prop:var_inequality}),
for each $\nu\in[k]$, 
\[
\begin{split} & \EE\bigg[\Big(\eb_{s_{\nu,h+1}}-\nabla L_{\nu,h}\big(\etab_{\nu,h}\big)\Big)\Big(\eb_{s_{\nu,h+1}}-\nabla L_{\nu,h}\big(\etab_{\nu,h}\big)\Big)^{\top}\bigg|\Fcal_{\nu,h}\bigg]\\
 & \succeq\EE\bigg[\Big(\eb_{s_{\nu,h+1}}-\EE\big[\eb_{s_{\nu,h+1}}\big|\Fcal_{\nu,h}\big]\Big)\Big(\eb_{s_{\nu,h+1}}-\EE\big[\eb_{s_{\nu,h+1}}\big|\Fcal_{\nu,h}\big]\Big)^{\top}\bigg|\Fcal_{\nu,h}\bigg]\\
 & =\nabla^{2}L_{\nu,h}(\etab_{\nu,h}^{\star})
\end{split}
\]
where the last equality uses Proposition \ref{prop:mean_var}. Because
\[
\nabla^{2}\ell_{\nu,h}(\thetab_{h}^{\star})=\vPhib(s_{\nu,h},a_{\nu,h})^{\top}\nabla^{2}L_{\nu,h}(\etab_{\nu,h}^{\star})\vPhib(s_{\nu,h},a_{\nu,h}),
\]
we obtain, 
\[
\sum_{\nu=1}^{k}\nabla\ell_{\nu,h}(\thetab_{\nu-1,h})\nabla\ell_{\nu,h}(\thetab_{\nu-1,h})^{\top}\succeq(1-e^{-1})\sum_{\nu=1}^{k}\nabla^{2}\ell_{\nu,h}(\thetab_{h}^{\star})-4B_{\vPhib}^{2}\log\frac{d}{\delta}\Ib_{d}.
\]
Thus,
\begin{align*}
\widehat{\Hb}_{k,h}:= & \sum_{\nu=1}^{k}\nabla\ell_{\nu,h}(\thetab_{\nu-1,h})\nabla\ell_{\nu,h}(\thetab_{\nu-1,h})^{\top}+\epsilon\Ib_{d}\\
\succeq & (1-e^{-1})\sum_{\nu=1}^{k}\nabla^{2}\ell_{\nu,h}(\thetab_{h}^{\star})+(\epsilon-4B_{\vPhib}^{2}\log\frac{d}{\delta})\Ib_{d}\\
\succeq & (1-e^{-1})\sum_{\nu=1}^{k}\nabla^{2}\ell_{\nu,h}(\thetab_{h}^{\star}),
\end{align*}
where the last inequality holds by $\epsilon>B_{\vPhib}^{2}+4B_{\vPhib}^{2}\log\frac{d}{\delta}$. 
\end{proof}

\subsection{Proof of Lemma \ref{lem:optimism}}
\Optimism*
\label{sec:proof_optimism}
\begin{proof}
\medspace{} For the base case $h=H$, we have 
\[
Q_{H}^{\star}(s,a)=r(s,a)=\widehat{Q}_{H}(s,a).
\]
Suppose $Q_{h+1}^{\star}(s,a)\le\widehat{Q}_{k,h+1}(s,a)$ holds.
Then, by definition of $Q_{h}^{\star}$, 
\[
\begin{split}Q_{h}^{\star}(s,a)= & \big(r+P_{h}V_{h+1}^{\star}\big)(s,a)\\
= & r(s,a)+\EE\Big[V_{h+1}^{\star}(s_{h+1})\Big|s,a\Big]\\
= & r(s,a)+\EE\Big[\max_{a_{h+1}\in\Acal}Q_{h+1}^{\star}(s^{\prime},a_{h+1})\Big|s,a\Big]\\
\le & r(s,a)+\EE\Big[\max_{a_{h+1}\in\Acal}\widehat{Q}_{k,h+1}(s^{\prime},a_{h+1})\Big|s,a\Big].
\end{split}
\]
By definition $\widehat{V}_{k,h+1}(s^{\prime}):=\max_{a_{h+1}\in\Acal}\widehat{Q}_{k,h+1}(s^{\prime},a_{h+1})$,
and 
\[
\EE\Big[\max_{a_{h+1}\in\Acal}\widehat{Q}_{k,h+1}(s^{\prime},a_{h+1})\Big|s,a\Big]=\sum_{s^{\prime}\in\Scal_{k,h}}\widehat{V}_{k,h+1}(s^{\prime})\PP(\eb_{s^{\prime}}|s,a,\thetab_{h}^{\star})
\]
By the second order bound for the multinomial kernel difference (Lemma
\ref{lem:diff_second}), 
\begin{equation}
\begin{aligned} & \sum_{s^{\prime}\in\Scal_{k,h}}\widehat{V}_{k,h+1}(s^{\prime})\PP(\eb_{s^{\prime}}|s,a,\thetab_{h}^{\star})\\
 & \le\sum_{s^{\prime}\in\Scal_{k,h}}\widehat{V}_{k,h+1}(s^{\prime})\PP(\eb_{s^{\prime}}|s,a,\widehat{\thetab}_{k-1,h})\\
 & +\sum_{s^{\prime}\in\Scal_{k,h}}\widehat{V}_{k,h+1}(s^{\prime})\eb_{s^{\prime}}^{\top}\nabla^{2}L_{\Scal_{h}(s,a)}(\vPhib(s,a)\widehat{\thetab}_{k-1,h})\vPhib(s,a)\big(\thetab_{h}^{\star}-\widehat{\thetab}_{k-1,h}\big)\\
 & +\max_{s^{\prime}\in\Scal_{k,h}}\widehat{V}_{k,h+1}(s^{\prime})\max_{s^{\prime}\in\Scal_{k,h}}\big((\thetab_{h}^{\star}-\widehat{\thetab}_{k-1,h})^{\top}\vPhib(s,a)^{\top}\eb_{s^{\prime}}\big)^{2}.
\end{aligned}
\label{eq:opt_expansion}
\end{equation}
By Cauchy-Schwartz inequality, 
\[
\begin{split} & \sum_{s^{\prime}\in\Scal_{k,h}}\widehat{V}_{k,h+1}(s^{\prime})\eb_{s^{\prime}}^{\top}\nabla^{2}L_{\Scal_{h}(s,a)}(\vPhib(s,a)\widehat{\thetab}_{k-1,h})\vPhib(s,a)\big(\thetab_{h}^{\star}-\widehat{\thetab}_{k-1,h}\big)\\
 & \le\norm{\sum_{s^{\prime}\in\Scal_{k,h}}\widehat{V}_{k,h+1}(s^{\prime})\vPhib(s,a)^{\top}\nabla^{2}L_{\Scal_{h}(s,a)}(\vPhib(s,a)\widehat{\thetab}_{k-1,h})\eb_{s^{\prime}}}_{\widehat{\Hb}_{k-1,h}^{-1}}\norm{\thetab_{h}^{\star}-\widehat{\thetab}_{k-1,h}}_{\widehat{\Hb}_{k-1,h}}\\
 & \le\beta_{k-1}\norm{\sum_{s^{\prime}\in\Scal_{k,h}}\widehat{V}_{k,h+1}(s^{\prime})\vPhib(s,a)^{\top}\nabla^{2}L_{\Scal_{h}(s,a)}(\vPhib(s,a)\widehat{\thetab}_{k-1,h})\eb_{s^{\prime}}}_{\widehat{\Hb}_{k-1,h}^{-1}},
\end{split}
\]
where the last inequality uses the confidence bound for $\widehat{\thetab}_{k-1,h}$
(Theorem \ref{thm:self}). Similarly, 
\[
\begin{split}\max_{s^{\prime}\in\Scal_{k,h}}\big((\thetab_{h}^{\star}-\widehat{\thetab}_{k-1,h})^{\top}\vPhib(s,a)^{\top}\eb_{s^{\prime}}\big)^{2}\le & \norm{\thetab_{h}^{\star}-\widehat{\thetab}_{k-1,h}}_{\widehat{\Hb}_{k-1,h}}^{2}\max_{s^{\prime}\in\Scal_{k,h}}\norm{\vPhib(s,a)^{\top}\eb_{s^{\prime}}}_{\widehat{\Hb}_{k-1,h}}^{2}\\
\le & \beta_{k-1}^{2}\max_{s^{\prime}\in\Scal_{k,h}}\norm{\vPhib(s,a)^{\top}\eb_{s^{\prime}}}_{\widehat{\Hb}_{k-1,h}}^{2},
\end{split}
\]
Thus, 
\[
\begin{split}
&\sum_{s^{\prime}\in\Scal_{k,h}}\widehat{V}_{k,h+1}(s^{\prime})\PP(\eb_{s^{\prime}}|s,a,\thetab_{h}^{\star}) \\ &\le\sum_{s^{\prime}\in\Scal_{k,h}}\widehat{V}_{k,h+1}(s^{\prime})\PP(\eb_{s^{\prime}}|s,a,\widehat{\thetab}_{k-1,h})\\
 & +\beta_{k-1}\norm{\sum_{\tilde{s}\in\Scal_{k,h}}\widehat{V}_{k,h+1}(\tilde{s})\vPhib(s,a)^{\top}\nabla^{2}L_{\Scal_{h}(s,a)}\big(\vPhib(s,a)\widehat{\thetab}_{k-1,h}\big)\eb_{\tilde{s}}}_{\widehat{\Hb}_{k-1,h}^{-1}}\\
 & +\beta_{k-1}^{2}\sup_{s^{\prime}\in\Scal_{k,h}}\widehat{V}_{k,h+1}(s^{\prime})\max_{s^{\prime}\in\Scal_{k,h}}\norm{\eb_{s^{\prime}}^{\top}\vPhib(s,a)}_{\widehat{\Hb}_{k-1,h}^{-1}}^{2},
\end{split}
\]
which implies 
\[
\begin{split}Q_{h}^{\star}(s,a)\le & r(s,a)+\EE\Big[\max_{a_{h+1}\in\Acal}\widehat{Q}_{k,h+1}(s^{\prime},a_{h+1})\Big|s,a\Big]\\
\le & r(s,a)+\sum_{s^{\prime}\in\Scal_{k,h}}\widehat{V}_{k,h+1}(s^{\prime})\PP(\eb_{s^{\prime}}|s,a,\widehat{\thetab}_{k-1,h})\\
 & +\beta_{k-1}\norm{\sum_{\tilde{s}\in\Scal_{k,h}}\widehat{V}_{k,h+1}(\tilde{s})\vPhib(s,a)^{\top}\nabla^{2}L_{\Scal_{h}(s,a)}\big(\vPhib(s,a)\widehat{\thetab}_{k-1,h}\big)\eb_{\tilde{s}}}_{\widehat{\Hb}_{k-1,h}^{-1}}\\
 & +\beta_{k-1}^{2}\sup_{s^{\prime}\in\Scal_{k,h}}\widehat{V}_{k,h+1}(s^{\prime})\max_{s^{\prime}\in\Scal_{k,h}}\norm{\eb_{s^{\prime}}^{\top}\vPhib(s,a)}_{\widehat{\Hb}_{k-1,h}^{-1}}^{2}.
\end{split}
\]
By definition, we obtain $Q_{h}^{\star}(s,a)\le\widehat{Q}_{k,h}(s,a)$.
By induction argument, the inequality holds for all $h\in[H]$. 
\end{proof}

\subsection{Proof of Theorem \ref{thm:regret_bound}}

\label{subsec:regret_proof} \begin{theorem}[Regret Bound of \UCB\ ]
Let $a_{k,h}$ denote the action selected by \UCB\ Algorithm at
step $h\in[H]$ in episode $k\in[K]$ and 
\[
\sigma_{k,h}^{2}:=\max_{\xb:\|\xb\|_{\infty}\le1}\xb^{\top}\nabla^{2}L_{k,h}\big(\vPhib_{k,h}\thetab_{h}^{\star}\big)\xb,
\]
denote the maximal variance of the multinomial distribution. If the
algorithm $\Pi_{K}$ is set to \UCB\, then for any MNL-MDP $\Mcal$,
With probability at least $1-5\delta$, 
\[
\begin{split}
&\Reg(\Pi_K)\le  \Big(\frac{H}{2}+2H\log\frac{K+1}{\delta}+2\sqrt{H}\log2\Big)\sqrt{1+\sum_{k=1}^{K}\sum_{h=1}^{H}\sigma_{k,h}^{2}}\\
 & +\frac{2\beta_{K-1}\!\sqrt{2dH\log\!\frac{K+1}{d}}}{\sqrt{1-e^{-1}}}\!\!\!\sqrt{12eH\sum_{k=1}^{K}\sum_{h=1}^{H}\sigma_{k,h}^{2}\!+\!1200eH^{4}\kappa^{-2}\beta_{K-1}^{4}\big(d\log\frac{K+1}{d}\!+\!3\log\frac{1}{\delta}\big)}\\
 & +16H^{2}\beta_{K-1}^{2}d\log\frac{K+1}{d},
\end{split}
\]
\end{theorem}
\begin{proof}
By definition of the action-value function and the value function,
$V_{h}(s;\pi_{n})=\EE_{\tilde{a}_{n,h}\sim\pi_{n}(s,h)}[Q_{h}(s,\tilde{a}_{n,h})]$
for $s\in\Scal$. Let $a_{n,h}:=\argmax{a\in\Acal}\widehat{Q}_{n,h}(s_{n,h},a)$
denote the action chosen by Algorithm \UCB\ at step $h$ and episode
$n$. Because \UCB\ chooses the action deterministically, we have
$V_{h}(s_{n,h};\pi_{n})=Q_{h}(s_{n,h},a_{n,h};\pi_{n})$. Thus the
regret at episode $n\in[K]$, 
\begin{align*}
 & V_{1}^{\star}(s_{n,1})-V_{n,1}(s_{n,1};\pi_{n})\\
 & =Q_{1}^{\star}\big(s_{n,1},\pi_{\star}(s_{n,1})\big)-Q_{n,1}(s_{n,1},a_{n,1};\pi_{n}).
\end{align*}
By the optimism principle for $\widehat{Q}_{n,1}$ (Lemma \ref{lem:optimism}),
\begin{align*}
 & =Q_{1}^{\star}\big(s_{n,1},\pi_{\star}(s_{n,1})\big)-Q_{n,1}(s_{n,1},a_{n,1};\pi_{n})\\
 & \leq\widehat{Q}_{n,1}\big(s_{n,1},\pi_{\star}(s_{n,1})\big)-Q_{n,1}\big(s_{n,1},a_{n,1};\pi_{n}\big)\\
 & \leq\widehat{Q}_{n,1}\big(s_{n,1},a_{n,1}\big)-Q_{n,1}\big(s_{n,1},a_{n,1};\pi_{n}\big),
\end{align*}
where the last inequality holds by $a_{n,1}:=\argmax{a\in\Acal}\widehat{Q}_{n,1}(s_{n,1},a)$.
For each $n\in[K]$ and $h\in[H]$, we write the feature matrix $\vPhib_{n,h}:=\vPhib(s_{n,h},a_{n,h})$
and the transition kernels $\PP_{n,h}^{\star}(\cdot):=\PP(\cdot|s_{n,h},a_{n,h},\thetab_{h}^{\star})$
and $\widehat{\PP}_{n,h}(\cdot):=\PP(\cdot|s_{n,h},a_{n,h},\widehat{\thetab}_{n-1,h})$.
For simplicity, we define $\Scal_{n,h}:=\Scal_{h}(s_{n,h},a_{n,h})$,
the set of states reachable from $(s_{n,h},a_{n,h})$ in step $h$
in episode $n$, and let $L_{n,h}:=L_{\Scal_{h}(s_{n,h},a_{n,h})}$
the corresponding log-sum-exp function. We also write $\Vb_{n,h}^{\star}:=\nabla^{2}L_{n,h}(\vPhib_{n,h}\thetab_{h}^{\star})$
and $\widehat{\Vb}_{n,h}:=\nabla^{2}L_{n,h}(\vPhib_{n,h}\widehat{\thetab}_{n,h})$.
Recall that $\Ab_{n,h}:=\sum_{\nu=1}^{n}\vPhib_{\nu,h}^{\top}\vPhib_{\nu,h}+B_{\vPhib}^{2}\Ib_{d}$
and $\Hb_{n,h}:=\sum_{\nu=1}^{n}\nabla^{2}\ell_{\nu,h}(\thetab_{h}^{\star})+B_{\vPhib}^{2}\Ib_{d}$.
By definition of $\widehat{Q}_{n,1}$, 
\begin{align*}
 & \widehat{Q}_{n,1}\big(s_{n,1},a_{n,1}\big)-Q_{n,1}\big(s_{n,1},a_{n,1};\pi_{n}\big)\\
 & \le\sum_{s^{\prime}\in\Scal_{n,1}}\widehat{V}_{n,2}(s^{\prime})\big(\widehat{\PP}(\eb_{s^{\prime}})-\PP_{n,h}^{\star}(\eb_{s^{\prime}})\big)\\
 & \quad+\sum_{s^{\prime}\in\Scal_{n,1}}\left(\widehat{V}_{n,2}(s^{\prime})-V_{n,2}(s^{\prime})\right)\PP_{n,h}^{\star}(\eb_{s^{\prime}})\\
 & \quad+\beta_{n-1}\Big\Vert\sum_{s^{\prime}\in\Scal_{n,1}}\widehat{V}_{n,2}(s^{\prime})\cdot\eb_{s^{\prime}}^{\top}\widehat{\Vb}_{n,1}\vPhib_{n,1}\Big\Vert_{\widehat{\Hb}_{n-1,1}^{-1}}\\
 & \quad+\beta_{n-1}^{2}\max_{s^{\prime}\in\Scal_{n,h}}\widehat{V}_{n,2}(s^{\prime})\max_{s^{\prime}\in\Scal_{n,1}}\left\Vert \vPhib_{n,1}^{\top}\eb_{s^{\prime}}\right\Vert _{\widehat{\Hb}_{n-1,1}^{-1}}^{2}.
\end{align*}
Using Lemma~\ref{lem:2nd_deriv_approx}, we can decompose the regret,
\begin{align*}
 & \widehat{Q}_{n,1}\big(s_{n,1},a_{n,1}\big)-Q_{n,1}\big(s_{n,1},a_{n,1};\pi_{n}\big)\\
 & \le\underbrace{\sum_{s^{\prime}\in\Scal_{n,1}}\widehat{V}_{n,2}(s^{\prime})\big(\widehat{\PP}(\eb_{s^{\prime}})-\PP_{n,h}^{\star}(\eb_{s^{\prime}})\big)}_{\text{Term I}}\\
 & +\sum_{s^{\prime}\in\Scal_{n,1}}\left(\widehat{V}_{n,2}(s^{\prime})-V_{n,2}(s^{\prime})\right)\PP_{n,h}^{\star}(\eb_{s^{\prime}})\\
 & +\beta_{n-1}\Big\Vert\sum_{s^{\prime}\in\Scal_{n,1}}\widehat{V}_{n,2}(s^{\prime})\cdot\eb_{s^{\prime}}^{\top}\Vb_{n,1}^{\star}\vPhib_{n,1}\Big\Vert_{\widehat{\Hb}_{n-1,1}^{-1}}\\
 & \quad+7\beta_{n-1}^{2}\max_{s^{\prime}\in\Scal_{n,h}}\widehat{V}_{n,2}(s^{\prime})\max_{s^{\prime}\in\Scal_{n,1}}\left\Vert \vPhib_{n,1}^{\top}\eb_{s^{\prime}}\right\Vert _{\widehat{\Hb}_{n-1,1}^{-1}}^{2}
\end{align*}
We use the second order approximation bound for the transition kernel
$\PP(\cdot|s,a,\thetab)$ (Lemma~\ref{lem:diff_second}) to bound
Term~I as follows: 
\begin{align*}
 & \text{Term I}\\
 & \leq\sum_{s^{\prime}\in\Scal_{n,1}}\widehat{V}_{n,2}(s^{\prime})\eb_{s^{\prime}}^{\top}\Vb_{n,1}^{\star}\vPhib_{n,1}(\widehat{\thetab}_{n-1,1}-\thetab_{1}^{\star})\\
 & +\!\!\max_{s^{\prime}\in\Scal_{n,1}}\!\!\widehat{V}_{n,2}(s^{\prime})\!\max_{s^{\prime}\in\Scal_{n,1}}\!\big(\eb_{s^{\prime}}^{\top}\vPhib_{n,1}(\widehat{\thetab}_{n-1,1}-\thetab_{1}^{\star})\big)^{2}\\
 & \le\Big\Vert\!\!\!\sum_{s^{\prime}\in\Scal_{n,1}}\!\!\widehat{V}_{n,2}(s^{\prime})\eb_{s^{\prime}}^{\top}\Vb_{n,1}^{\star}\vPhib_{n,1}\Big\Vert_{\widehat{\Hb}_{n-1,1}^{-1}}\!\!\!\!\!\!\!\|\widehat{\thetab}_{n-1,1}\!-\!\thetab_{1}^{\star}\|_{\widehat{\Hb}_{n-1,1}}\\
 & +\!\!\max_{s^{\prime}\in\Scal_{n,1}}\!\!\widehat{V}_{n,2}(s^{\prime})\!\!\max_{s^{\prime}\in\Scal_{n,1}}\!\!\!\Vert\vPhib_{n,1}^{\top}\eb_{s^{\prime}}\Vert_{\widehat{\Hb}_{n-1,1}^{-1}}^{2}\!\!\!\!\!\!\|\widehat{\thetab}_{n-1,1}\!-\!\thetab_{1}^{\star}\|_{\widehat{\Hb}_{n-1,1}}^{2}\!\!\!,
\end{align*}
where the last inequality holds by Cauchy-Schwarz inequality. Next,
the confidence ellipsoid (Theorem~\ref{thm:self}) implies that 
\begin{align*}
\text{Term I}\le & \beta_{n-1}\Big\Vert\!\!\!\sum_{s^{\prime}\in\Scal_{n,1}}\!\!\widehat{V}_{n,2}(s^{\prime})\eb_{s^{\prime}}^{\top}\Vb_{n,1}^{\star}\vPhib_{n,1}\Big\Vert_{\widehat{\Hb}_{n-1,1}^{-1}}\\
 & +\beta_{n-1}^{2}\max_{s^{\prime}\in\Scal_{n,1}}\!\!\widehat{V}_{n,2}(s^{\prime})\!\!\max_{s^{\prime}\in\Scal_{n,1}}\!\!\!\Vert\vPhib_{n,1}^{\top}\eb_{s^{\prime}}\Vert_{\widehat{\Hb}_{n-1,1}^{-1}}^{2}.
\end{align*}
It follows that 
\begin{align*}
 & \widehat{Q}_{n,1}\big(s_{n,1},a_{n,1}\big)-Q_{n,1}\big(s_{n,1},a_{n,1};\pi_{n}\big)\\
 & \le\sum_{s^{\prime}\in\Scal_{n,1}}\left(\widehat{V}_{n,2}(s^{\prime})-V_{n,2}(s^{\prime})\right)\PP_{n,h}^{\star}(\eb_{s^{\prime}})\\
 & \;\;+2\beta_{n-1}\Big\Vert\sum_{s^{\prime}\in\Scal_{n,1}}\widehat{V}_{n,2}(s^{\prime})\cdot\eb_{s^{\prime}}^{\top}\Vb_{n,1}^{\star}\vPhib_{n,1}\Big\Vert_{\widehat{\Hb}_{n-1,1}^{-1}}\\
 & \;\;+8\beta_{n-1}^{2}\max_{s^{\prime}\in\Scal_{n,h}}\widehat{V}_{n,2}(s^{\prime})\max_{s^{\prime}\in\Scal_{n,1}}\left\Vert \vPhib_{n,1}^{\top}\eb_{s^{\prime}}\right\Vert _{\widehat{\Hb}_{n-1,1}^{-1}}^{2}.
\end{align*}
Define $\Delta_{n,h}:=\sum_{s^{\prime}\in\Scal_{n,h}}\!\!\big(\widehat{V}_{n,h+1}(s^{\prime})\!-\!V_{n,h+1}(s^{\prime})\big)\PP_{n,h}^{\star}(\eb_{s^{\prime}})\!-\!\big(\widehat{V}_{n,h+1}(s_{n,h+1})\!-\!V_{n,h+1}(s_{n,h+1})\big)$.
Since $s_{n,h+1}$ is sampled from the distribution $\PP_{n,h}^{\star}(\cdot)$,
it follows that $\Delta_{n,h}=\EE\big[\widehat{V}_{n,h+1}(s_{n,h+1})\!-\!V_{n,h+1}(s_{n,h+1})\mid s_{n,h},a_{n,h}\big]\!-\!\big(\widehat{V}_{n,h+1}(s_{n,h+1})\!-\!V_{n,h+1}(s_{n,h+1})\big)$,
which is a martingale difference. Then we have the following decomposition:
\begin{align*}
 & \sum_{s^{\prime}\in\Scal_{n,1}}\left(\widehat{V}_{n,2}(s^{\prime})-V_{n,2}(s^{\prime})\right)\PP_{n,h}^{\star}(\eb_{s^{\prime}})\\
 & \le\widehat{Q}_{n,2}(s_{n,2},a_{n,2})-Q_{2}(s_{n,2},a_{n,2};\pi_{n})+\Delta_{n,2},
\end{align*}
and the regret at episode $n\in[K]$, 
\begin{align*}
 & \widehat{Q}_{n,1}\big(s_{n,1},a_{n,1}\big)-Q_{n,1}\big(s_{n,1},a_{n,1};\pi_{n}\big)\\
 & \le\Delta_{n,2}+\widehat{Q}_{n,2}(s_{n,2},a_{n,2})-Q_{2}(s_{n,2},a_{n,2};\pi_{n})\\
 & \;\;+2\beta_{n-1}\Ubrace{\Big\Vert\sum_{s^{\prime}\in\Scal_{n,1}}\widehat{V}_{n,2}(s^{\prime})\cdot\eb_{s^{\prime}}^{\top}\Vb_{n,1}^{\star}\vPhib_{n,1}\Big\Vert_{\widehat{\Hb}_{n-1,1}^{-1}}}{:=\alpha_{n,1}}\\
 & \;\;+8\beta_{n-1}^{2}H\Ubrace{\max_{s^{\prime}\in\Scal_{n,1}}\big\Vert\vPhib_{n,1}^{\top}\eb_{s^{\prime}}\big\Vert_{\widehat{\Hb}_{n-1,1}^{-1}}^{2}}{:=\chi_{n,1}},
\end{align*}
where the last inequality uses $\widehat{V}_{n,h}(s)\le H$ for all
$s\in\Scal_{n,h}$. Recursively, the regret, 
\begin{align*}
 & \Reg(\Mcal,\Pi_{K})\\
 & \le\sum_{n=1}^{K}\widehat{Q}_{n,1}\big(s_{n,1},a_{n,1}\big)-Q_{n,1}\big(s_{n,1},a_{n,1};\pi_{n}\big)\\
 & \le\sum_{n=1}^{K}\sum_{h=1}^{H}\big(\Delta_{n,h}+2\beta_{n-1}\alpha_{n,h}+8H\beta_{n-1}^{2}\chi_{n,h}\big)
\end{align*}
It is easy to verify $\EE[\Delta_{n,h}^{2}|s_{n,h},a_{n,h}]\le H^{2}\sigma_{n,h}^{2}$
and the sum $\sum_{n=1}^{K}\sum_{h=1}^{H}\Delta_{n,h}=O(H\sqrt{\sum_{n=1}^{K}\sum_{h=1}^{H}\sigma_{n,h}^{2}})$
by the Freedman's inequality (Lemma \ref{lem:freedman}). Using simple
linear algebra manipulation, 
\[
\alpha_{n,h}\!\!\le\!\!\!\sqrt{\Tr\big(\widehat{\Hb}_{n-1,h}^{-1}\!\!\nabla^{2}\ell_{n-1,h}(\thetab_{h}^{\star})\big)}\big\Vert\!\!\!\!\sum_{\tilde{s}\in\Scal_{n,h}}\!\!\!\widehat{V}_{n,h+1}\!(\tilde{s})\eb_{\tilde{s}}\big\Vert_{\Vb_{n,h}^{\star}}\!\!\!\!.
\]
By the Hessian bound (Lemma~\ref{lem:Hessian_bound}), we have $H_{n-1,h}^{\star}\preceq(1-e^{-1})^{-1}\widehat{H}_{n-1,h}$
and 
\begin{align*}
 & \sum_{n=1}^{K}\sum_{h=1}^{H}\Tr\big(\widehat{\Hb}_{n-1,h}^{-1}\!\!\nabla^{2}\ell_{n-1,h}(\thetab_{h}^{\star})\big)\\
 & \le(1-e^{-1})^{-1}\sum_{n=1}^{K}\sum_{h=1}^{H}\Tr\big((\Hb_{n-1,h}^{\star})^{-1}\nabla^{2}\ell_{n-1,h}(\thetab_{h}^{\star})\big)\\
 & \le\frac{2dH}{1-e^{-1}}\log(K+1),
\end{align*}
where the last inequality uses the matrix potential lemma (Lemma~\ref{lem:matrix_potent})
for each $h\in[H]$. By Cauchy-Schwarz inequality, 
\begin{align*}
 & \sum_{n=1}^{K}\sum_{h=1}^{H}2\beta_{n-1}\alpha_{n,h}\\
 & \le\!\!\frac{2\beta_{K-1}\!\sqrt{2dH\log\!\frac{K+1}{d}}}{\sqrt{1-e^{-1}}}\!\!\!\sqrt{\!\!\sum_{n=1}^{K}\sum_{h=1}^{H}\Big\Vert\!\!\!\sum_{\tilde{s}\in\Scal_{n,h}}\!\!\widehat{V}_{n,2}(\tilde{s})\eb_{\tilde{s}}\Big\Vert_{\Vb_{n,h}^{\star}}^{2}}
\end{align*}
With the identity for the covariance of the multinomial distribution
(Proposition \ref{prop:mean_var}), we obtain the following conditional
variance, 
\[
\Big\Vert\!\!\!\sum_{\tilde{s}\in\Scal_{n,h}}\!\!\widehat{V}_{n,h+1}(\tilde{s})\eb_{\tilde{s}}\Big\Vert_{\Vb_{n,h}^{\star}}^{2}\!\!=\!\!\VV\big[\widehat{V}_{n,h+1}(s_{n,h+1})\big|s_{n,h},a_{n,h}\big].
\]
Note that both $\widehat{V}_{n,h}$ and $V_{n,h}$ is in $[0,H]$.
Thus, $|\Delta_{n,2}|\le2H$, Here, we use a novel bound for the sum
of variance of the estimated value function $\widehat{V}_{n,h}$.
Lemma~\ref{lem:variance_bound} implies 
\[
\begin{split} & \sum_{k=1}^{K}\sum_{h=1}^{H}2\beta_{k-1}\alpha_{k,h}\\
 & \le\!\!\frac{2\beta_{K-1}\!\sqrt{2dH\log\!\frac{K+1}{d}}}{\sqrt{1-e^{-1}}}\!\!\!\sqrt{\sum_{k=1}^{K}\sum_{h=1}^{H}\Big\Vert\!\!\!\sum_{\tilde{s}\in\Scal_{k,h}}\!\!\widehat{V}_{k,2}(\tilde{s})\eb_{\tilde{s}}\Big\Vert_{\Vb_{k,h}^{\star}}^{2}}\\
 & \le\!\!\frac{2\beta_{K-1}\!\sqrt{2dH\log\!\frac{K+1}{d}}}{\sqrt{1-e^{-1}}}\!\!\!\sqrt{12eH\sum_{k=1}^{K}\sum_{h=1}^{H}\sigma_{k,h}^{2}\!+\!1200eH^{4}\kappa^{-2}\beta_{K-1}^{4}\big(d\log\frac{K+1}{d}\!+\!3\log\frac{1}{\delta}\big)}.
\end{split}
\]
Because $\widehat{H}_{n,h}\succeq(1-e^{-1})\Hb_{n,h}^{\star}\succeq\kappa\Ab_{n,h}$,
we obtain 
\[
\chi_{n,h}\le\kappa^{-1}\max_{s^{\prime}\in\Scal_{n,h}}\big\Vert\vPhib_{n,h}^{\top}\eb_{s^{\prime}}\big\Vert_{\Ab_{n-1,h}}^{2}.
\]
Applying the elliptical potential lemma for feature matrices (Lemma~\ref{lem:vPhib_potential})
for each $h\in[H]$, 
\[
\sum_{n=1}^{K}\sum_{h=1}^{H}\chi_{n,h}\le2\kappa^{-1}Hd\log\frac{K+1}{d}.
\]
Thus,

\[
\sum_{k=1}^{K}\sum_{h=1}^{H}8H\beta_{k-1}^{2}\chi_{k,h}\le16H^{2}\beta_{K-1}^{2}d\log\frac{K+1}{d}.
\]
By the self-normalized bound (Lemma \ref{lem:faury_self}), with probability
at least $1-\delta$ 
\[
\begin{aligned}
&\Big|\sum_{k=1}^{K}\sum_{h=1}^{H}\Delta_{k,h+1}\Big| \\
&\le\Big(\frac{H}{2}+2\sqrt{H}\log\frac{\sqrt{H+\sum_{k=1}^{K}\sum_{h=1}^{H}\sigma_{k,h}^{2}}}{\delta}+2\sqrt{H}\log2\Big)\sqrt{1+\sum_{k=1}^{K}\sum_{h=1}^{H}\sigma_{k,h}^{2}}.\\
 & \le\Big(\frac{H}{2}+2H\log\frac{K+1}{\delta}+2\sqrt{H}\log2\Big)\sqrt{1+\sum_{k=1}^{K}\sum_{h=1}^{H}\sigma_{k,h}^{2}}.
\end{aligned}
\]
Thus, 
\[
\begin{split}
&\Reg(K)\le  \Big(\frac{H}{2}+2H\log\frac{K+1}{\delta}+2\sqrt{H}\log2\Big)\sqrt{1+\sum_{k=1}^{K}\sum_{h=1}^{H}\sigma_{k,h}^{2}}\\
 & +\frac{2\beta_{K-1}\!\sqrt{2dH\log\!\frac{K+1}{d}}}{\sqrt{1-e^{-1}}}\!\!\!\sqrt{12eH\sum_{k=1}^{K}\sum_{h=1}^{H}\sigma_{k,h}^{2}\!+\!1200eH^{4}\kappa^{-2}\beta_{K-1}^{4}\big(d\log\frac{K+1}{d}\!+\!3\log\frac{1}{\delta}\big)}\\
 & +16H^{2}\beta_{K-1}^{2}d\log\frac{K+1}{d},
\end{split}
\]
which completes the proof. 
\end{proof}
Because $\sigma_{k,h}^{2}\le1$, the variance adaptive regret bound
implies the following the problem-dependent regret bound. 

\begin{corollary}[Problem-dependent regret bound of \UCB\ ] \label{cor:regret_bound}
If the algorithm $\Pi_{K}$ is set to \UCB\, then for any MNL-MDP
$\Mcal$, with probability at least $1-5\delta$, 
\begin{align*}
\Reg_{\Pi_{K}}(\Mcal_{\Thetab^{\star}})=O\bigg( & dH^{3/2}\!\!\sqrt{K\log\frac{KH}{\delta}}\log K\!+\!\kappa^{-1}d^{\frac{5}{2}}H^{\frac{5}{2}}\log^{\frac{3}{2}}\!K\bigg)
\end{align*}
\end{corollary}

\subsection{Proof of Lemma~\ref{lem:variance_bound}}
\TotalVariance*

\begin{proof}
Let us write $\vPhib_{k,h}:=\vPhib(s_{k,h},a_{k,h})$. Define $\Fcal_{k,h}$
denote the sigma-algebra generated by $\cup_{\nu=1}^{k-1}\cup_{h=1}^{H}\{s_{\nu,h},a_{\nu,h}\}$
and $\cup_{h^{\prime}=1}^{h}\{s_{k,h^{\prime}},a_{k,h^{\prime}}\}$.
For simplicity, we use $\EE_{k,h}:=\EE[\cdot|\Fcal_{k,h}]$ and $\widehat{\EE}_{k,h}:=[\cdot|\Fcal_{k,h},\widehat{\thetab}_{k-1,h}]$
for each $k\in[K]$ and $h\in[H]$. Because $0\le\sum_{h=1}^{H}\VV\Big[\widehat{V}_{k,h+1}(s_{k,h+1})\Big|s_{k,h},a_{k,h}\Big]\le H^{3}$,
by Chernoff bound (Lemma \ref{lem:chernoff}), 
\begin{equation}
\begin{aligned} & \sum_{k=1}^{K}\sum_{h=1}^{H}\VV\Big[\widehat{V}_{k,h+1}(s_{k,h+1})\Big|s_{k,h},a_{k,h}\Big]\\
 & \le(e-1)\sum_{k=1}^{K}\EE_{k,1}\bigg[\sum_{h=1}^{H}\VV\Big[\widehat{V}_{k,h+1}(s_{k,h+1})\Big|s_{k,h},a_{k,h}\Big]\bigg]+H^{3}\log\frac{1}{\delta}.
\end{aligned}
\label{eq:var_bound_chernoff}
\end{equation}
Let us write the martingale difference $\Delta_{k,h}:=\widehat{V}_{k,h+1}(s_{k,h+1})-\mathbb{E}_{k,h}\big[\widehat{V}_{k,h+1}(s_{k,h+1})\big]$.
It follows that $\EE_{k,h}[\Delta_{k,h}]=0$ and for $h\neq h'$ we
have $\EE_{k,1}[\Delta_{k,h}\Delta_{k,h'}]=0$. Then, 
\[
\begin{split}\EE_{k,1}\bigg[\sum_{h=1}^{H}\VV\Big[\widehat{V}_{k,h+1}(s_{k,h+1})\Big|s_{k,h},a_{k,h}\Big]\bigg]=\EE_{k,1}\sum_{h=1}^{H}\Delta_{k,h}^{2}=\EE_{k,1}\bigg[\Big(\sum_{h=1}^{H}\Delta_{k,h}\Big)^{2}\bigg].\end{split}
\]
Because $\widehat{V}_{k,H+1}(s)=0$, for all $s\in\Scal$, 
\[
\begin{split} & \sum_{h=1}^{H}\Delta_{k,h}\\
 & =\sum_{h=1}^{H}\widehat{V}_{k,h+1}(s_{k,h+1})-\widehat{V}_{k,h}(s_{k,h})+\widehat{V}_{k,h}(s_{k,h})-\EE_{k,h}\widehat{V}_{k,h+1}(s_{k,h+1})\\
 & =-\widehat{V}_{k,1}(s_{k,1})+\sum_{h=1}^{H}\big(\widehat{V}_{k,h}(s_{k,h})-\EE_{k,h}\widehat{V}_{k,h+1}(s_{k,h+1})\big)
\end{split}
\]
By definition of $\widehat{V}_{k,h}(s_{k,h})=\widehat{Q}_{k,h}(s_{k,h},a_{k,h})$:
\[
\begin{split}\widehat{V}_{k,h}(s_{k,h})= & r(s_{k,h},a_{k,h})+\EE_{k,h}\big[\widehat{V}_{k,h+1}(s_{k,h+1})\big]\\
 & \widehat{\EE}_{k,h}\big[\widehat{V}_{k,h+1}(s_{k,h+1})\big]-\EE_{k,h}\big[\widehat{V}_{k,h+1}(s_{k,h+1})\big]\\
 & +\beta_{k-1}\bigg\Vert\sum_{s^{\prime}\in\Scal_{k,h}}\widehat{V}_{k,h+1}(s^{\prime})\cdot\vPhib_{k,h}^{\top}\nabla^{2}L_{k,h}\big(\vPhib_{k,h}\widehat{\thetab}_{k-1,h}\big)\eb_{s^{\prime}}\bigg\Vert_{\widehat{\Hb}_{k-1,h}^{-1}}\\
 & +2\beta_{k-1}^{2}\sup_{s^{\prime}\in\Scal_{k,h}}\widehat{V}_{k,h+1}(s^{\prime})\max_{s^{\prime}\in\Scal_{k,h}}\bigg\Vert\vPhib_{k,h}^{\top}\eb_{s^{\prime}}\bigg\Vert_{\widehat{\Hb}_{k-1,h}^{-1}}^{2}.
\end{split}
\]
Define the martingale difference $\rho_{k,h}:=r(s_{k,h},a_{k,h})-\EE_{k,h-1}[r(s_{k,h},a_{k,h})]$.
Then, 
\[
\begin{split}  \sum_{h=1}^{H}\Delta_{k,h}
 &=\sum_{h=1}^{H}\big(r(s_{k,h},a_{k,h})-\rho_{k,h}\big)-\widehat{V}_{k,1}(s_{k,1})\\
 &\quad+\sum_{h=1}^{H}\big(\widehat{V}_{k,h}(s_{k,h})-\EE_{k,h}\widehat{V}_{k,h+1}(s_{k,h+1})-r(s_{k,h},a_{k,h})+\rho_{k,h}\big)
\end{split}
\]
By triangular inequality and Cauchy-Schwartz inequality, 
\begin{equation}
\begin{aligned} & \big(\sum_{h=1}^{H}\Delta_{k,h}\big)^{2}\\
 & \le2\Big(\sum_{h=1}^{H}\big(r(s_{k,h},a_{k,h})-\rho_{k,h}\big)-\widehat{V}_{k,1}(s_{k,1})\Big)^{2}\\
 & \quad+2\bigg(\sum_{h=1}^{H}\big(\widehat{V}_{k,h}(s_{k,h})-\EE_{k,h}\widehat{V}_{k,h+1}(s_{k,h+1})-r(s_{k,h},a_{k,h})+\rho_{k,h}\big)\bigg)^{2}\\
 & \le2\Big(\sum_{h=1}^{H}\big(r(s_{k,h},a_{k,h})-\rho_{k,h}\big)-\widehat{V}_{k,1}(s_{k,1})\Big)^{2}\\
 & \quad+2H\sum_{h=1}^{H}\bigg(\widehat{V}_{k,h}(s_{k,h})-\EE_{k,h}\widehat{V}_{k,h+1}(s_{k,h+1})-r(s_{k,h},a_{k,h})+\rho_{k,h}\bigg)^{2}
\end{aligned}
\label{eq:var_bound_decomp}
\end{equation}
By definition of $\widehat{V}_{k,h}(s_{k,h})$: 
\[
\begin{split} & \widehat{V}_{k,h}(s_{k,h}):=r(s_{k,h},a_{k,h})+\widehat{\EE}_{k,h}\big[\widehat{V}_{k,h+1}(s_{k,h+1})\big]\\
 & \quad+\beta_{k-1}\bigg\Vert\sum_{s^{\prime}\in\Scal_{k,h}}\widehat{V}_{k,h+1}(s^{\prime})\cdot\vPhib_{k,h}^{\top}\nabla^{2}L_{k,h}\big(\vPhib_{k,h}\widehat{\thetab}_{k-1,h}\big)\eb_{s^{\prime}}\bigg\Vert_{\widehat{\Hb}_{k-1,h}^{-1}}\\
 & \quad+2\beta_{k-1}^{2}\sup_{s^{\prime}\in\Scal_{k,h}}\widehat{V}_{k,h+1}(s^{\prime})\max_{s^{\prime}\in\Scal_{k,h}}\bigg\Vert\vPhib_{k,h}^{\top}\eb_{s^{\prime}}\bigg\Vert_{\widehat{\Hb}_{k-1,h}^{-1}}^{2},
\end{split}
\]
we obtain 
\begin{equation}
\begin{aligned} & \abs{\widehat{V}_{k,h}(s_{k,h})-\EE_{k,h}\widehat{V}_{k,h+1}(s_{k,h+1})-r(s_{k,h},a_{k,h})+\rho_{k,h}}\\
 & \le\abs{\rho_{k,h}}+\abs{\widehat{\EE}_{k,h}\widehat{V}_{k,h+1}(s_{k,h+1})-\EE_{k,h}\widehat{V}_{k,h+1}(s_{k,h+1})}\\
 & \quad+\beta_{k-1}\bigg\Vert\sum_{s^{\prime}\in\Scal_{k,h}}\widehat{V}_{k,h+1}(s^{\prime})\cdot\vPhib_{k,h}^{\top}\nabla^{2}L_{k,h}\big(\vPhib_{k,h}\widehat{\thetab}_{k-1,h}\big)\eb_{s^{\prime}}\bigg\Vert_{\widehat{\Hb}_{k-1,h}^{-1}}\\
 & \quad+2\beta_{k-1}^{2}\sup_{s^{\prime}\in\Scal_{k,h}}\widehat{V}_{k,h+1}(s^{\prime})\max_{s^{\prime}\in\Scal_{k,h}}\bigg\Vert\vPhib_{k,h}^{\top}\eb_{s^{\prime}}\bigg\Vert_{\widehat{\Hb}_{k-1,h}^{-1}}^{2}.
\end{aligned}
\label{eq:var_bound_value_decomp}
\end{equation}
By the bound for second order confidence bound (Lemma \ref{lem:2nd_deriv_approx}),
\[
\begin{split} & \beta_{k-1}\bigg\Vert\sum_{s^{\prime}\in\Scal_{k,h}}\widehat{V}_{k,h+1}(s^{\prime})\cdot\vPhib_{k,h}^{\top}\nabla^{2}L_{k,h}\big(\vPhib_{k,h}\widehat{\thetab}_{k-1,h}\big)\eb_{s^{\prime}}\bigg\Vert_{\widehat{\Hb}_{k-1,h}^{-1}}\\
 & +2\beta_{k-1}^{2}\sup_{s^{\prime}\in\Scal_{k,h}}\widehat{V}_{k,h+1}(s^{\prime})\max_{s^{\prime}\in\Scal_{k,h}}\bigg\Vert\vPhib_{k,h}^{\top}\eb_{s^{\prime}}\bigg\Vert_{\widehat{\Hb}_{k-1,h}^{-1}}^{2}\\
 & \le\beta_{k-1}\bigg\Vert\sum_{s^{\prime}\in\Scal_{k,h}}\widehat{V}_{k,h+1}(s^{\prime})\cdot\vPhib_{k,h}^{\top}\nabla^{2}L_{k,h}\big(\vPhib(s_{k,h},a_{k,h})\thetab_{h}^{\star}\big)\eb_{s^{\prime}}\bigg\Vert_{(\Hb_{k-1,h}^{\star})^{-1}}\\
 & +8\kappa^{-1}H\beta_{k-1}^{2}\max_{s^{\prime}\in\Scal_{k,h}}\bigg\Vert\vPhib_{k,h}^{\top}\eb_{s^{\prime}}\bigg\Vert_{\Ab_{k-1,h}^{-1}}^{2}.
\end{split}
\]
Because $\widehat{V}_{k,h+1}(s)\le H$, 
\[
\begin{split} & \bigg\Vert\sum_{s^{\prime}\in\Scal_{k,h}}\widehat{V}_{k,h+1}(s^{\prime})\cdot\vPhib_{k,h}^{\top}\nabla^{2}L_{k,h}\big(\vPhib_{k,h}\thetab_{h}^{\star}\big)\eb_{s^{\prime}}\bigg\Vert_{(\Hb_{k-1,h}^{\star})^{-1}}\\
 & =\norm{\vPhib_{k,h}^{\top}\nabla^{2}L_{k,h}\big(\vPhib_{k,h}\thetab_{h}^{\star}\big)\Big(\sum_{s^{\prime}\in\Scal_{k,h}}\widehat{V}_{k,h+1}(s^{\prime})\eb_{s^{\prime}}\Big)}_{(\Hb_{k-1,h}^{\star})^{-1}}\\
 & \le H\max_{\xb:\|\xb\|_{\infty}\le1}\norm{\vPhib_{k,h}^{\top}\nabla^{2}L_{k,h}\big(\vPhib_{k,h}\thetab_{h}^{\star}\big)\xb}_{(\Hb_{k-1,h}^{\star})^{-1}}.
\end{split}
\]
By the definition of the Fisher information lower bound, 
\[
\begin{split} & H\max_{\xb:\|\xb\|_{\infty}\le1}\norm{\vPhib_{k,h}^{\top}\nabla^{2}L_{k,h}\big(\vPhib_{k,h}\thetab_{h}^{\star}\big)\xb}_{(\Hb_{k-1,h}^{\star})^{-1}}\\
 & \le H\kappa^{-1/2}\max_{\xb:\|\xb\|_{\infty}\le1}\norm{\vPhib_{k,h}^{\top}\nabla^{2}L_{k,h}\big(\vPhib_{k,h}\thetab_{h}^{\star}\big)\xb}_{\Ab_{k-1,h}^{-1}}\\
 & =H\kappa^{-1/2}\max_{\xb:\|\xb\|_{\infty}\le1}\sqrt{\xb^{\top}\nabla^{2}L_{k,h}\big(\vPhib_{k,h}\thetab_{h}^{\star}\big)\vPhib_{k,h}\Ab_{k-1,h}^{-1}\vPhib_{k,h}^{\top}\nabla^{2}L_{k,h}\big(\vPhib_{k,h}\thetab_{h}^{\star}\big)\xb}\\
 & \le H\kappa^{-1/2}\max_{\xb:\|\xb\|_{\infty}\le1}\|\xb\|_{\nabla^{2}L_{k,h}\big(\vPhib_{k,h}\thetab_{h}^{\star}\big)}\\
 & \cdot\lambda_{\max}\Big(\nabla^{2}L_{k,h}\big(\vPhib_{k,h}\thetab_{h}^{\star}\big)^{1/2}\vPhib_{k,h}\Ab_{k-1,h}^{-1}\vPhib_{k,h}^{\top}\nabla^{2}L_{k,h}\big(\vPhib_{k,h}\thetab_{h}^{\star}\big)^{1/2}\Big).
\end{split}
\]
Because $\lambda_{\max}(\Mb\Mb^{\top})=\lambda_{\max}(\Mb^{\top}\Mb)$
for any matrix $\Mb$, we obtain 
\[
\begin{split} & \lambda_{\max}\Big(\nabla^{2}L_{k,h}(\vPhib_{k,h}\thetab_{h}^{\star})^{1/2}\vPhib_{k,h}\Ab_{k-1,h}^{-1}\vPhib_{k,h}^{\top}\nabla^{2}L_{k,h}(\vPhib_{k,h}\thetab_{h}^{\star})^{1/2}\Big)\\
 & =\lambda_{\max}\Big(\Ab_{k-1,h}^{-1/2}\vPhib_{k,h}^{\top}\nabla^{2}L_{k,h}(\vPhib_{k,h}\thetab_{h}^{\star})\vPhib_{k,h}\Ab_{k-1,h}^{-1/2}\Big)\\
 & \le\lambda_{\max}\Big(\Ab_{k-1,h}^{-1/2}\vPhib_{k,h}^{\top}\text{diag}\Big(\nabla L_{k,h}(\vPhib_{k,h}\thetab_{h}^{\star})\Big)\vPhib_{k,h}\Ab_{k-1,h}^{-1/2}\Big),
\end{split}
\]
where the last inequality uses the fact that $\nabla^{2}L_{k,h}(\vPhib_{k,h}\thetab_{h}^{\star})\preceq\text{diag}\big(\nabla L_{k,h}(\vPhib_{k,h}\thetab_{h}^{\star})\big)$.
Because $\text{diag}\big(\nabla L_{k,h}(\vPhib_{k,h}\thetab_{h}^{\star})\big)=\sum_{s\in\Scal_{k,h}}\eb_{s}^{\top}\nabla L_{k,h}(\vPhib_{k,h}\thetab_{h}^{\star})\eb_{s}\eb_{s}^{\top}$,
\[
\begin{split} & \le\lambda_{\max}\Big(\Ab_{k-1,h}^{-1/2}\vPhib_{k,h}^{\top}\text{diag}\Big(\nabla L_{k,h}(\vPhib_{k,h}\thetab_{h}^{\star})\Big)\vPhib_{k,h}\Ab_{k-1,h}^{-1/2}\Big)\\
 & \le\sum_{s\in\Scal_{k,h+1}}\eb_{s}^{\top}\nabla L_{k,h}(\vPhib_{k,h}\thetab_{h}^{\star})\lambda_{\max}\Big(\Ab_{k-1,h}^{-1/2}\vPhib_{k,h}^{\top}\eb_{s}\eb_{s}^{\top}\vPhib_{k,h}\Ab_{k-1,h}^{-1/2}\Big)\\
 & =\sum_{s\in\Scal_{k,h+1}}\eb_{s}^{\top}\nabla L_{k,h}(\vPhib_{k,h}\thetab_{h}^{\star})\Big\Vert\vPhib_{k,h}^{\top}\eb_{s}\Big\Vert_{\Ab_{k-1,h}^{-1}}^{2}\\
 & \le\max_{s\in\Scal_{k,h+1}}\Big\Vert\vPhib_{k,h}^{\top}\eb_{s}\Big\Vert_{\Ab_{k-1,h}^{-1}}^{2},
\end{split}
\]
where the last inequality uses H{ö}lder's inequality and the fact
that $\|\nabla L_{k,h}(\vPhib_{k,h}\thetab_{h}^{\star})\|_{1}=1$.
Plugging in \eqref{eq:var_bound_value_decomp}, 
\[
\begin{split} & \bigg|\widehat{V}_{k,h}(s_{k,h})-\EE_{k,h}\big[\widehat{V}_{k,h+1}(s_{k,h+1})\big]-r(s_{k,h},a_{k,h})+\rho_{k,h}\bigg|\\
 & \le\Big|\rho_{k,h}\Big|+\bigg|\widehat{\EE}_{k,h}\widehat{V}_{k,h+1}(s_{k,h+1})-\EE_{k,h}\widehat{V}_{k,h+1}(s_{k,h+1})\bigg|\\
 & \quad+H\kappa^{-1/2}\beta_{k-1}\max_{s\in\Scal_{k,h+1}}\Big\Vert\vPhib_{k,h}^{\top}\eb_{s}\Big\Vert_{\Ab_{k-1,h}^{-1}}+8\kappa^{-1}H\beta_{k-1}^{2}\max_{s^{\prime}\in\Scal_{k,h+1}}\bigg\Vert\vPhib_{k,h}^{\top}\eb_{s^{\prime}}\bigg\Vert_{\Ab_{k-1,h}^{-1}}^{2}.
\end{split}
\]
For simplicity, we define $U_{k,h}:=\beta_{k-1}\kappa^{-1/2}\max_{s^{\prime}\in\Scal_{k,h}}\|\vPhib_{k,h}^{\top}\eb_{s^{\prime}}\|_{\Ab_{k-1,h}^{-1}}$
. Then, the above inequality is 
\[
\begin{split} & \bigg|\widehat{V}_{k,h}(s_{k,h})-\EE_{k,h}\big[\widehat{V}_{k,h+1}(s_{k,h+1})\big]-r(s_{k,h},a_{k,h})+\rho_{k,h}\bigg|\\
 & \le\Big|\rho_{k,h}\Big|+\bigg|\widehat{\EE}_{k,h}\widehat{V}_{k,h+1}(s_{k,h+1})-\EE_{k,h}\widehat{V}_{k,h+1}(s_{k,h+1})\bigg|\\
 & \quad+HU_{k,h}+8HU_{k,h}^{2}.
\end{split}
\]
By the first order multinomial kernel difference bound (Lemma \ref{lem:diff_first}),
\[
\begin{split} & \bigg|\widehat{\EE}_{k,h}\widehat{V}_{k,h+1}(s_{k,h+1})-\EE_{k,h}\widehat{V}_{k,h+1}(s_{k,h+1})\bigg|\\
 & =\bigg|\sum_{s^{\prime}\in\Scal_{k,h}}\widehat{V}_{k,h+1}(s^{\prime})\big(\PP(\eb_{s^{\prime}}|s_{k,h},a_{k,h},\widehat{\thetab}_{k-1,h})-\PP(\eb_{s^{\prime}}|s_{k,h},a_{k,h},\thetab_{\star,h})\big)\bigg|\\
 & \le HU_{k,h}.
\end{split}
\]
Thus, 
\[
\bigg|\widehat{V}_{k,h}(s_{k,h})-\EE_{k,h}\widehat{V}_{k,h+1}(s_{k,h+1})-r(s_{k,h},a_{k,h})+\rho_{k,h}\bigg|\le\Big|\rho_{k,h}\Big|+2HU_{k,h}+8HU_{k,h}^{2}.
\]
Because $\Ab_{k-1,h}\succeq B_{\vPhib}^{2}\Ib_{d}$, we obtain, 
\[
\begin{split}U_{k,h} & :=\beta_{k-1}\kappa^{-1/2}\max_{s^{\prime}\in\Scal_{k,h}}\|\vPhib_{k,h}^{\top}\eb_{s^{\prime}}\|_{\Ab_{k-1,h}^{-1}}\\
 & \le\beta_{k-1}\kappa^{-1/2}B_{\vPhib}^{-1}\max_{s^{\prime}\in\Scal_{k,h}}\|\vPhib_{k,h}^{\top}\eb_{s^{\prime}}\|_{2}\\
 & \le\beta_{k-1}\kappa^{-1/2}.
\end{split}
\]
Thus, 
\begin{equation}
\begin{aligned} & \bigg|\widehat{V}_{k,h}(s_{k,h})-\EE_{k,h}\widehat{V}_{k,h+1}(s_{k,h+1})-r(s_{k,h},a_{k,h})+\rho_{k,h}\bigg|\\
 & \le\Big|\rho_{k,h}\Big|+2H(1+4\beta_{k-1}\kappa^{-1/2})U_{k,h}\\
 & \le\Big|\rho_{k,h}\Big|+10H\beta_{k-1}\kappa^{-1/2}U_{k,h},
\end{aligned}
\label{eq:var_bound_value_diff}
\end{equation}
and the squared term, 
\[
\begin{split} & \bigg(\widehat{V}_{k,h}(s_{k,h})-\EE_{k,h}\widehat{V}_{k,h+1}(s_{k,h+1})-r(s_{k,h},a_{k,h})+\rho_{k,h}\bigg)^{2}\\
 & \le\bigg(\Big|\rho_{k,h}\Big|+10H\beta_{k-1}\kappa^{-1/2}U_{k,h}\bigg)^{2}\\
 & \le2\rho_{k,h}^{2}+200H^{2}\beta_{k-1}^{2}\kappa^{-1}U_{k,h}^{2},
\end{split}
\]
where the last inequality holds by $(a+b)^{2}\le2a^{2}+2b^{2}$ for
$a,b\in\RR$. Plugging in \eqref{eq:var_bound_decomp}, 
\begin{equation}
\begin{aligned} & \EE_{k,1}\Big[\Big(\sum_{h=1}^{H}\widehat{V}_{k,h+1}(s_{k,h+1})-\EE_{k,h}\widehat{V}_{k,h+1}(s_{k,h+1})\Big)^{2}\Big]\\
 & \le2\EE_{k,1}\bigg[\Big(\sum_{h=1}^{H}\big(r(s_{k,h},a_{k,h})-\rho_{k,h}\big)-\widehat{V}_{k,1}(s_{k,1})\Big)^{2}\bigg]\\
 & \quad+2H\EE_{k,1}\bigg[\sum_{h=1}^{H}\bigg(\widehat{V}_{k,h}(s_{k,h})-\EE_{k,h}\widehat{V}_{k,h+1}(s_{k,h+1})-r(s_{k,h},a_{k,h})+\rho_{k,h}\bigg)^{2}\bigg]\\
 & \le2\EE_{k,1}\bigg[\Big(\sum_{h=1}^{H}\big(r(s_{k,h},a_{k,h})-\rho_{k,h}\big)-\widehat{V}_{k,1}(s_{k,1})\Big)^{2}\bigg]\\
 & \quad+2H\EE_{k,1}\bigg[\sum_{h=1}^{H}2\rho_{k,h}^{2}+200H^{2}\beta_{k-1}^{2}\kappa^{-1}U_{k,h}^{2}\bigg],
\end{aligned}
\label{eq:var_bound_first_decomp}
\end{equation}
where the last inequality uses Jensen's inequality. In the first term,
by \eqref{eq:var_bound_value_diff}, 
\begin{align*}
 & \Big(\sum_{h=1}^{H}\big(r(s_{k,h},a_{k,h})-\rho_{k,h}\big)-\widehat{V}_{k,1}(s_{k,1})\Big)^{2}\\
 & =\Big(\sum_{h=2}^{H}\big(r(s_{k,h},a_{k,h})-\rho_{k,h}\big)-\EE_{k,1}\widehat{V}_{k,2}(s_{k,2})+\EE_{k,1}\widehat{V}_{k,2}(s_{k,2})-\widehat{V}_{k,1}(s_{k,1})+r(s_{k,1},a_{k,1})-\rho_{k,1}\Big)^{2}\\
 & \le\Big(\Big|\sum_{h=2}^{H}\big(r(s_{k,h},a_{k,h})-\rho_{k,h}\big)-\EE_{k,1}\widehat{V}_{k,2}(s_{k,2})\Big|+\Big|\EE_{k,1}\widehat{V}_{k,2}(s_{k,2})-\widehat{V}_{k,1}(s_{k,1})+r(s_{k,1},a_{k,1})-\rho_{k,1}\Big)^{2}\\
 & \le\Big(\Big|\sum_{h=2}^{H}\big(r(s_{k,h},a_{k,h})-\rho_{k,h}\big)-\EE_{k,1}\widehat{V}_{k,2}(s_{k,2})\Big|+\Big|\rho_{k,1}\Big|+10H\beta_{k-1}\kappa^{-1/2}U_{k,1}\Big)^{2}.
\end{align*}
By Jensen's inequality,
\begin{align*}
 & \EE_{k,1}\bigg[\Big(\sum_{h=1}^{H}\big(r(s_{k,h},a_{k,h})-\rho_{k,h}\big)-\widehat{V}_{k,1}(s_{k,1})\Big)^{2}\bigg]\\
 & \le\EE_{k,1}\bigg[\Big(\Big|\sum_{h=2}^{H}\big(r(s_{k,h},a_{k,h})-\rho_{k,h}\big)-\EE_{k,1}\widehat{V}_{k,2}(s_{k,2})\Big|+\Big|\rho_{k,1}\Big|+10H\beta_{k-1}\kappa^{-1/2}U_{k,1}\Big)^{2}\bigg]\\
 & \le\EE_{k,1}\bigg[\Big(\Big|\sum_{h=2}^{H}\big(r(s_{k,h},a_{k,h})-\rho_{k,h}\big)-\widehat{V}_{k,2}(s_{k,2})\Big|+\Big|\rho_{k,1}\Big|+10H\beta_{k-1}\kappa^{-1/2}U_{k,1}\Big)^{2}\bigg].
\end{align*}
Again applying \eqref{eq:var_bound_value_diff},
\begin{align*}
 & \Big|\sum_{h=2}^{H}\big(r(s_{k,h},a_{k,h})-\rho_{k,h}\big)-\widehat{V}_{k,2}(s_{k,2})\Big|\\
 & \le\Big|\sum_{h=3}^{H}\big(r(s_{k,h},a_{k,h})-\rho_{k,h}\big)-\EE_{k,2}\widehat{V}_{k,3}(s_{k,3})\Big|+\Big|\EE_{k,2}\widehat{V}_{k,3}(s_{k,3})-\widehat{V}_{k,2}(s_{k,2})+r(s_{k,h},a_{k,h})-\rho_{k,h}\Big|\\
 & \le\Big|\sum_{h=3}^{H}\big(r(s_{k,h},a_{k,h})-\rho_{k,h}\big)-\EE_{k,2}\widehat{V}_{k,3}(s_{k,3})\Big|+\Big|\rho_{k,2}\Big|+10H\beta_{k-1}\kappa^{-1/2}U_{k,2},
\end{align*}
it follows that
\begin{align*}
 & \EE_{k,1}\bigg[\Big(\sum_{h=1}^{H}\big(r(s_{k,h},a_{k,h})-\rho_{k,h}\big)-\widehat{V}_{k,1}(s_{k,1})\Big)^{2}\bigg]\\
 & \le\EE_{k,1}\bigg[\Big(\Big|\sum_{h=3}^{H}\big(r(s_{k,h},a_{k,h})-\rho_{k,h}\big)-\EE_{k,2}\widehat{V}_{k,3}(s_{k,3})\Big|+\sum_{h=1}^{2}\Big|\rho_{k,h}\Big|+10H\beta_{k-1}\kappa^{-1/2}\sum_{h=1}^{2}U_{k,h}\Big)^{2}\bigg]\\
 & \le\EE_{k,1}\bigg[\EE_{k,2}\bigg[\Big(\Big|\sum_{h=3}^{H}\big(r(s_{k,h},a_{k,h})-\rho_{k,h}\big)-\widehat{V}_{k,3}(s_{k,3})\Big|+\sum_{h=1}^{2}\Big|\rho_{k,h}\Big|+10H\beta_{k-1}\kappa^{-1/2}\sum_{h=1}^{2}U_{k,h}\Big)^{2}\bigg]\bigg]\\
 & =\EE_{k,1}\bigg[\Big(\Big|\sum_{h=3}^{H}\big(r(s_{k,h},a_{k,h})-\rho_{k,h}\big)-\widehat{V}_{k,3}(s_{k,3})\Big|+\sum_{h=1}^{2}\Big|\rho_{k,h}\Big|+10H\beta_{k-1}\kappa^{-1/2}\sum_{h=1}^{2}U_{k,h}\Big)^{2}\bigg],
\end{align*}
where the last inequality uses Jensen's inequality and the last inequality
uses the tower property of the conditional expectation. Recursively,
\[
\EE_{k,1}\bigg[\Big(\sum_{h=1}^{H}\big(r(s_{k,h},a_{k,h})-\rho_{k,h}\big)-\widehat{V}_{k,1}(s_{k,1})\Big)^{2}\bigg]\le\EE_{k,1}\bigg[\Big(\sum_{h=1}^{H}\Big|\rho_{k,h}\Big|+10H\beta_{k-1}\kappa^{-1/2}\sum_{h=1}^{H}U_{k,h}\Big)^{2}\bigg]
\]
By Cauchy-Schwartz inequality, 
\[
\begin{split} & 2\EE_{k,1}\bigg[\Big(\sum_{h=1}^{H}\Big|\rho_{k,h}\Big|+10H\beta_{k-1}\kappa^{-1/2}\sum_{h=1}^{H}U_{k,h}\Big)^{2}\bigg].\\
 & \le4\EE_{k,1}\bigg[\big(\sum_{h=1}^{H}\Big|\rho_{k,h}\Big|\big)^{2}+100H^{2}\beta_{k-1}^{2}\kappa^{-1}\Big(\sum_{h=1}^{H}U_{k,h}\Big)^{2}\bigg]\\
 & \le4H\EE_{k,1}\bigg[\sum_{h=1}^{H}\rho_{k,h}^{2}+100H^{2}\beta_{k-1}^{2}\kappa^{-1}\sum_{h=1}^{H}U_{k,h}^{2}\bigg].
\end{split}
\]
Plugging in \eqref{eq:var_bound_first_decomp}, 
\[
\begin{split} & \EE_{k,1}\bigg[\bigg(\sum_{h=1}^{H}\widehat{V}_{k,h+1}(s_{k,h+1})-\EE\Big[\widehat{V}_{k,h+1}(s_{k,h+1})\Big|s_{k,h},a_{k,h}\Big]\bigg)^{2}\bigg]\\
 & \le8H\EE_{k,1}\bigg[\sum_{h=1}^{H}\rho_{k,h}^{2}+100H^{2}\beta_{k-1}^{2}\kappa^{-1}\sum_{h=1}^{H}U_{k,h}^{2}\bigg]\\
 & =8H\EE_{k,1}\bigg[\sum_{h=1}^{H}\EE_{k,h-1}[\rho_{k,h}^{2}]+100H^{2}\beta_{k-1}^{2}\kappa^{-1}\sum_{h=1}^{H}U_{k,h}^{2}\bigg],
\end{split}
\]
where the last equality uses the Tower property of the conditional
expectation. Because $s_{k,1}$ and $a_{k,1}$ are deterministic given
$\Fcal_{k-1,H}$, 
\[
\rho_{k,1}:=r(s_{k,1},a_{k,1})-\EE_{k-1,H}[r(s_{k,1},a_{k,1})]=0.
\]
Thus,
\begin{align*}
 & \EE_{k,1}\bigg[\bigg(\sum_{h=1}^{H}\widehat{V}_{k,h+1}(s_{k,h+1})-\EE\Big[\widehat{V}_{k,h+1}(s_{k,h+1})\Big|s_{k,h},a_{k,h}\Big]\bigg)^{2}\bigg]\\
 & \le8H\EE_{k,1}\bigg[\sum_{h=1}^{H-1}\EE_{k,h}[\rho_{k,h+1}^{2}]+100H^{2}\beta_{k-1}^{2}\kappa^{-1}\sum_{h=1}^{H}U_{k,h}^{2}\bigg],
\end{align*}
Let $\pi_{k,h}(s)=\argmax a\widehat{Q}_{k,h}(s,a)$ denote the action
selected by the deterministic policy. 
Then,
\begin{align*}
\EE_{k,h}[\rho_{k,h+1}^2] &= \EE_{k,h}[(r(s_{k,h+1},a_{k,h+1})-\EE_{k,h}[r(s_{k,h+1},a_{k,h+1})])^{2}]\\
 & =\EE_{k,h}\Big[\Big(\big(r\big(s_{k,h+1},\pi_{k,h+1}(s_{k,h+1})\big)-\EE_{k,h}[(r(s_{k,h+1},\pi_{k,h+1}(s_{k,h+1}))]\big)^{2}\Big].
\end{align*}
Let $v_{k,h+1}:=\sum_{s\in S_{k,h}}r(s,\pi_{k,h+1}(s))e_{s}$ denote
the stacked vector of possible expected rewards values on $S_{k,h}$.
Because the policy $\pi_{k,h+1}$ is determined at $k-1$th episode,
$v_{k,h+1}$ is $\mathcal{F}_{k,h}$-measurable. Taking expectation
on both sides, 
\begin{align*}
 & \EE_{k,h}\Big[\big(v_{k,h+1}^{\top}e_{s_{k,h+1}}-E_{k,h}[v_{k,h+1}^{\top}e_{s_{k,h+1}}]\big)^{2}\Big]\\
 & =v_{k,h+1}^{\top}\EE_{k,h}\bigg[\big(\eb_{s_{k,h+1}}-\EE_{k,h}\big[\eb_{s_{k,h+1}}\big]\big)\big(\eb_{s_{k,h+1}}-\EE_{k,h}\big[\eb_{s_{k,h+1}}\big]\big)^{\top}\bigg]v_{k,h+1}.
\end{align*}
By the identity for the covariance of Multinomial distribution (Proposition
\ref{prop:mean_var}), 
\[
\EE_{k,h}\bigg[\big(\eb_{s_{k,h+1}}-\EE_{k,h}\big[\eb_{s_{k,h+1}}\big]\big)\big(\eb_{s_{k,h+1}}-\EE_{k,h}\big[\eb_{s_{k,h+1}}\big]\big)^{\top}\bigg]=\nabla^{2}L_{k,h}(\vPhib_{k,h}\thetab_{h}^{\star}),
\]
Thus, 
\[
\EE_{k,h}[(r(s_{k,h+1},a_{k,h+1})-\EE_{k,h}[r(s_{k,h+1},a_{k,h+1})])^{2}]\le\max_{\xb:\|\xb\|_{\infty}\le1}\xb^{\top}\nabla^{2}L_{k,h}(\vPhib_{k,h}\thetab_{h}^{\star})\xb,
\]
and we define $\sigma_{k,h}^{2}:=\max_{\xb:\|\xb\|_{\infty}\le1}\xb^{\top}\nabla^{2}L_{k,h}(\vPhib_{k,h}\thetab_{h}^{\star})\xb$.
Thus, 
\[
\begin{split} & \EE_{k,1}\bigg[\bigg(\sum_{h=1}^{H-1}\widehat{V}_{k,h+1}(s_{k,h+1})-\EE\Big[\widehat{V}_{k,h+1}(s_{k,h+1})\Big|s_{k,h},a_{k,h}\Big]\bigg)^{2}\bigg]\\
 & \le8H\EE_{k,1}\bigg[\sum_{h=1}^{H}\sigma_{k,h}^{2}+100H^{2}\beta_{k-1}^{2}\kappa^{-1}\sum_{h=1}^{H}U_{k,h}^{2}\bigg].
\end{split}
\]
Summing up over $k\in[K]$, 
\begin{equation}
\begin{aligned} & \sum_{k=1}^{K}\EE_{k,1}\bigg[\sum_{h=1}^{H}\Big(\widehat{V}_{k,h+1}(s_{k,h+1})-\EE\Big[\widehat{V}_{k,h+1}(s_{k,h+1})\Big|s_{k,h},a_{k,h}\Big]\Big)^{2}\bigg]\\
 & \le12H\sum_{k=1}^{K}\EE_{k,1}\bigg[\Ubrace{\sum_{h=1}^{H}\sigma_{k,h}^{2}}{\phi_{k}}+100H^{2}\beta_{k-1}^{2}\kappa^{-1}\Ubrace{\sum_{h=1}^{H}U_{k,h}^{2}}{\psi_{k}}\bigg].
\end{aligned}
\label{eq:phi_psi_decomp}
\end{equation}
By the elliptical potential lemma (Lemma \ref{lem:vPhib_potential}),
\begin{equation}
\begin{split}\sum_{k=1}^{K}\psi_{k}:=\sum_{h=1}^{K}\sum_{h=1}^{H}U_{k,h}^{2}:= & \sum_{k=1}^{K}\sum_{h=1}^{H}\max_{s^{\prime}\in\Scal_{k,h}}\|\vPhib_{k,h}^{\top}\eb_{s^{\prime}}\|_{\Ab_{k-1,h}}^{2}\\
\le & \beta_{K-1}^{2}\kappa^{-1}\sum_{k=1}^{K}\sum_{h=1}^{H}\max_{s^{\prime}\in\Scal_{k,h}}\|\vPhib_{k,h}^{\top}\eb_{s^{\prime}}\|_{\Ab_{k-1,h}}^{2}\\
\le & H\beta_{K-1}^{2}\kappa^{-1}d\log\frac{K+1}{d}.
\end{split}
\label{eq:psi_n_bound}
\end{equation}
Note that $0\le\psi_{k}\le H\kappa^{-1}\beta_{K-1}^{2}$. By Chernoff
bound (Lemma \ref{lem:chernoff}), with probability at least $1-\delta$,
\[
\begin{split}\sum_{k=1}^{K}\EE_{k,1}[\psi_{k}]\le & \frac{e}{e-1}\sum_{k=1}^{K}\psi_{k}+\frac{eH\kappa^{-1}\beta_{K-1}^{2}}{e-1}\log\frac{1}{\delta}\\
\le & \frac{eH\kappa^{-1}\beta_{K-1}^{2}}{e-1}\big(d\log\frac{K+1}{d}+\log\frac{1}{\delta}\big)
\end{split}
\]
where the last inequality uses~\eqref{eq:psi_n_bound}. Plugging
in~\eqref{eq:phi_psi_decomp}, 
\begin{equation}
\begin{aligned} & \sum_{k=1}^{K}\EE_{k,1}\bigg[\sum_{h=1}^{H}\Big(\widehat{V}_{k,h+1}(s_{k,h+1})-\EE\Big[\widehat{V}_{k,h+1}(s_{k,h+1})\Big|s_{k,h},a_{k,h}\Big]\Big)^{2}\bigg]\\
 & \le12H\sum_{k=1}^{K}\EE_{k,1}[\phi_{k}]+\frac{1200eH^{4}\kappa^{-2}\beta_{K-1}^{4}}{e-1}\big(d\log\frac{K+1}{d}+\log\frac{1}{\delta}\big),
\end{aligned}
\label{eq:var_bound_sum}
\end{equation}
where we use the fact that $\beta_{n}$ is nondecreasing in $k$.
Because $\phi_{k}\in[0,H]$, by Chernoff bound (Lemma \ref{lem:chernoff}),
with probability at least $1-\delta$, 
\[
\sum_{k=1}^{K}\EE_{k,1}[\phi_{k}]\le\frac{e}{e-1}\big(\sum_{k=1}^{K}\phi_{k}+H\log\frac{1}{\delta}\big)
\]
Plugging in \eqref{eq:var_bound_sum}, 
\[
\begin{split} & \sum_{k=1}^{K}\EE_{k,1}\bigg[\sum_{h=1}^{H}\Big(\widehat{V}_{k,h+1}(s_{k,h+1})-\EE\Big[\widehat{V}_{k,h+1}(s_{k,h+1})\Big|s_{k,h},a_{k,h}\Big]\Big)^{2}\bigg]\\
 & \le\frac{e}{e-1}\bigg(12H\sum_{k=1}^{K}\phi_{k}+H\log\frac{1}{\delta}+{1200H^{4}\kappa^{-2}\beta_{K-1}^{4}}\big(d\log\frac{K+1}{d}+\log\frac{1}{\delta}\big)\bigg)\\
 & \le\frac{12He}{e-1}\bigg(\sum_{k=1}^{K}\phi_{k}+{100H^{3}\kappa^{-2}\beta_{K-1}^{4}}\big(d\log\frac{K+1}{d}+2\log\frac{1}{\delta}\big)\bigg).
\end{split}
\]
Plugging in \eqref{eq:var_bound_chernoff} completes the proof. 
\end{proof}

\subsection{Proof of the Variance Adaptive Lower Bound (Theorem \ref{thm:lower_bound})}
\label{subsec:lower_bound_proof}
\LowerBound*

We construct a hard MNL-MDP instance as follows. Let $\Acal:=\{\pm\sqrt{\Delta}\}$
denote the action space where for some $\Delta\in(0,\frac{\log2}{4(d-1)})$.
The initial state is $s_{1}=1$ and the good state is $s=H+2$ where
the reward is $r(2H+1,a)=1,$ for any action $\ab\in\Acal=\{\pm\sqrt{\Delta}\}^{d-1}$.
Otherwise the reward $r(s,\ab)=0$ for any $s\neq2H+1$ and $\ab\in\Acal$.
The reachable set is defined as $\Scal_{h}(2H+1,\ab)=\{2H+1\}$ and
for any $h\in[H]$ and $\ab\in\Acal$, i.e., $2H+1$ is the absorbing
state. For each $h\in[H]$, $\Scal_{h}(s,\ab)=\{2H+1,2h+1,2h+2\}$
for $s\in\{2h-1,2h\}$ and any $\ab\in\Acal$.

Let us fix $\thetab_{1}^{\star},\ldots,\thetab_{H}^{\star}$ and without
loss of generality we assume that $\theta_{h,d}^{\star}\neq0$ for
$\thetab_{h}^{\star}:=(\theta_{h,1}^{\star},\ldots,\theta_{h,d}^{\star})$.
Define, 
\[
\begin{split}\Dcal_{h}= & \Big\{\thetab_{h}^{\star}+(u_{h,1},\ldots,u_{h,d-1},0)^{\top}:u_{h,1},\ldots,u_{h,d-1}\in\{\pm\sqrt{\Delta}\}\Big\}\\
= & \Big\{\thetab_{h}^{\star}+\ub_{h}:\ub_{h}\in\{\pm\sqrt{\Delta}\}^{d-1}\times\{0\}\Big\}.
\end{split}
\]

For $\epsilon\in(0,1/2)$ to be determined later, define 
\begin{equation}
\tilde{\Delta}:=\frac{1}{(d-1)}\Big(\frac{1}{1+\frac{1-\epsilon}{\epsilon}\exp\Big(-4(d-1)\Delta\Big)}-\epsilon\Big),\quad\phi:=\frac{1}{2}\sqrt{\frac{1-\epsilon}{\epsilon}\cdot\frac{1-\epsilon-(d-1)\tilde{\Delta}}{\epsilon+(d-1)\tilde{\Delta}}},\label{eq:Delta_tilde}
\end{equation}
and a function $p(x):=1/(1+2\phi e^{-2x})$. These definitions are
constructed to satisfy $p\big((d-1)\Delta\big)=\epsilon+(d-1)\tilde{\Delta}$
and $p\big(-(d-1)\Delta\big)=\epsilon$. Because $\Delta\in(0,\frac{\log2}{4(d-1)})$,
with simple algebra, for any $\epsilon\in(0,1/2)$, 
\[
\Delta\le\frac{1}{4(d-1)}\log\frac{2-2\epsilon}{1-2\epsilon}\Longleftrightarrow(d-1)\tilde{\Delta}\le\epsilon
\]
For $\ab\in\Acal=\{\pm\sqrt{\Delta}\}^{d-1}$ we define a feature
map 
\[
\vPhib\big(s,\ab\big):=\begin{pmatrix}1\\
-1\\
-1
\end{pmatrix}\bigg(\ab^{\top},-\frac{\ab^{\top}(\theta_{h,1}^{\star},\ldots,\theta_{h,d-1}^{\star})}{\theta_{h,d}^{\star}}-\frac{\log\phi}{2\theta_{h,d}^{\star}}\bigg)\in\RR^{3\times d},\quad s\in\{2h-1,2h\}
\]
for $h\in[H]$. The algorithm only knows the image of $\vPhib$ over
$\Scal\times\Acal$, i.e., $\vPhib(s,\ab)$ for $(s,\ab)\in\Scal\times\Acal$;
however it does not know the model and the parameter $\thetab_{h}^{\star}$.
It is easy to show that 
\[
\vPhib(s,\ab)\thetab_{h}^{\star}=-\frac{1}{2}\log\phi\begin{pmatrix}1\\
-1\\
-1
\end{pmatrix}
\]
and for $\ub:=(u_{1},\ldots,u_{d-1})\in\{\pm\sqrt{\Delta}\}^{d-1}$,
we have 
\[
\begin{split}\vPhib\big(s,\ab\big)(\thetab_{h}^{\star}+\ub_{h})= & \big(\ab^{\top}\ub_{h}-\frac{1}{2}\log\phi\big)\begin{pmatrix}1\\
-1\\
-1
\end{pmatrix}.\\
= & \Big(\Delta\sum_{i=1}^{d-1}\Sgn(u_{h,i})\Sgn(a_{i})-\frac{1}{2}\log\phi\Big)\begin{pmatrix}1\\
-1\\
-1
\end{pmatrix}.
\end{split}
\]
Then, the optimal action at step $h\in[H]$ given $\thetab_{h}^{\star}+\ub_{h}\in\Dcal_{h}$
is $a_{\star}(\thetab_{h}^{\star}+\ub_{h})=\ub_{h}$, which maximizes
the probability of reaching state $2H+1$. Under this setting, we
prove a lower bound of the regret. 
\begin{lemma} \label{lem:reduction}
(Reduction to the multi-armed bandits instances) Assume that $H\ge4$
and $\Delta\in(0,\frac{\log2}{4(d-1)})$. For any $\epsilon\in(0,1/H)$,
nonzero parameters $\Thetab^{\star}:=\{\thetab_{h}^{\star}:h\in[H]\}$,
perturbations $\Ub:=\big\{\ub_{h}\in\{\pm\sqrt{\Delta}\}^{d-1}\times\{0\}:h\in[H]\big\}$,
and policy $\Pi_{k}$, the problem constructed above has regret: 
\[
\Reg(\Thetab^{\star}+\Ub,\Pi_{K})\ge\frac{H\epsilon}{8(d-1)}\sum_{k=1}^{K}\sum_{h=1}^{\lfloor H/2\rfloor}\sum_{i=1}^{d-1}\PP_{\Thetab^{\star}+\Ub,\Pi_{K}}(\eb_{i}^{\top}a_{h,k}\neq u_{h,i}),
\]
where $a_{h,k}$ is the action chosen at step $h\in[H]$ and episode
$k\in[K]$ by the policy $\Pi_{K}$ given $\Thetab^{\star}+\Ub$.
\end{lemma}
\begin{proof}
\medspace{} Let us fix $\Thetab^{\star}$ throughout the proof. For
each $k\in[K]$, we write the value at episode $k\in[K]$ as 
\[
V_{k}:=\EE_{\Thetab^{\star}+\Ub,\Pi_{K}}\Big[\sum_{h=1}^{H}r(s_{h,k},a_{h,k})\Big|s_{1,k}=1\Big]
\]
By construction of the problem, 
\[
V_{k}=\EE_{\Thetab^{\star}+\Ub,\Pi_{K}}\Big[\sum_{h=1}^{H-1}(H-h)\PP_{\Thetab^{\star}+\Ub,\Pi_{K}}(s_{h+1,k}=H+2,s_{h,k}\neq H+2|s_{1,k}=1)\Big].
\]
By the law of total probability and the Markov property, 
\[
\begin{split} & \PP_{\Thetab^{\star}+\Ub,\Pi_{K}}(s_{h+1,k}=H+2|s_{h,k}\neq H+2,s_{1,k}=1)\\
 & \sum_{a\in\Acal}\PP_{\Thetab^{\star}+\Ub,\Pi_{K}}(s_{h+1,k}=H+2,a_{h,k}=a|s_{h,k}\neq H+2,s_{1,k}=1)\\
 & =\sum_{a\in\Acal}p_a\PP_{\Thetab^{\star}+\Ub,\Pi_{K}}(a_{h,k}=a|s_{h,k}\neq H+2,s_{1,k}=1),
\end{split}
\]
where \(p_a:=\PP_{\Thetab^{\star}+\Ub,\Pi_{K}}(s_{h+1}=H+2|s_{h,k}\neq H+2,a_{h,k}=a) \)
By construction the transition probability, 
\[
\begin{split} 
p_a = \Ubrace{\frac{1}{1+2\phi\exp\big(-2\Delta\sum_{i=1}^{d-1}\Sgn(u_{h,i})\Sgn(a_{i})\big)}}{p\big(\Delta\sum_{i=1}^{d-1}\Sgn(u_{i})\Sgn(a_{i})\big)}.
\end{split}
\]
Thus, for $\epsilon\in(0,1/2)$, 
\[
\begin{split} & \PP_{\Thetab^{\star}+\Ub,\Pi_{K}}(s_{h+1,k}=H+2|s_{h,k}\neq H+2,s_{1,k}=1)\\
 & =\sum_{a\in\Acal}p\Big(\Delta\sum_{i=1}^{d-1}\Sgn(u_{h,i})\Sgn(a_{i})\Big)\PP_{\Thetab^{\star}+\Ub,\Pi_{K}}(a_{h,k}=a|s_{h,k}\neq H+2,s_{1,k}=1)\\
 & =\epsilon+\Ubrace{\sum_{a\in\Acal}\bigg(p\Big(\Delta\sum_{i=1}^{d-1}\Sgn(u_{h,i})\Sgn(a_{i})\Big)-\epsilon\bigg)\PP_{\Thetab^{\star}+\Ub,\Pi_{K}}(a_{h,k}=a|s_{h,k}\neq H+2,s_{1,k}=1)}{\psi_{h,k}}
\end{split}
\]
Consequently, 
\[
\PP_{\Thetab^{\star}+\Ub,\Pi_{K}}(s_{h+1,k}\neq H+2|s_{h,k}\neq H+2,s_{1,k}=1)=1-\epsilon-\psi_{h,k}.
\]
This implies, 
\[
\PP_{\Thetab^{\star}+\Ub,\Pi_{K}}(s_{h+1,k}=H+2,s_{h,k}\neq H+2|s_{1,k}=1)=(\epsilon+\psi_{h,k})\prod_{h^{\prime}=1}^{h-1}(1-\epsilon-\psi_{h^{\prime},k}).
\]
Thus, the value function, 
\[
V_{k}=\EE\Big[\sum_{h=1}^{H-1}(H-h)(\epsilon+\psi_{h,k})\prod_{h^{\prime}=1}^{h-1}(1-\epsilon-\psi_{h^{\prime},k})\Big].
\]
Because the optimal policy is deterministic $a_{\star}=\ub_{h}$,
we have 
\[
\begin{split}V_{k}^{\star}= & \EE_{\Thetab^{\star}+\Ub,\pi^{\star}}\Big[\sum_{h=1}^{H}r(s_{h,k},\ub_{h})\Big|s_{1,k}=1\Big]\\
= & \sum_{h=1}^{H-1}(H-h)\Ubrace{p\big(\Delta(d-1)\big)}{:=p^{\star}}\bigg(1-p\big(\Delta(d-1)\big)\bigg)^{h-1}\\
= & \sum_{h=1}^{H-1}(H-h)p^{\star}(1-p^{\star})^{h-1}
\end{split}
\]
For $j\in[H-1]$, define, 
\[
\begin{split}V_{j}^{\star} & :=\sum_{h=j}^{H-1}(H-h)p^{\star}(1-p^{\star})^{h-j},\\
\widehat{V}_{j} & :=\sum_{h=j}^{H-1}(H-h)(\epsilon+\psi_{h,k})\prod_{h^{\prime}=j}^{h-1}(1-\epsilon-\psi_{h^{\prime},k}).
\end{split}
\]
Then, with some simple algebra, 
\[
V_{j}^{\star}-\widehat{V}_{j}=(H-j-V_{j+1}^{\star})\Big(p^{\star}-\psi_{j,k}\Big)+(V_{j+1}^{\star}-\widehat{V}_{j+1})(1-\epsilon-\psi_{j,k}),
\]
which result is concurrent to equation (C.24) in \citet{park2025infinite}.
Recursively the regret at episode $k$ is 
\begin{equation}
V_{1}^{\star}-\widehat{V}_{1}=\sum_{h=1}^{H-1}(H-h-V_{h+1}^{\star})(p^{\star}-\psi_{h,k})\prod_{j=1}^{h-1}(1-\epsilon-\psi_{j,k}).\label{eq:lower_bound_regret}
\end{equation}
Meanwhile, $V_{j}^{\star}$ can be written as, 
\[
V_{j}^{\star}=\frac{(1-p^{\star})^{H-j}-1}{p^{\star}}+H-j+1-(1-p^{\star})^{H-j},
\]
which implies, for $h\le\lfloor H/2\rfloor$ 
\[
\begin{split}H-h-V_{h+1}^{\star}= & (1-p^{\star})^{H-h-1}+\frac{1-(1-p^{\star})^{H-h-1}}{p^{\star}}\\
= & \frac{1-(1-p^{\star})^{H-h}}{p^{\star}}\\
\ge & \frac{1-(1-p^{\star})^{H/2}}{p^{\star}}.
\end{split}
\]
Because $(d-1)\tilde{\Delta}\le\epsilon\le1/H$, 
\[
\begin{split}p^{\star}:=p(\Delta(d-1))= & \epsilon+(d-1)\tilde{\Delta}\le\frac{2}{H}.\end{split}
\]
Because $x\mapsto\frac{1-(1-x)^{H/2}}{x}$ is nonincreasing in $(0,1)$,
\[
\begin{split}H-h-V_{h+1}^{\star}\ge & \frac{1-(1-p^{\star})^{H/2}}{p^{\star}}\\
\ge & \frac{1-(1-\frac{2}{H})^{H/2}}{\frac{2}{H}}\\
\ge & \frac{1-e^{-1}}{2}H\\
\ge & \frac{H}{4},
\end{split}
\]
where the second inequality holds by $(1-x^{-1})^{x}\le e^{-1}$ for
$x\ge1$. Plugging in \eqref{eq:lower_bound_regret}, 
\[
V_{1}^{\star}-\widehat{V}_{1}\ge\frac{H}{4}\sum_{h=1}^{\lfloor H/2\rfloor}\Big(p^{\star}-\psi_{h,k}\Big)\prod_{j=1}^{h-1}(1-\epsilon-\psi_{j,k}).
\]
Note that for any $j\in[H]$, we have 
\[
\begin{split}\epsilon+\psi_{j,k}= & \PP_{\Thetab^{\star}+\Ub,\Pi_{K}}(s_{h+1,k}=H+2|s_{h,k}\neq H+2,s_{1,k}=1)\\
\le & p\big((d-1)\Delta\big)\le1/H
\end{split}
\]
and thus, 
\[
\prod_{j=1}^{h-1}(1-\epsilon-\psi_{j,k})\ge(1-\frac{1}{H})^{H/2}\ge\frac{1}{\sqrt{e}}\ge\frac{1}{2}.
\]
It follows that 
\[
V_{1}^{\star}-\widehat{V}_{1}\ge\frac{H}{8}\sum_{h=1}^{\lfloor H/2\rfloor}\Big(p^{\star}-\psi_{h,k}\Big).
\]
By definition of $p$, recall that $p\big(-(d-1)\Delta\big)=\epsilon$
and thus, 
\[
\begin{split}p\Big(\Delta\sum_{i=1}^{d-1}\Sgn(u_{h,i})\Sgn(a_{i})\Big)-\epsilon= & p\Big(\Delta\sum_{i=1}^{d-1}\Sgn(u_{h,i})\Sgn(a_{i})\Big)-p\big(-\Delta(d-1)\big)\\
\le & \big(\epsilon+(d-1)\tilde{\Delta}\big)\Delta\Big(\sum_{i=1}^{d-1}\Sgn(u_{h,i})\Sgn(a_{i})+(d-1)\Big),
\end{split}
\]
where the last inequality holds because $\sup_{x\in[-(d-1)\Delta,(d-1)\Delta]}|p^{\prime}(x)|\le\epsilon+(d-1)\tilde{\Delta}$
and the mean value theorem. By definition of $\tilde{\Delta}$, we
have 
\[
\begin{split}\frac{\Delta}{\tilde{\Delta}}= & \frac{1}{4\tilde{\Delta}(d-1)}\log\frac{(\epsilon+(d-1)\tilde{\Delta})(1-\epsilon)}{(1-\epsilon-(d-1)\tilde{\Delta})\epsilon}\\
= & \frac{1}{4\tilde{\Delta}(d-1)}\log\Big(1+\frac{(d-1)\tilde{\Delta}}{(1-\epsilon-(d-1)\tilde{\Delta})\epsilon}\Big)\\
\le & \frac{1}{4\tilde{\Delta}(d-1)}\cdot\frac{(d-1)\tilde{\Delta}}{(1-\epsilon-(d-1)\tilde{\Delta})\epsilon}\\
= & \frac{1}{4\big(1-\epsilon-(d-1)\tilde{\Delta}\big)\epsilon}.
\end{split}
\]
Because $\Delta\in\big(0,\frac{\log2}{4(d-1)}\big)$, and $(d-1)\tilde{\Delta}\le\epsilon\le1/H$,
we obtain, 
\[
\begin{split}\frac{\big(\epsilon+(d-1)\tilde{\Delta}\big)\Delta}{\tilde{\Delta}}= & \frac{\epsilon+(d-1)\tilde{\Delta}}{4\epsilon\big(1-\epsilon-(d-1)\tilde{\Delta}\big)}\\
\le & \frac{2\epsilon}{4\epsilon\big(1-2\epsilon\big)}\\
\le & \frac{1}{2(1-2/H)}\\
\le & 1,
\end{split}
\]
where the last inequality holds by $H\ge4$. Thus, 
\[
\begin{split}p\Big(\Delta\sum_{i=1}^{d-1}\Sgn(u_{h,i})\Sgn(a_{i})\Big)-\epsilon\le & \frac{\big(\epsilon+(d-1)\tilde{\Delta}\big)\Delta}{\tilde{\Delta}}\tilde{\Delta}\Big(\sum_{i=1}^{d-1}\Sgn(u_{h,i})\Sgn(a_{i})+(d-1)\Big)\\
\le & \tilde{\Delta}\Big(\sum_{i=1}^{d-1}\Sgn(u_{h,i})\Sgn(a_{i})+(d-1)\Big).
\end{split}
\]
Rearranging the terms, 
\[
\tilde{\Delta}(d-1)-p\Big(\Delta\sum_{i=1}^{d-1}\Sgn(u_{h,i})\Sgn(a_{i})\Big)+\epsilon\ge-\tilde{\Delta}\sum_{i=1}^{d-1}\Sgn(u_{h,i})\Sgn(a_{i}).
\]
It follows that $p^{\star}-\psi_{h,k}=p\big(\Delta(d-1)\big)-\psi_{h,k}=\epsilon+\tilde{\Delta}(d-1)-\psi_{h,k}$
and define \(\pi_a:=\PP_{\Thetab^{\star}+\Ub,\Pi_{K}}(a_{h,k}=a|s_{h,k}\neq H+2,s_{1,k}=s_{1})\) 
\[
\begin{split} & \epsilon+\tilde{\Delta}(d-1)-\psi_{h,k}\\
 & =\sum_{a\in\Acal}\bigg(\epsilon+\tilde{\Delta}(d-1)-p\Big(\Delta\sum_{i=1}^{d-1}\Sgn(u_{h,i})\Sgn(a_{i})\Big)+\epsilon\bigg)\pi_a\\
 & \ge\sum_{a\in\Acal}\bigg(\epsilon-\tilde{\Delta}\sum_{i=1}^{d-1}\Sgn(u_{h,i})\Sgn(a_{i})\bigg)\pi_a\\
 & \ge\sum_{a\in\Acal}\frac{\epsilon}{d-1}\bigg(d-1-\sum_{i=1}^{d-1}\Sgn(u_{h,i})\Sgn(a_{i})\bigg) \pi_a,
\end{split}
\]
where the second inequality holds by $(d-1)\tilde{\Delta}\le\epsilon$.
Then, 
\[
\begin{split} & V_{1}^{\star}-\widehat{V}_{1}\\
 & \ge\frac{H}{8}\sum_{h=1}^{\lfloor H/2\rfloor}\Big(p^{\star}-\psi_{h,k}\Big)\\
 & \ge\frac{H\epsilon}{8(d-1)}\sum_{h=1}^{\lfloor H/2\rfloor}\sum_{a\in\Acal}\bigg(d-1-\sum_{i=1}^{d-1}\Sgn(u_{h,i})\Sgn(a_{i})\bigg)\pi_a\\
 & =\frac{H\epsilon}{8(d-1)}\sum_{h=1}^{\lfloor H/2\rfloor}\sum_{a\in\Acal}\sum_{i=1}^{d-1}\bigg(1-\Sgn(u_{h,i})\Sgn(a_{i})\bigg)\pi_a
\end{split}
\]
Because $\Reg(\Thetab^{\star}+\Ub,\Pi_{K})=\EE\Big[\sum_{k=1}^{K}(V_{k}^{\star}-\widehat{V}_{k})\Big]$,
taking the expectation on both sides and summing up over $k\in[K]$
completes the proof. 
\end{proof}
\begin{theorem}[Variance-adaptive lower bound for MNL-MDP] \label{thm:lower_bound_general}
Assume that $d\ge2$, $H\ge4$, and $K\ge\max\{144H(d-1)^{2},\kappa^{-2}H^{-1}\}$.
Given any parameter $\Thetab^{\star}=(\thetab_{1}^{\star},\ldots,\thetab_{H}^{\star})\in\RR^{d\times H}$,
there exists a MNL-MDP instance $\Mcal_{\Thetab^{\star}}:=\Mcal(\Scal,\Acal,H,\Thetab^{\star},\vPhib,r)$
such that with probability at least $\delta\in(0,1/e)$, the regret
of any algorithm $\Pi_{K}$, 
\[
\Reg_{\Pi_{K}}\!(\Mcal_{\Thetab^{\star}})\!\ge\!\frac{2Hd\sqrt{\sum_{h=1}^{H}\!\sum_{k=1}^{K}\!\sigma_{k,h}^{2}}\sqrt{1-\frac{\delta}{1-\delta}\sqrt{\log\frac{1}{\delta}}}}{17}.
\]
\end{theorem}

\begin{remark} The term in the square root, 
\[
1-\frac{\delta}{1-\delta}\sqrt{\log\!\frac{1}{\delta}}
\]
is nonnegative and non-increasing over $\delta\in(0,1/e)$. This result
cannot be derived by the concentration inequalities and but by our
novel anti-concentration inequality for nonnegative random variables
(Lemma \ref{lem:anti}). Thus, Theorem~\ref{thm:lower_bound_general}
implies Theorem~\ref{thm:lower_bound} for $\delta=1/e$. \end{remark}
\begin{proof}
\medspace{} By the reduction lemma (Lemma \ref{lem:reduction}),
\[
\begin{split}\Reg(\Thetab^{\star}+\Ub,\Pi_{K})\ge & \frac{H\epsilon}{8(d-1)}\sum_{k=1}^{K}\sum_{h=1}^{\lfloor H/2\rfloor}\sum_{i=1}^{d-1}\PP_{\Thetab^{\star}+\Ub,\Pi_{K}}(\eb_{i}^{\top}a_{h,k}\neq u_{h,i})\\
= & \frac{H\epsilon}{8(d-1)}\sum_{h=1}^{\lfloor H/2\rfloor}\sum_{i=1}^{d-1}\EE_{\Thetab^{\star}+\Ub,\Pi_{K}}\Big[\sum_{k=1}^{K}\II(\eb_{i}^{\top}a_{h,k}\neq u_{h,i})\Big].
\end{split}
\]
Let $\tilde{\Ub}:=\sqrt{\tilde{\Delta}/\Delta}\Ub$ and $\tilde{a}_{i}=\sqrt{\tilde{\Delta}/\Delta}a_{i}$
for sufficiently small $\tilde{\Delta}>0$. Define $M_{h,i}(x):=\sum_{k=1}^{K}\II(\eb_{i}^{\top}\tilde{a}_{h,k}\neq\sqrt{\tilde{\Delta}}x)$
for $x\in\{-1,1\}$. Note that $\sum_{k=1}^{K}\II(\eb_{i}^{\top}a_{h,k}\neq u_{h,i})=M_{h,i}\big(\text{sign}(\tilde{u}_{h,i})\big)$
is scale-invariant and the distribution of $\sum_{k=1}^{K}\II(\eb_{i}^{\top}a_{h,k}\neq u_{h,i})$
given $\Thetab^{\star}+\Ub$ is equal to the distribution of $M_{h,i}\big(\text{sign}(\tilde{u}_{h,i})\big)$
given $\Thetab^{\star}+\tilde{\Ub}$. Thus, 
\[
\Reg(\Thetab^{\star}+\Ub,\Pi_{K})\ge\frac{H\epsilon}{8(d-1)}\sum_{h=1}^{\lfloor H/2\rfloor}\sum_{i=1}^{d-1}\EE_{\Thetab^{\star}+\tilde{\Ub},\Pi_{K}}\Big[M_{h,i}\big(\text{sign}(\tilde{u}_{h,i})\big)\Big]
\]
For each $\tilde{\Ub}\in\{\pm\sqrt{\tilde{\Delta}}\}^{(d-1)\times H}\times\{0\}^{H}$,
we define $\bar{\Ub}(\tilde{h},\tilde{i})$ by 
\[
\bar{\Ub}_{h,i}(\tilde{h},\tilde{i})=\begin{cases}
-\tilde{u}_{h,i} & h=\tilde{h},i=\tilde{i}\\
\tilde{u}_{h,i} & h\neq\tilde{h}\text{ or }i\neq\tilde{i}
\end{cases},
\]
where only the sign of $(\tilde{h},\tilde{i})$-th coordinate differs.
Let $\PP_{\tilde{\Ub}}$ and $\PP_{\bar{\Ub}(\tilde{h},\tilde{i})}$
denote the environment/learner interaction measure for $\{s_{k,h},a_{k,h}:k\in[K],h\in[H]\}$
given $\Thetab^{\star}+\tilde{\Ub}$ and $\Thetab^{\star}+\bar{\Ub}(\tilde{h},\tilde{i})$,
respectively. Because $M_{h,i}(1)\le K$, by Pinsker's inequality,
\begin{equation}
\begin{aligned}\EE_{\tilde{\Ub}}[M_{\tilde{h},\tilde{i}}(1)] & \ge\EE_{\bar{\Ub}(\tilde{h},\tilde{i})}[M_{\tilde{h},\tilde{i}}(1)]-K\|\PP_{\tilde{\Ub}}-\PP_{\bar{\Ub}(\tilde{h},\tilde{i})}\|_{1}\\
 & \ge\EE_{\bar{\Ub}(\tilde{h},\tilde{i})}[M_{\tilde{h},\tilde{i}}(1)]-K\sqrt{\frac{1}{2}KL(\PP_{\tilde{\Ub}},\PP_{\bar{\Ub}(\tilde{h},\tilde{i})})}
\end{aligned}
\label{eq:low_expect}
\end{equation}
Note that 
\[
\|\vPhib(s_{k,h},a_{k,h})\big(\Thetab^{\star}+\tilde{\Ub}-\Thetab^{\star}+\bar{\Ub}(\tilde{h},\tilde{i})\big)\eb_{h}\|_{\infty}\le2\II(\tilde{h}=h)\tilde{\Delta}
\]
By the divergence bound (Lemma \ref{lem:KL}), 
\[
KL(\PP_{\tilde{\Ub}},\PP_{\bar{\Ub}(\tilde{h},\tilde{i})})\le2\tilde{\Delta}^{2}\EE_{\tilde{\Ub}}[\sum_{k=1}^{K}\max_{\xb:\|\xb\|_{\infty}\le1}\xb^{\top}\nabla^{2}L_{k,\tilde{h}}\big(\vPhib(s_{k,\tilde{h}},a_{k,\tilde{h}})(\Thetab^{\star}+\tilde{\Ub})\eb_{\tilde{h}}\big)\xb]+\frac{8N\tilde{\Delta}^{3}}{\sqrt{6\kappa}}.
\]
Because the log-sum function is $\sqrt{6}$-self-concordant, by the
ratio bound (Lemma \ref{lem:log_sum_ratio}), 
\[
\begin{split} & \max_{\xb:\|\xb\|_{\infty}\le1}\xb^{\top}\nabla^{2}L_{k,\tilde{h}}\big(\vPhib(s_{k,\tilde{h}},a_{k,\tilde{h}})(\Thetab^{\star}+\tilde{\Ub})\eb_{\tilde{h}}\big)\xb\\
 & \le\exp(\sqrt{6}\|\vPhib(s_{k,\tilde{h}},a_{k,\tilde{h}})\tilde{\Ub}\eb_{\tilde{h}}\|_{2})\Ubrace{\max_{\xb:\|\xb\|_{\infty}\le1}\xb^{\top}\nabla^{2}L_{k,\tilde{h}}\big(\vPhib(s_{k,\tilde{h}},a_{k,\tilde{h}})\Thetab^{\star}\eb_{\tilde{h}}\big)\xb}{\sigma_{k,\tilde{h}}^{2}}.
\end{split}
\]
By definition of $\vPhib$ and $\tilde{u}_{H}$, 
\[
\begin{split}\exp(\sqrt{6}\|\vPhib(s_{k,\tilde{h}},a_{k,\tilde{h}})\tilde{\Ub}\eb_{\tilde{h}}\|_{2})= & e^{3\sqrt{2a_{K,H}^{\top}\tilde{\ub}_{H}}}\\
\le & e^{3\sqrt{2(d-1)}\tilde{\Delta}}\\
\le & 2,
\end{split}
\]
where the last inequality uses a bound for the ratio of variance (Lemma
\ref{lem:log_sum_ratio}) and the last inequality uses sufficiently
small $\tilde{\Delta}\in(0,\frac{\log2}{3\sqrt{2(d-1)}})$. Thus,
\[
KL(\PP_{\tilde{\Ub}},\PP_{\bar{\Ub}(\tilde{h},\tilde{i})})\le4\tilde{\Delta}^{2}\EE_{\tilde{\Ub}}[\sum_{k=1}^{K}\sigma_{k,\tilde{h}}^{2}]+\frac{8N\tilde{\Delta}^{3}}{\sqrt{6\kappa}}.
\]
Plugging in \eqref{eq:low_expect}, 
\[
\EE_{\Ub}[M_{h,i}(1)]\ge\EE_{\bar{\Ub}(h,i)}[M_{h,i}(1)]-K\sqrt{2\tilde{\Delta}^{2}\sum_{k=1}^{K}\EE_{\Thetab^{\star}+\tilde{\Ub},\Pi_{K}}\big[\sigma_{k,h}^{2}]+\frac{4N\tilde{\Delta}^{3}}{\sqrt{6\kappa}}}.
\]
Because $M_{h,i}(1)+M_{h,i}(-1)=K$, 
\[
\begin{split} & \EE_{\Ub}[M_{h,i}(1)]+\EE_{\bar{\Ub}(h,i)}[M_{h,i}(-1)]\\
 & \ge\EE_{\bar{\Ub}(h,i)}[M_{h,i}(1)+M_{h,i}(-1)]-K\sqrt{2\tilde{\Delta}^{2}\sum_{k=1}^{K}\EE_{\Thetab^{\star}+\tilde{\Ub},\Pi_{K}}\big[\sigma_{k,h}^{2}]+\frac{4N\tilde{\Delta}^{3}}{\sqrt{6\kappa}}}\\
 & \ge K-K\sqrt{2\tilde{\Delta}^{2}\sum_{k=1}^{K}\EE_{\Thetab^{\star}+\tilde{\Ub},\Pi_{K}}\big[\sigma_{k,h}^{2}]+\frac{4N\tilde{\Delta}^{3}}{\sqrt{6\kappa}}}
\end{split}
\]
Thus, the sum of regret over $\Ucal:=\{\pm\sqrt{\tilde{\Delta}}\}^{H(d-1)}\times\{0\}^{H}$
and $\{0\}^{H\times d}$, 
\[
\begin{split} & \sum_{\tilde{\Ub}\in\Ucal\cup\{0\}^{H\times d}}\frac{H\epsilon}{8(d-1)}\sum_{h=1}^{\lfloor H/2\rfloor}\sum_{i=1}^{d-1}\EE_{\Thetab^{\star}+\tilde{\Ub},\Pi_{K}}\Big[M_{h,i}\big(\text{sign}(\tilde{u}_{h,i})\big)\Big]\\
 & \ge\sum_{\tilde{\Ub}\in\Ucal}\frac{H\epsilon}{8(d-1)}\sum_{h=1}^{\lfloor H/2\rfloor}\sum_{i=1}^{d-1}\EE_{\Thetab^{\star}+\tilde{\Ub},\Pi_{K}}\Big[M_{h,i}\big(\text{sign}(\tilde{u}_{h,i})\big)\Big]\\
 & \ge\frac{H\epsilon}{8(d-1)}\sum_{h=1}^{\lfloor H/2\rfloor}\sum_{i=1}^{d-1}\sum_{\tilde{\Ub}\in\Ucal}\EE_{\Thetab^{\star}+\tilde{\Ub},\Pi_{K}}\bigg[M_{h,i}\big(\Sgn(\tilde{u}_{h,i})\big)\bigg]\\
 & =\frac{H\epsilon}{16(d-1)}\sum_{h=1}^{\lfloor H/2\rfloor}\sum_{i=1}^{d-1}\sum_{\tilde{\Ub}\in\Ucal}\EE_{\Thetab^{\star}+\tilde{\Ub},\Pi_{K}}\bigg[M_{h,i}\big(-\Sgn(\tilde{u}_{h,i})\big)\bigg]+\EE_{\Thetab^{\star}+\tilde{\Ub},\Pi_{K}}\bigg[M_{h,i}\big(\Sgn(\tilde{u}_{h,i})\big)\bigg]\\
 & \ge\frac{H\epsilon}{16(d-1)}\sum_{h=1}^{\lfloor H/2\rfloor}\sum_{i=1}^{d-1}\sum_{\tilde{\Ub}\in\Ucal}\bigg(K-K\sqrt{2\tilde{\Delta}^{2}\sum_{k=1}^{K}\EE_{\Thetab^{\star}+\tilde{\Ub},\Pi_{K}}\big[\sigma_{k,h}^{2}]+\frac{4N\tilde{\Delta}^{3}}{\sqrt{6\kappa}}}\bigg).
\end{split}
\]
For each $h\in[H]$, define the variance-maximizing pertubation 
\[
\widehat{\Ub}_{h}(\Pi_{K}):=\argmax{\Ub\in\{\pm\sqrt{\tilde{\Delta}}\}^{H(d-1)}\times\{0\}^{H}\cup\{0\}^{H\times d}}\EE_{\Thetab^{\star}+\tilde{\Ub},\Pi_{K}}\big[\sum_{k=1}^{K}\sigma_{k,h}^{2}\big].
\]
Then, 
\[
\begin{split} & \sum_{\tilde{\Ub}\in\Ucal\cup\{0\}^{H\times d}}\frac{H\epsilon}{16(d-1)}\sum_{h=1}^{\lfloor H/2\rfloor}\sum_{i=1}^{d-1}\EE_{\Thetab^{\star}+\tilde{\Ub},\Pi_{K}}\Big[M_{h,i}\big(\text{sign}(\tilde{u}_{h,i})\big)\Big]\\
 & \ge\frac{H\epsilon2^{H(d-1)}}{16(d-1)}\sum_{h=1}^{\lfloor H/2\rfloor}\sum_{i=1}^{d-1}\bigg(K-K\sqrt{2\tilde{\Delta}^{2}\sum_{k=1}^{K}\EE_{\Thetab^{\star}+\widehat{\Ub}_{h}(\Pi_{K}),\Pi_{K}}\big[\sigma_{k,h}^{2}]+\frac{4N\tilde{\Delta}^{3}}{\sqrt{6\kappa}}}\bigg).
\end{split}
\]
Using the averaging hammer, 
\[
\begin{split} & \max_{\tilde{\Ub}\in\Ucal\cup\{0\}^{H\times d}}\frac{H\epsilon}{8(d-1)}\sum_{h=1}^{\lfloor H/2\rfloor}\sum_{i=1}^{d-1}\EE_{\Thetab^{\star}+\tilde{\Ub},\Pi_{K}}\Big[M_{h,i}\big(\text{sign}(\tilde{u}_{h,i})\big)\Big]\\
 & \ge\frac{1}{|\Ucal|+1}\sum_{\tilde{\Ub}\in\Ucal\cup\{0\}^{H\times d}}\frac{H\epsilon}{8(d-1)}\sum_{h=1}^{\lfloor H/2\rfloor}\sum_{i=1}^{d-1}\EE_{\Thetab^{\star}+\tilde{\Ub},\Pi_{K}}\Big[M_{h,i}\big(\text{sign}(\tilde{u}_{h,i})\big)\Big]\\
 & \ge\frac{2^{H(d-1)}}{2^{H(d-1)}+1}\cdot\frac{H\epsilon}{16(d-1)}\sum_{h=1}^{\lfloor H/2\rfloor}\sum_{i=1}^{d-1}\bigg(K-K\sqrt{2\tilde{\Delta}^{2}\sum_{k=1}^{K}\EE_{\Thetab^{\star}+\tilde{\Ub},\Pi_{K}}\big[\sigma_{k,h}^{2}]+\frac{4N\tilde{\Delta}^{3}}{\sqrt{6\kappa}}}\bigg)\\
 & \ge\frac{H\epsilon}{17(d-1)}\sum_{h=1}^{\lfloor H/2\rfloor}\sum_{i=1}^{d-1}\bigg(K-K\sqrt{2\tilde{\Delta}^{2}\sum_{k=1}^{K}\EE_{\Thetab^{\star}+\tilde{\Ub},\Pi_{K}}\big[\sigma_{k,h}^{2}]+\frac{4N\tilde{\Delta}^{3}}{\sqrt{6\kappa}}}\bigg)\\
 & =\frac{HN\epsilon}{17}\sum_{h=1}^{\lfloor H/2\rfloor}\bigg(1-\sqrt{2\tilde{\Delta}^{2}\sum_{k=1}^{K}\EE_{\Thetab^{\star}+\tilde{\Ub},\Pi_{K}}\big[\sigma_{k,h}^{2}]+\frac{4N\tilde{\Delta}^{3}}{\sqrt{6\kappa}}}\bigg),
\end{split}
\]
where the last inequality uses $2^{H(d-1)}\ge16$. By Cauchy-Schwartz
inequality, 
\[
\begin{split} & \max_{\tilde{\Ub}\in\Ucal\cup\{0\}^{H\times d}}\frac{H\epsilon}{8(d-1)}\sum_{h=1}^{\lfloor H/2\rfloor}\sum_{i=1}^{d-1}\EE_{\Thetab^{\star}+\tilde{\Ub},\Pi_{K}}\Big[M_{h,i}\big(\text{sign}(\tilde{u}_{h,i})\big)\Big]\\
 & \ge\frac{HN\epsilon}{17}\sum_{h=1}^{\lfloor H/2\rfloor}\bigg(1-\sqrt{2\tilde{\Delta}^{2}\EE_{\Thetab^{\star}+\widehat{\Ub}_{h}(\Pi_{K}),\Pi_{K}}\big[\sum_{k=1}^{K}\sigma_{k,h}^{2}]+\frac{4N\tilde{\Delta}^{3}}{\sqrt{6\kappa}}}\bigg).\\
 & =\frac{HN\epsilon}{17}\bigg(\lfloor\frac{H}{2}\rfloor-\sum_{h=1}^{\lfloor H/2\rfloor}\sqrt{2\tilde{\Delta}^{2}\EE_{\Thetab^{\star}+\widehat{\Ub}_{h}(\Pi_{K}),\Pi_{K}}\big[\sum_{k=1}^{K}\sigma_{k,h}^{2}]+\frac{4N\tilde{\Delta}^{3}}{\sqrt{6\kappa}}}\bigg)\\
 & \ge\frac{HN\epsilon}{17}\bigg(\frac{H}{3}-\sqrt{\frac{H}{2}}\sqrt{2\tilde{\Delta}^{2}\sum_{h=1}^{\lfloor H/2\rfloor}\EE_{\Thetab^{\star}+\widehat{\Ub}_{h}(\Pi_{K}),\Pi_{K}}\big[\sum_{k=1}^{K}\sigma_{k,h}^{2}]+\frac{2HN\tilde{\Delta}^{3}}{\sqrt{6\kappa}}}\bigg)\\
 & =\frac{H^{2}K\epsilon}{51}\bigg(1-3\sqrt{\frac{\tilde{\Delta}^{2}}{H}\sum_{h=1}^{\lfloor H/2\rfloor}\EE_{\Thetab^{\star}+\widehat{\Ub}_{h}(\Pi_{K}),\Pi_{K}}\big[\sum_{k=1}^{K}\sigma_{k,h}^{2}]+\frac{K\tilde{\Delta}^{3}}{\sqrt{6\kappa}}}\bigg).
\end{split}
\]
It follows that 
\[
\begin{split} & \max_{\tilde{\Ub}\in\Ucal\cup\{0\}^{H\times d}}\frac{H\epsilon}{8(d-1)}\sum_{h=1}^{\lfloor H/2\rfloor}\sum_{i=1}^{d-1}\EE_{\Thetab^{\star}+\tilde{\Ub},\Pi_{K}}\Big[M_{h,i}\big(\text{sign}(\tilde{u}_{h,i})\big)\Big]\\
 & \ge\frac{H^{2}K\epsilon}{51}\bigg(1-3\sqrt{\frac{\tilde{\Delta}^{2}}{H}\sum_{h=1}^{H}\EE_{\Thetab^{\star}+\widehat{\Ub}_{h}(\Pi_{K}),\Pi_{K}}\big[\sum_{k=1}^{K}\sigma_{k,h}^{2}]+\frac{K\tilde{\Delta}^{3}}{\sqrt{6\kappa}}}\bigg).
\end{split}
\]
Because $\sqrt{ax^{2}+bx^{3}}\le\sqrt{a}x+\frac{b}{\sqrt{a}}x^{2}$
for $a,b,x>0$, 
\[
\begin{split} & \sqrt{\frac{\tilde{\Delta}^{2}}{H}\sum_{h=1}^{H}\EE_{\Thetab^{\star}+\widehat{\Ub}_{h}(\Pi_{K}),\Pi_{K}}\big[\sum_{k=1}^{K}\sigma_{k,h}^{2}]+\frac{K\tilde{\Delta}^{3}}{\sqrt{6\kappa}}}\\
 & \le\sqrt{\frac{1}{H}\sum_{h=1}^{H}\EE_{\Thetab^{\star}+\widehat{\Ub}_{h}(\Pi_{K}),\Pi_{K}}\big[\sum_{k=1}^{K}\sigma_{k,h}^{2}]}\tilde{\Delta}+\frac{K\sqrt{H}\tilde{\Delta}^{2}}{\sqrt{6\kappa\sum_{h=1}^{H}\EE_{\Thetab^{\star}+\widehat{\Ub}_{h}(\Pi_{K}),\Pi_{K}}\big[\sum_{k=1}^{K}\sigma_{k,h}^{2}]}}.
\end{split}
\]
Let us write $\Xi(\Pi_{K}):=\sum_{h=1}^{H}\sum_{k=1}^{K}\EE_{\tilde{\Scal}_{\Thetab^{\star}+\widehat{\Ub}_{h}(\Pi_{K})},\Pi_{K}}\big[\sigma_{k,h}^{2}]$.
It follows that 
\[
\begin{split} & \max_{\tilde{\Ub}\in\Ucal\cup\{0\}^{H\times d}}\frac{H\epsilon}{8(d-1)}\sum_{h=1}^{\lfloor H/2\rfloor}\sum_{i=1}^{d-1}\EE_{\Thetab^{\star}+\tilde{\Ub},\Pi_{K}}\Big[M_{h,i}\big(\text{sign}(\tilde{u}_{h,i})\big)\Big]\\
 & \ge\frac{H^{2}K\epsilon}{51}\bigg(1-3\sqrt{\frac{\Xi(\Pi_{K})}{H}}\tilde{\Delta}-\frac{3N\sqrt{H}\kappa^{-1/2}\tilde{\Delta}^{2}}{\sqrt{6\Xi(\Pi_{K})}}\bigg)\\
 & =\frac{H^{2}K\epsilon}{51}\bigg(1-3\sqrt{\frac{\Xi(\Pi_{K})}{H}}\tilde{\Delta}-\frac{\sqrt{3}K\sqrt{H}\kappa^{-1/2}\tilde{\Delta}^{2}}{\sqrt{2\Xi(\Pi_{K})}}\bigg)
\end{split}
\]
Setting $\epsilon=\frac{12(d-1)\sqrt{\Xi(\Pi_{K})}}{NH}$ we have
\[
\begin{split} & \max_{\tilde{\Ub}\in\Ucal\cup\{0\}^{H\times d}}\frac{H\epsilon}{8(d-1)}\sum_{h=1}^{\lfloor H/2\rfloor}\sum_{i=1}^{d-1}\EE_{\Thetab^{\star}+\tilde{\Ub},\Pi_{K}}\Big[M_{h,i}\big(\text{sign}(\tilde{u}_{h,i})\big)\Big]\\
 & \ge\frac{4H^{3/2}(d-1)\sqrt{\Xi(\Pi_{K})}}{17}\bigg(1-3\sqrt{\frac{\Xi(\Pi_{K})}{H}}\tilde{\Delta}-\frac{\sqrt{3}K\sqrt{H}\kappa^{-1/2}\tilde{\Delta}^{2}}{\sqrt{2\Xi(\Pi_{K})}}\bigg).
\end{split}
\]
Because $K\ge144H(d-1)^{2}$, 
\[
\epsilon=\frac{12(d-1)\sqrt{\Xi(\Pi_{K})}}{NH}\le\frac{12(d-1)}{\sqrt{NH}}\le\frac{1}{H},
\]
Thus, 
\[
\begin{split} & \sup_{\Ub\in\{\pm\sqrt{\Delta}\}^{(d-1)H}\times\{0\}^{H\times d}}\Reg(\Thetab^{\star}+\Ub,\Pi_{K})\\
 & \ge\frac{4H(d-1)\sqrt{\Xi(\Pi_{K})}}{17}\bigg(1-3\sqrt{\frac{\Xi(\Pi_{K})}{H}}\tilde{\Delta}-\frac{\sqrt{3}K\sqrt{H}\kappa^{-1/2}\tilde{\Delta}^{2}}{\sqrt{2\Xi(\Pi_{K})}}\bigg)\\
 & \ge\frac{4H(d-1)\sqrt{\Xi(\Pi_{K})}}{17}\bigg(1-3\sqrt{K}\tilde{\Delta}-\frac{\sqrt{3}\sqrt{K}\tilde{\Delta}^{2}}{\kappa}\bigg),
\end{split}
\]
where the last inequality holds by the definition of the Fisher information lower bound (Definition \ref{def:kappa}) and
$\Xi(\Pi_{K})\ge\kappa NH$. By definition of $\widehat{\Thetab}_{h}(\Pi_{K})$.
\[
\begin{split} & \sup_{\Ub:\|\Ub\|_{\infty}\le\sqrt{\Delta}}\Reg(\Thetab^{\star}+\Ub,\Pi_{K})\\
 & \ge\frac{4H(d-1)\sqrt{\Xi(\Pi_{K})}}{17}\bigg(1-3\sqrt{K}\tilde{\Delta}-\frac{\sqrt{3}\sqrt{K}\tilde{\Delta}^{2}}{\kappa}\bigg)\\
 & \ge\frac{4H(d-1)\sqrt{\sum_{h=1}^{H}\EE_{\widehat{\Thetab}_{h}(\Pi_{K}),\Pi_{K}}\big[\sum_{k=1}^{K}\sigma_{k,h}^{2}]}}{17}\bigg(1-3\sqrt{K}\tilde{\Delta}-\frac{\sqrt{3}\sqrt{K}\tilde{\Delta}^{2}}{\kappa}\bigg)\\
 & \ge\frac{4H(d-1)\sqrt{\EE_{\Thetab^{\star},\Pi_{K}}[\sum_{h=1}^{H}\sum_{k=1}^{K}\sigma_{k,h}^{2}]}}{17}\bigg(1-3\sqrt{K}\tilde{\Delta}-\frac{\sqrt{3}\sqrt{K}\tilde{\Delta}^{2}}{\kappa}\bigg).
\end{split}
\]
Because $\Thetab\mapsto\Reg(\Thetab,\Pi_{K})$ is Lipchitz continuous,
letting $\Delta\to0$ implies $\tilde{\Delta}\to0$ and 
\[
\Reg(\Thetab^{\star},\Pi_{K})\ge\frac{4H(d-1)\sqrt{\EE_{\thetab^{\star},\Pi_{K}}[\sum_{h=1}^{H}\sum_{k=1}^{K}\sigma_{k,h}^{2}]}}{17}.
\]
Because $\sigma_{k,h}^{2}\in[0,1]$, by
Lemma~\ref{lem:anti}, for any $\delta\in(0,1/2)$, 
\[
\PP\Big(\sum_{h=1}^{H}\sum_{k=1}^{K}(1-\sigma_{k,h}^{2})\ge\EE_{\thetab^{\star},\Pi_{K}}[\sum_{h=1}^{H}\sum_{k=1}^{K}(1-\sigma_{k,h}^{2})]-\frac{\delta}{1-\delta}\sqrt{HN\log\frac{1}{\delta}}\Big)\ge\delta.
\]
Rearranging the terms, with probability at least $\delta\in(0,1)$,
\[
\EE_{\thetab^{\star},\Pi_{K}}[\sum_{h=1}^{H}\sum_{k=1}^{K}\sigma_{k,h}^{2}]\ge\sum_{h=1}^{H}\sum_{k=1}^{K}\sigma_{k,h}^{2}-\frac{\delta}{1-\delta}\sqrt{HN\log\frac{1}{\delta}}
\]
Because $K\ge H^{-1}\kappa^{-2}$, the right hand side is nonnegative
for all $\delta\in(0,1/e)$. Thus, with probability at least $\delta$,
\[
\begin{aligned}\Reg(\Thetab^{\star},\Pi_{K}) & \ge\frac{4H(d-1)\sqrt{\sum_{h=1}^{H}\sum_{k=1}^{K}\sigma_{k,h}^{2}-\frac{\delta}{1-\delta}\sqrt{HN\log\frac{1}{\delta}}}}{17}\\
 & \ge\frac{2Hd\sqrt{\sum_{h=1}^{H}\sum_{k=1}^{K}\sigma_{k,h}^{2}-\frac{\delta}{1-\delta}\sqrt{HN\log\frac{1}{\delta}}}}{17}.
\end{aligned}
\]
Because 
\[
\frac{\delta}{1-\delta}\sqrt{\log\frac{1}{\delta}}\sum_{h=1}^{H}\sum_{k=1}^{K}\sigma_{k,h}^{2}\ge\frac{\delta}{1-\delta}\sqrt{\log\frac{1}{\delta}}\kappa HN\ge\frac{\delta}{1-\delta}\sqrt{HN\log\frac{1}{\delta}},
\]
we obtain the lower bound. 
\end{proof}

\section{Technical Results}

\subsection{Properties of Self Concordant Functions}

\begin{lemma} \label{lem:log_sum_ratio}(Ratio of the self-concordant
function) Suppose $f$ is $M_{f}$-self-concordant. Then for any $\xb,\yb,\wb\in\RR^{d}$,
\[
\xb^{\top}\nabla^{2}f(\yb+\wb)\xb\le\exp(M_f\|\wb\|_{2})\xb^{\top}\nabla^{2}f(\yb)\xb.
\]
\end{lemma} 
\begin{proof}
\medspace{} Given $\xb,\yb,\wb\in\RR^{d}$ define $\psi(t)=\xb^{\top}\nabla^{2}f(\yb+t\wb)\xb$.
Then, by the mean value theorem, there exists $\tilde{t}\in[0,1]$
such that 
\[
\log\psi(1)-\log\psi(0)=\frac{\psi^{\prime}(\tilde{t})}{\psi(\tilde{t})}=\frac{\sum_{i=1}^{S}\xb^{\top}w_{i}D_{i}\nabla^{2}f(\yb+\tilde{t}\wb)\xb}{\xb^{\top}\nabla^{2}f(\yb+\tilde{t}\wb)\xb}.
\]
Because $f$ is $M_{f}$-self-concordant, 
\[
\sum_{i=1}^{S}\xb^{\top}w_{i}D_{i}\nabla^{2}f(\yb+\tilde{t}\wb)\xb\le M_{f}\|\wb\|_{2}\xb^{\top}\nabla^{2}f(\yb+\tilde{t}\wb)\xb.
\]
Thus, 
\[
\log\psi(1)-\log\psi(0)\le M_{f}\|\wb\|_{2},
\]
which holds for any $\xb\in\RR^{S}$, Thus, 
\[
\psi(1)\le\psi(0)\exp(M_{f}\|\wb\|_{2}),
\]
which completes the proof. 
\end{proof}
\begin{proposition} \label{prop:1st_hessian} Let $g:\RR^{d}\to\RR$
denote a self-concordant function with $M_{g}$. Then, for any $\xb_{1},\xb_{2}\in\RR^{d}$,
\[
\int_{0}^{1}\nabla^{2}L_{\nu,h}\big(\xb_{1}+t(\xb_{2}-\xb_{1})\big)dt\succeq\frac{\nabla^{2}L_{\nu,h}(\xb_{1})}{1+M_{g}\|\xb_{1}-\xb_{2}\|_{2}},
\]
and 
\[
\int_{0}^{1}\nabla^{2}L_{\nu,h}\big(\xb_{1}+t(\xb_{2}-\xb_{1})\big)dt\succeq\frac{\nabla^{2}L_{\nu,h}(\xb_{2})}{1+M_{g}\|\xb_{1}-\xb_{2}\|_{2}}.
\]
\end{proposition}
\begin{proof}
\medspace{} By Proposition 8 in \citealt{sun2019generalized}, 
\[
\nabla^{2}L_{\nu,h}(\xb_{1})e^{-M_{g}t\|\xb_{1}-\xb_{2}\|_{2}}\preceq\nabla^{2}L_{\nu,h}\big(\xb_{1}+t(\xb_{2}-\xb_{1})\big).
\]
Then, 
\[
\begin{split}\int_{0}^{1}\nabla^{2}L_{\nu,h}\big(\xb_{1}+t(\xb_{2}-\xb_{1})\big)\succeq & \nabla^{2}L_{\nu,h}(\xb_{1})\int_{0}^{1}e^{-M_{g}t\|\xb_{1}-\xb_{2}\|_{2}}dt\\
= & \nabla^{2}L_{\nu,h}(\xb_{1})\frac{1-e^{-M_{g}\|\xb_{1}-\xb_{2}\|_{2}}}{M_{g}\|\xb_{1}-\xb_{2}\|_{2}}\\
\succeq & \nabla^{2}L_{\nu,h}(\xb_{1})\frac{1}{1+M_{g}\|\xb_{1}-\xb_{2}\|_{2}},
\end{split}
\]
where the last inequality uses $(1-e^{-x})/x\ge1/(1+x)$ which is
implied by $e^{x}\ge1+x$. For the second inequality, changing the
variable $u=1-t$, 
\[
\begin{split}\int_{0}^{1}\nabla^{2}L_{\nu,h}\big(\xb_{1}+t(\xb_{2}-\xb_{1})\big)dt= & \int_{0}^{1}\nabla^{2}L_{\nu,h}\big(\xb_{1}+(1-u)(\xb_{2}-\xb_{1})\big)du\\
= & \int_{0}^{1}\nabla^{2}L_{\nu,h}\big(\xb_{2}+t(\xb_{1}-\xb_{2})\big)dt,
\end{split}
\]
completes the proof. 
\end{proof}
\begin{lemma} \label{lem:con_hessian_upper} (Lemma 16 in \citealt{lee2024improved})
Let $g:\RR^{d}\to\RR$ denote a function that satisfies 
\[
e^{d(x_{1},x_{2})}\nabla^{2}L_{\nu,h}(\xb_{2})\preceq\nabla^{2}L_{\nu,h}(\xb_{1})\preceq e^{d(x_{1},x_{2})}\nabla^{2}L_{\nu,h}(\xb_{2}),
\]
for any $\xb_{1},\xb_{2}\in\RR^{d}$ and a metric $d:\RR^{d}\times\RR^{d}\to\RR_{+}$.
Then for any $\xb_{1},\xb_{2}$ in the domain of $g$, 
\[
\int_{0}^{1}\nabla^{2}L_{\nu,h}\big(\xb_{1}+t(\xb_{2}-\xb_{1})\big)(1-t)dt\succeq\frac{1}{2+d(\xb_{1},\xb_{2})}\nabla^{2}L_{\nu,h}(\xb_{1}),
\]
\end{lemma}

We develop a novel bound for the other side. \begin{lemma} \label{lem:con_Hessian_lower}
Let $g:\RR^{d}\to\RR$ denote a function that satisfies 
\[
e^{d(x_{1},x_{2})}\nabla^{2}L_{\nu,h}(\xb_{2})\preceq\nabla^{2}L_{\nu,h}(\xb_{1})\preceq e^{d(x_{1},x_{2})}\nabla^{2}L_{\nu,h}(\xb_{2}),
\]
for any $\xb_{1},\xb_{2}\in\RR^{d}$ and a metric $d:\RR^{d}\times\RR^{d}\to\RR_{+}$.
Then for any $\xb_{1},\xb_{2}$ in the domain of $g$, 
\[
\int_{0}^{1}\nabla^{2}L_{\nu,h}\big(\xb_{1}+t(\xb_{2}-\xb_{1})\big)(1-t)dt\succeq\frac{1}{3+d(\xb_{1},\xb_{2})^{2}}\nabla^{2}L_{\nu,h}(\xb_{2}).
\]
\end{lemma}
\begin{proof}
\medspace{} Note that 
\[
\begin{split}\int_{0}^{1}\nabla^{2}L_{\nu,h}\big(\xb_{1}+t(\xb_{2}-\xb_{1})\big)(1-t)dt\succeq & \bigg(\int_{0}^{1}e^{-(1-t)d(\xb_{1},\xb_{2})}(1-t)dt\bigg)\nabla^{2}L_{\nu,h}(\xb_{2})\\
= & \bigg(\frac{1-(1+d(\xb_{1},\xb_{2}))e^{-d(\xb_{1},\xb_{2})}}{d(\xb_{1},\xb_{2})^{2}}\bigg)\nabla^{2}L_{\nu,h}(\xb_{2}).
\end{split}
\]
Note that for $x>0$, we have $e^{x}\ge x^{3}/3+x^{2}/3+x+1$. Rearranging
the terms, 
\[
(x+1)e^{-x}\le\frac{3}{3+x^{2}},
\]
which imples 
\[
\frac{1-(1+x)e^{-x}}{x^{2}}\ge\frac{1}{3+x^{2}}.
\]
Plugging in $x=d(\xb_{1},\xb_{2})$ completes the proof. 
\end{proof}

\subsection{Properties of Log-Sum Function}

\begin{proposition}[Moments of state transition]\label{prop:mean_var}
Let $s_{h+1}$ denote a random variable distributed according to the
transition kernel \eqref{eq:transit_kernel}. For $\Scal^{\prime}\subseteq\Scal$,
let 
\begin{equation}
L_{\Scal^{\prime}}(\ub):=\log\Big(\sum_{s\in\Scal^{\prime}}\exp(\eb_{s}^{\top}\ub)\Big)\label{eq:log-sum-exp}
\end{equation}
denotes the log-sum function. Then the (conditional) mean and covariance
of the Euclidean basis vector $\eb_{s_{h+1}}$ is given by 
\begin{equation}
\begin{aligned}\EE[\eb_{s_{h+1}}\mid s_{h},a_{h},\thetab_{h}^{\star}] & =\nabla L_{\Scal_{h}(s_{h},a_{h})}\big(\vPhib(s_{h},a_{h})\thetab_{h}^{\star}\big),\notag\\
\VV[\eb_{s_{h+1}}\mid s_{h},a_{h},\thetab_{h}^{\star}] & =\nabla^{2}L_{\Scal_{h}(s_{h},a_{h})}\big(\vPhib(s_{h},a_{h})\thetab_{h}^{\star}\big).
\end{aligned}
\label{eq:moment-expressions}
\end{equation}
\end{proposition}

\begin{proposition} \label{prop:log_sum} \citep[Lemma 4]{tran2015composite}
The log-sum function $L_{\Scal}(\ub)=\log\big(\sum_{s\in\Scal}\exp(\eb_{s}^{\top}\ub)\big)$
is $\sqrt{6}$-self-concordant. \end{proposition}
\begin{proof}
\medspace{} For $t\in\RR$ and $\xb,\ub\in\RR^{d}$, let $\psi(t):=L_{\Scal}(\xb+t\ub)$.
According to Theorem 3 in \citep{tran2015composite} it is sufficient
to prove 
\[
\abs{\psi^{\prime\prime\prime}(t)}\le \sqrt{6}\psi^{\prime\prime}(t)\norm{\ub}_{2}.
\]
Let us fix $t\in\RR$ and $\xb,\ub\in\RR^{d}$ throughout the proof.
For each $i\in[d]$, let $p_{i}=e^{x_{i}+tu_{i}}/(\sum_{j=1}^{d}e^{x_{j}+tu_{j}})$.
According to Section A.6 in \citet{tran2015composite}, 
\[
\begin{split}\psi^{\prime\prime}(t) & :=\sum_{i<j}(u_{i}-u_{j})^{2}p_{i}p_{j}\\
\psi^{\prime\prime\prime}(t) & :=\sum_{i<j}(u_{i}-u_{j})^{2}p_{i}p_{j}\big\{\sum_{k=1}^{d}(u_{i}+u_{j}-2u_{k})p_{k}\big\}.
\end{split}
\]
Suppose $k\in[d]\setminus\{i,j\}$. Then, by Cauchy-Schwarz inequality,
\[
|u_{i}+u_{j}-2u_{k}|\le\sqrt{1+1+4}\sqrt{u_{i}^{2}+u_{j}^{2}+u_{k}^{2}}\le\sqrt{6}\norm{\ub}_{2}.
\]
Suppose $k=i$, then 
\[
|u_{i}+u_{j}-2u_{k}|=|u_{i}-u_{j}|\le\sqrt{2}\sqrt{u_{i}^{2}+u_{j}^{2}}\le\sqrt{6}\norm{\ub}_{2}.
\]
Thus, 
\[
\begin{split}\abs{\psi^{\prime\prime\prime}(t)}\le & \sum_{i<j}(u_{i}-u_{j})^{2}p_{i}p_{j}\big\{\sum_{k=1}^{d}|u_{i}+u_{j}-2u_{k}|p_{k}\big\}\\
\le & \sqrt{6}\|\ub\|_{2}\sum_{i<j}(u_{i}-u_{j})^{2}p_{i}p_{j}\\
= & \sqrt{6}\|\ub\|_{2}\psi^{\prime\prime}(t),
\end{split}
\]
which completes the proof. 
\end{proof}
\begin{proposition} \label{prop:multi_lower_bound}For $L_{\Scal}(\etab):=\log\big(\sum_{s\in\Scal}\exp(\eb_{s}^{\top}\etab)\big)$
and $\etab_{1},\etab_{2}\in[-C,C]^{d}$ for some $C>0$, define $\psi(t):=(\etab_{2}-\etab_{1})^{\top}\nabla^{2}L_{\Scal}\big(\etab_{1}+t(\etab_{2}-\etab_{1})\big)(\etab_{2}-\etab_{1})$
for $t\in[0,1]$. Then we have 
\[
\int_{0}^{1}\psi(t)(1-t)dt\ge\frac{\psi(0)}{2+2C},
\]
and 
\[
\int_{0}^{1}\psi(t)(1-t)dt\ge\frac{\psi(1)}{3+4C^{2}}.
\]
\end{proposition}
\begin{proof}
\medspace{} The proof follows that in Lemma 4 in \citealt{jezequel2021mixability}.
Let $X$ denote a random variable which takes value $\eb_{i}^{\top}(\etab_{2}-\etab_{1})$
with probability 
\[
p_{i}(t):=\frac{\exp\big(\eb_{i}^{\top}\etab_{1}+t\eb_{i}^{\top}(\etab_{2}-\etab_{1})\big)}{\sum_{j=1}^{d}\exp\big(\eb_{j}^{\top}\etab_{1}+t\eb_{j}^{\top}(\etab_{2}-\etab_{1})\big)}.
\]
By (24) in \citealt{jezequel2021mixability}, 
\[
\begin{split}\psi^{\prime}(t)= & \sum_{i_{1},i_{2},i_{3}}\Big\{\nabla^{3}g\big(\etab_{1}+t(\etab_{2}-\etab_{1})\big)\Big\}_{i,j,k}(\eta_{2,i_{1}}-\eta_{1,i_{1}})(\eta_{2,i_{2}}-\eta_{1,i_{2}})(\eta_{2,i_{3}}-\eta_{1,i_{3}})\\
= & \EE\big[(X-\EE(X))^{3}\big].
\end{split}
\]
By (25) in \citealt{jezequel2021mixability}, 
\[
\psi(t)=\EE\big[(X-\EE(X))^{2}\big].
\]
Because $\etab_{1}\in[-C,C]$, 
\[
\abs{\psi^{\prime}(t)}\le2C\EE\big[(X-\EE(X))^{2}\big]\le2C\psi(t),
\]
Setting $h^{\prime\prime}(t)=\psi(t),$ we obtain an $2C$-self-concordant
function $h$. By Proposition 8 in \citealt{sun2019generalized},
\[
h^{\prime\prime}(t_{1})e^{-2C|t_{1}-t_{2}|}\le h^{\prime\prime}(t_{2})\le h^{\prime\prime}(t_{1})e^{2C|t_{1}-t_{2}|},
\]
for $t_{1},t_{2}\in[0,1]$. Applying Lemma \ref{lem:con_hessian_upper}
and Lemma \ref{lem:con_Hessian_lower} on $h^{\prime\prime}=\psi$
completes the proof. 
\end{proof}

\subsection{Elliptical Potential Lemmas}

\begin{lemma} \label{lem:abbasi_elliptical} (Eliptical Potential
Lemma; Lemma 11 in \citet{abbasi2011improved}) Let $\{\xb_{t}:t\ge1\}$
be a sequence in $\RR^{d}$ such that $\|\xb_{t}\|_{2}\le L$ for
some $L>0$. Define $\bar{\Vb}_{t}=\sum_{s=1}^{t}\xb_{s}\xb_{s}^{\top}+V$
for a positive definite matrix $V$ such that $\lambda_{\min}(V)\ge L^{2}$.
Then, 
\[
\sum_{t=1}^{k}\|\xb_{t}\|_{\bar{\Vb}_{t-1}^{-1}}^{2}\le2d\log\frac{(k+1)}{d}.
\]
\end{lemma}

\begin{lemma} \label{lem:vPhib_potential} (Eliptical Potential Lemma
for Features) For the feature matrix $\vPhib_{k,h}:=\vPhib(s_{k,h},a_{k,h})$
let $\Ab_{k,h}:=\sum_{\nu=1}^{k}\vPhib_{\nu,h}^{\top}\vPhib_{\nu,h}+B_{\vPhib}^{2}\Ib_{d}$
denote the Gram matrix at episode $k$ and step $h$. Then, for each
$h\in[H]$, 
\[
\sum_{k=1}^{K}\max_{s\in\Scal_{k,h}}\|\vPhib_{k,h}^{\top}\eb_{s}\|_{\Ab_{k-1,h}^{-1}}^{2}\le2d\log\frac{K+1}{d}.
\]
\end{lemma}
\begin{proof}
\medspace{} For each $k\in[K]$ and $h\in[H]$ let $\tilde{s}_{k,h}:=\argmax{s\in\Scal_{k,h}}\|\vPhib_{k,h}^{\top}\eb_{s}\|_{\Ab_{k-1,h}^{-1}}^{2}$.
Define $\tilde{A}_{k-1,h}:=\sum_{\nu=1}^{k-1}\vPhib_{\nu,h}^{\top}\eb_{\tilde{s}_{\nu,h}}\eb_{\tilde{s}_{\nu,h}}^{\top}\vPhib_{\nu,h}+B_{\vPhib}^{2}\Ib_{d}$.
Then $\Ab_{k-1,h}\succeq\tilde{\Ab}_{k-1,h}$ and 
\[
\sum_{k=1}^{K}\max_{s\in\Scal_{k,h}}\|\vPhib_{k,h}^{\top}\eb_{s}\|_{\Ab_{k-1,h}^{-1}}^{2}=\sum_{k=1}^{K}\|\vPhib_{k,h}^{\top}\eb_{\tilde{s}_{k,h}}\|_{\Ab_{k-1,h}^{-1}}^{2}\le\sum_{k=1}^{K}\|\vPhib_{k,h}^{\top}\eb_{\tilde{s}_{k,h}}\|_{\tilde{\Ab}_{k-1,h}^{-1}}^{2}.
\]
Because $\|\vPhib_{k,h}^{\top}\eb_{\tilde{s}_{k,h}}\|_{2}\le B_{\vPhib}$,
applying Lemma \ref{lem:abbasi_elliptical}, gives, 
\[
\sum_{k=1}^{K}\|\vPhib_{k,h}^{\top}\eb_{\tilde{s}_{k,h}}\|_{\tilde{\Ab}_{k-1,h}^{-1}}^{2}\le2d\log\frac{K+1}{d},
\]
which completes the proof. 
\end{proof}

\subsection{Regret Bound for Online Newton Step}

\begin{lemma}[Regret of the ONS estimate $\thetab_{k,h}$] \label{lem:online_regret}
Let $\thetab_{k,h}$ be the ONS estimate of $\thetab_{h}^{\star}$
obtained with the initial matrix $\widehat{\Hb}_{0,h}:=\epsilon I_{d}$,
where $\epsilon>4B_{\vPhib}^{2}$. Set the learning rate $\eta_{\vPhib,\thetab}:=(e-1)(3+4B_{\vPhib}^{2}B_{\thetab}^{2})$.
Then, with probability at least $1-\delta$, 
\[
\sum_{\nu=1}^{k}\big(\ell_{\nu,h}(\thetab_{\nu-1,h})-\ell_{\nu,h}(\thetab_{h}^{\star})\big)\le\frac{2B_{\thetab}^{2}\epsilon}{\eta_{\vPhib,\thetab}}+d\eta_{\vPhib,\thetab}\log\frac{k+1}{d}+\frac{8B_{\vPhib}^{2}B_{\thetab}^{2}}{\eta_{\vPhib,\thetab}}\log\frac{1}{\delta}
\]
for all $h\in[H]$ and $k\in[K]$. \end{lemma}

Lemma~\ref{lem:online_regret} provides, to our knowledge, the first
\emph{$\kappa$-free} regret bound for the multinomial ONS estimator.
Earlier $\kappa$-free results in this setting (e.g., \citealp{cho2024randomized,li2024provably})
are derived for Online Mirror Descent based on \citet{zhang2024online},
while the ONS bound for generalized linear models in \citet{jun2017scalable}
retains a $\kappa$-dependence. The regret bound in \citet{hazan2007logarithmic}
is for $\kappa^{-1}$-exp concave loss functions, and the resulting
regret bound has $\kappa^{-1}$; the $\kappa$-free regret guarantees
on MNL regression have not been established yet. 
\begin{proof}
By definition of the Bregman divergence, for $\nu\in[k]$ and $h\in[H]$
\begin{equation}
\ell_{\nu,h}(\thetab_{\nu-1,h})-\ell_{\nu,h}(\thetab_{h}^{\star})=\nabla\ell_{\nu,h}(\thetab_{\nu-1,h})^{\top}(\thetab_{\nu-1,h}-\thetab_{h}^{\star})-B_{\ell_{\nu,h}}(\thetab_{h}^{\star},\thetab_{\nu-1,h}).\label{eq:likelihood_Bregman}
\end{equation}
For brevity, let $L_{\Scal_{h}(s_{\nu,h},a_{\nu,h})}:=L_{\nu,h}$
denote the log-sum function for $\Scal_{\nu,h}:=\Scal_{h}(s_{\nu,h},a_{\nu,h})$
and let $\vPhib_{\nu,h}:=\vPhib(s_{\nu,h},a_{\nu,h})$. Then, by definition
\[
\begin{split}\nabla\ell_{\nu,h}(\thetab_{\nu-1,h}) & =\vPhib_{\nu,h}^{\top}\big(\nabla L_{\nu,h}(\vPhib_{\nu,h}\thetab_{\nu-1,h})-\eb_{s_{\nu,h+1}}\big)\end{split}
.
\]
By the information matrix bound (Lemma~\ref{lem:Information_matrix_bound})
\begin{align*}
&\sum_{\nu=1}^{k}\big(\nabla\ell_{\nu,h}(\thetab_{\nu-1,h})^{\top}(\thetab_{\nu-1,h}-\thetab_{h}^{\star})\big)^{2}\\
&\le
 2(e-1)B_{\vPhib}^{2}B_{\thetab}^{2}\sum_{\nu=1}^{k}B_{\ell_{\nu,h}}(\thetab_{h}^{\star},\thetab_{\nu-1,h})\\
 &+(e-1)\|\thetab_{\nu-1,h}-\thetab_{h}^{\star}\|_{\nabla^{2}\ell_{\nu,h}(\thetab_{h}^{\star})}^{2}\\
 & +16B_{\vPhib}^{2}B_{\thetab}^{2}\log\frac{1}{\delta}.
\end{align*}
Let $\psi_{\nu,h}(t):=L_{\nu,h}\big(\vPhib_{\nu,h}\thetab_{\nu-1,h}+t\vPhib_{\nu,h}(\thetab_{h}^{\star}-\thetab_{\nu-1,h})\big).$
By definition of Bregman divergence, 
\begin{align*}
B_{\ell_{\nu,h}}(\thetab_{h}^{\star},\thetab_{\nu-1,h})= & \ell_{\nu,h}(\thetab_{h}^{\star})-\ell_{\nu,h}(\thetab_{\nu-1,h})-\nabla\ell_{\nu,h}(\thetab_{\nu-1,h})(\thetab_{h}^{\star}-\thetab_{\nu-1,h})\\
= & \frac{1}{2}\int_{0}^{1}\psi_{\nu,h}^{\prime\prime}(t)(1-t)dt.
\end{align*}
Because $\|\vPhib_{\nu,h}\thetab_{\nu-1,h}\|_{\infty}\le B_{\vPhib}B_{\thetab}$
and $\|\vPhib_{\nu,h}\thetab_{h}^{\star}\|_{\infty}\le B_{\vPhib}B_{\thetab}$,
by the lower bound for the log-sum function (Proposition~\ref{prop:multi_lower_bound}),
\begin{align*}
&\int_{0}^{1}\psi_{\nu,h}^{\prime\prime}(t)(1-t)dt \\
&\ge \frac{\|\vPhib_{\nu,h}(\thetab_{h}^{\star}-\thetab_{\nu-1,h})\|_{\nabla^{2}L_{\nu,h}(\thetab_{h}^{\star})}}{3+4B_{\vPhib}^{2}B_{\thetab}^{2}}\\
&= \frac{\|\thetab_{\nu-1,h}-\thetab_{h}^{\star}\|_{\nabla^{2}\ell_{\nu,h}(\thetab_{h}^{\star})}}{3+4B_{\vPhib}^{2}B_{\thetab}^{2}}.
\end{align*}
Thus, $\|\thetab_{\nu-1,h}-\thetab_{h}^{\star}\|_{\nabla^{2}\ell_{\nu,h}(\thetab_{h}^{\star})}\le(3+4B_{\vPhib}^{2}B_{\thetab}^{2})B_{\ell_{\nu,h}}(\thetab_{h}^{\star},\thetab_{\nu-1,h})$.
Now we have 
\begin{align*}
&\sum_{\nu=1}^{k}\big(\nabla\ell_{\nu,h}(\thetab_{\nu-1,h})^{\top}(\thetab_{\nu-1,h}-\thetab_{h}^{\star})\big)^{2}\\
&\le(e-1)(3+4B_{\vPhib}^{2}B_{\thetab}^{2})\sum_{\nu=1}^{k}B_{\ell_{\nu,h}}(\thetab_{h}^{\star},\thetab_{\nu-1,h})+16B_{\vPhib}^{2}B_{\thetab}^{2}\log\frac{1}{\delta}.
\end{align*}
Rearranging the terms, 
\begin{align*}
&-\sum_{\nu=1}^{k}B_{\ell_{\nu,h}}(\thetab_{h}^{\star},\thetab_{\nu-1,h})\\
&\le\frac{-\sum_{\nu=1}^{k}\big(\nabla\ell_{\nu,h}(\thetab_{\nu-1,h})^{\top}(\thetab_{\nu-1,h}-\thetab_{h}^{\star})\big)^{2}+16B_{\vPhib}^{2}B_{\thetab}^{2}\log\frac{1}{\delta}}{(e-1)(3+4B_{\vPhib}^{2}B_{\thetab}^{2})}.
\end{align*}
From \eqref{eq:likelihood_Bregman}, 
\begin{equation}
\begin{split}
&\sum_{\nu=1}^{k}\big(\ell_{\nu,h}(\thetab_{\nu-1,h})-\ell_{\nu,h}(\thetab_{h}^{\star})\big)\\
&=  \sum_{\nu=1}^{k}\nabla\ell_{\nu,h}(\thetab_{\nu-1,h})^{\top}(\thetab_{\nu-1,h}-\thetab_{h}^{\star})-\sum_{\nu=1}^{k}B_{\ell_{\nu,h}}(\thetab_{h}^{\star},\thetab_{\nu-1,h})\\
&\le  \sum_{\nu=1}^{k}\nabla\ell_{\nu,h}(\thetab_{\nu-1,h})^{\top}(\thetab_{\nu-1,h}-\thetab_{h}^{\star})\\
 &\quad +\frac{-\sum_{\nu=1}^{k}\big(\nabla\ell_{\nu,h}(\thetab_{\nu-1,h})^{\top}(\thetab_{\nu-1,h}-\thetab_{h}^{\star})\big)^{2}+16B_{\vPhib}^{2}B_{\thetab}^{2}\log\frac{1}{\delta}}{(e-1)(3+4B_{\vPhib}^{2}B_{\thetab}^{2})}.
\end{split}
\label{eq:loss_2nd}
\end{equation}
By definition of $\tilde{\thetab}_{\nu,h}$ in Algorithm~\ref{alg:ONS-MST},
\[
\tilde{\thetab}_{\nu,h}-\thetab_{h}^{\star}=\thetab_{\nu-1,h}-\thetab_{h}^{\star}-\eta\widehat{\Hb}_{\nu,h}^{-1}\nabla\ell_{\nu,h}(\thetab_{\nu-1,h}),
\]
which implies 
\[
\begin{split}\norm{\tilde{\thetab}_{\nu,h}-\thetab_{h}^{\star}}_{\widehat{\Hb}_{\nu,h}}^{2}= & \norm{\thetab_{\nu-1,h}-\thetab_{h}^{\star}}_{\widehat{\Hb}_{k,h}}^{2}-2\eta\big(\nabla\ell_{\nu,h}(\thetab_{\nu-1,h})\big)^{\top}(\thetab_{\nu-1,h}-\thetab_{h}^{\star})\\
 & +\eta^{2}\norm{\nabla\ell_{\nu,h}(\thetab_{\nu-1,h})}_{\widehat{\Hb}_{\nu,h}^{-1}}^{2}.
\end{split}
\]
By the property of the generalized projection (see e.g., Lemma 8 in
\citealt{hazan2007logarithmic}), 
\[
\norm{\tilde{\thetab}_{\nu,h}-\thetab_{h}^{\star}}_{\widehat{\Hb}_{\nu,h}}^{2}\ge\norm{\thetab_{\nu,h}-\thetab_{h}^{\star}}_{\widehat{\Hb}_{\nu,h}}^{2}.
\]
Thus, 
\[
\begin{split}\norm{\thetab_{\nu,h}-\thetab_{h}^{\star}}_{\widehat{\Hb}_{\nu,h}}^{2}\le & \norm{\thetab_{\nu-1,h}-\thetab_{h}^{\star}}_{\widehat{\Hb}_{\nu,h}}^{2}-2\eta\big(\nabla\ell_{\nu,h}(\thetab_{\nu-1,h})\big)^{\top}(\thetab_{\nu-1,h}-\thetab_{h}^{\star})\\
 & +\eta^{2}\norm{\nabla\ell_{\nu,h}(\thetab_{\nu-1,h})}_{(\widehat{\Hb}_{\nu,h})^{-1}}^{2}.
\end{split}
\]
For $\eta>0$, rearranging the terms, 
\[
\begin{aligned}
&\nabla\ell_{\nu,h}(\thetab_{\nu-1,h})^{\top}(\thetab_{\nu-1,h}-\thetab_{h}^{\star})\\
&\le\frac{\norm{\thetab_{\nu-1,h}-\thetab_{h}^{\star}}_{\widehat{\Hb}_{\nu,h}}^{2}-\norm{\thetab_{\nu,h}-\thetab_{h}^{\star}}_{\widehat{\Hb}_{\nu,h}}^{2}}{2\eta}+\frac{\eta}{2}\norm{\nabla\ell_{\nu,h}(\thetab_{\nu-1,h})}_{\widehat{\Hb}_{\nu,h}^{-1}}^{2}.
\end{aligned}
\]
Plugging in \eqref{eq:loss_2nd}, 
\[
\begin{split}\sum_{\nu=1}^{k}\Big(\ell_{\nu,h}(\thetab_{\nu-1,h})-\ell_{\nu,h}(\thetab_{h}^{\star})\Big)\le & \frac{\eta}{2}\norm{\nabla\ell_{\nu,h}(\thetab_{\nu-1,h})}_{\widehat{\Hb}_{\nu,h}^{-1}}^{2}+\frac{16B_{\vPhib}^{2}B_{\thetab}^{2}\log\frac{1}{\delta}}{(e-1)(3+4B_{\vPhib}^{2}B_{\thetab}^{2})}\\
 & +\sum_{\nu=1}^{k}\frac{\norm{\thetab_{\nu-1,h}-\thetab_{h}^{\star}}_{\widehat{\Hb}_{\nu,h}}^{2}-\norm{\thetab_{\nu,h}-\thetab_{h}^{\star}}_{\widehat{\Hb}_{\nu,h}}^{2}}{2\eta}\\
 & -\frac{1}{(e-1)(3+4B_{\vPhib}^{2}B_{\thetab}^{2})}\sum_{\nu=1}^{k}\big(\nabla\ell_{\nu,h}(\thetab_{\nu-1,h})^{\top}(\thetab_{\nu-1,h}-\thetab_{h}^{\star})\big)^{2}.
\end{split}
\]
Setting $\eta=(e-1)(3+4B_{\vPhib}^{2}B_{\thetab}^{2})/2$, we obtain,
\[
\begin{aligned} & \sum_{\nu=1}^{k}\frac{\norm{\thetab_{\nu-1,h}-\thetab_{h}^{\star}}_{\widehat{\Hb}_{k,h}}^{2}-\norm{\thetab_{\nu,h}-\thetab_{h}^{\star}}_{\widehat{\Hb}_{k,h}}^{2}}{2\eta}-\frac{1}{2\eta}\sum_{\nu=1}^{k}\big(\nabla\ell_{\nu,h}(\thetab_{\nu-1,h})^{\top}(\thetab_{\nu-1,h}-\thetab_{h}^{\star})\big)^{2}\\
 & =\sum_{\nu=1}^{k}\frac{\norm{\thetab_{\nu-1,h}-\thetab_{h}^{\star}}_{\widehat{\Hb}_{\nu-1,h}}^{2}-\norm{\thetab_{\nu,h}-\thetab_{h}^{\star}}_{\widehat{\Hb}_{h}^{(\nu)}}^{2}}{2\eta}.
\end{aligned}
\]
Thus, 
\[
\begin{split} & \sum_{\nu=1}^{k}\Big(\ell_{\nu,h}(\thetab_{\nu-1,h})-\ell_{\nu,h}(\thetab_{h}^{\star})\Big)\\
 & \le\sum_{\nu=1}^{k}\frac{\norm{\thetab_{\nu-1,h}-\thetab_{h}^{\star}}_{\widehat{\Hb}_{\nu-1,h}}^{2}\!\!\!\!-\norm{\thetab_{\nu,h}-\thetab_{h}^{\star}}_{\widehat{\Hb}_{\nu,h}}^{2}}{2\eta}\!+\!\frac{\eta}{2}\sum_{\nu=1}^{k}\norm{\nabla\ell_{\nu,h}(\thetab_{\nu-1,h})}_{\widehat{\Hb}_{\nu,h}^{-1}}^{2}\!+\!\frac{16B_{\vPhib}^{2}B_{\thetab}^{2}\log\frac{1}{\delta}}{2\eta}\\
 & =\frac{\norm{\thetab_{0,h}-\thetab_{h}^{\star}}_{\widehat{\Hb}_{0,h}}^{2}-\norm{\thetab_{k,h}-\thetab_{h}^{\star}}_{\widehat{\Hb}_{k,h}}^{2}}{2\eta}+\frac{\eta}{2}\sum_{\nu=1}^{k}\norm{\nabla\ell_{\nu,h}(\thetab_{\nu-1,h})}_{\widehat{\Hb}_{\nu,h}^{-1}}^{2}+\frac{16B_{\vPhib}^{2}B_{\thetab}^{2}\log\frac{1}{\delta}}{2\eta}\\
 & \le\frac{4B_{\thetab}^{2}\epsilon}{2\eta}+\frac{\eta}{2}\sum_{\nu=1}^{k}\norm{\nabla\ell_{\nu,h}(\thetab_{\nu-1,h})}_{\widehat{\Hb}_{\nu,h}^{-1}}^{2}+\frac{16B_{\vPhib}^{2}B_{\thetab}^{2}\log\frac{1}{\delta}}{2\eta}
\end{split}
\]
Now, to bound the self-normalized gradient, recall that 
\[
\widehat{\Hb}_{\nu,h}:=\sum_{\nu^{\prime}=1}^{\nu}\nabla\ell_{\nu^{\prime},h}(\thetab_{\nu^{\prime}-1,h})\nabla\ell_{\nu^{\prime},h}(\thetab_{\nu^{\prime}-1,h})^{\top}+\epsilon\Ib_{d}.
\]
Because $\|\eb_{s_{\nu,h+1}}-\nabla L_{\nu,h}(\etab_{\nu,h})\|_{1}\le2$,
\begin{equation}
\begin{aligned}\lambda_{\max}\big(\nabla\ell_{\nu,h}(\thetab_{\nu-1,h})\nabla\ell_{\nu,h}(\thetab_{\nu-1,h})^{\top}\big)= & \Big\Vert\nabla\ell_{\nu,h}(\thetab_{\nu-1,h})\Big\Vert_{2}^{2}\\
= & \Big\Vert\big(\eb_{s_{\nu,h+1}}-\nabla L_{\nu,h}(\etab_{\nu,h})\big)^{\top}\vPhib(s_{\nu,h},a_{\nu,h})\Big\Vert_{2}^{2}\\
\le & 4\max_{\xb:\|\xb\|_{1}\le1}\Big\Vert\xb^{\top}\vPhib(s_{\nu,h},a_{\nu,h})\Big\Vert_{2}^{2}.
\end{aligned}
\label{eq:ell_norm_bound_1}
\end{equation}
Writing $\xb=(x_{1},\ldots,x_{|\Scal_{\nu,h}|})$, since $\max_{s^{\prime}\in\Scal_{\nu,h}}\|\vPhib(s_{\nu,h},a_{\nu,h})^{\top}\eb_{s^{\prime}}\|_{2}\le B_{\vPhib}$,
\begin{equation}
\begin{aligned}
&4\max_{\xb:\|\xb\|_{1}\le1}\Big\Vert\xb^{\top}\vPhib(s_{\nu,h},a_{\nu,h})\Big\Vert_{2}^{2}\\
&=  4\max_{\xb:\|\xb\|_{1}\le1}\Big\Vert\sum_{s^{\prime}\in\Scal_{\nu,h}}x_{s^{\prime}}\eb_{s^{\prime}}^{\top}\vPhib(s_{\nu,h},a_{\nu,h})\Big\Vert_{2}^{2}\\
&=  4\max_{\xb:\|\xb\|_{1}\le1}\sum_{s^{\prime}\in\Scal_{\nu,h}}\sum_{s^{\prime\prime}\in\Scal_{\nu,h}}x_{s^{\prime}}x_{s^{\prime\prime}}\eb_{s^{\prime}}^{\top}\vPhib(s_{\nu,h},a_{\nu,h})\vPhib(s_{\nu,h},a_{\nu,h})^{\top}\eb_{s^{\prime\prime}}\\
&\le  4B_{\vPhib}^{2}\max_{\xb:\|\xb\|_{1}\le1}\Big(\sum_{s^{\prime}\in\Scal_{\nu,h}}|x_{s^{\prime}}|\Big)^{2}\\
&\le  4B_{\vPhib}^{2}.
\end{aligned}
\label{eq:ell_norm_bound_2}
\end{equation}
Thus, $\Vert\nabla\ell_{\nu^{\prime},h}(\thetab_{\nu^{\prime}-1,h})\Vert_{2}\le2B_{\vPhib}$.
Since $\lambda_{\min}(\widehat{\Hb}_{0,h})=\epsilon\ge4B_{\vPhib}^{2}$.
Thus, by the elliptical potential bound (Lemma \ref{lem:abbasi_elliptical}),
\[
\sum_{\nu=1}^{k}\norm{\nabla\ell_{\nu,h}(\thetab_{\nu-1,h})}_{\widehat{\Hb}_{\nu,h}^{-1}}^{2}\le\sum_{\nu=1}^{k}\norm{\nabla\ell_{\nu,h}(\thetab_{\nu-1,h})}_{\widehat{\Hb}_{\nu-1,h}^{-1}}^{2}\le2d\log\frac{k+1}{d},
\]
which completes the proof. 
\end{proof}

\subsection{Hessian Approximation Error Bound}
\begin{lemma}[Hessian Approximation Error Bound]
\label{lem:2nd_deriv_approx} For each
$h\in[H]$ and $k\in[K]$, with probability at least $1-\delta,$
\begin{align*}
\Big\Vert \sum_{s^{\prime}\in\Scal_{k,h}} \widehat{V}_{k,h+1}(s^{\prime})
\cdot \eb_{s^{\prime}}^{\top} \widehat{\Vb}_{k,h} \vPhib_{k,h} 
\Big\Vert_{\widehat{\Hb}_{k-1,h}^{-1}} \le& \Big\Vert \sum_{s^{\prime}\in\Scal_{k,h}} \widehat{V}_{k,h+1}(s^{\prime})
\cdot \eb_{s^{\prime}}^{\top} \Vb^{\star}_{k,h} \vPhib_{k,h} 
\Big\Vert_{\widehat{\Hb}_{k-1,h}^{-1}}\\
&+\frac{6\beta_{k-1}H}{(1-e^{-1})\kappa}\max_{s\in\Scal_{k,h}}\norm{\vPhib_{k,h}\eb_{s}}_{\Ab_{k-1,h}^{-1}}^{2}
\end{align*}
\end{lemma}

\begin{proof}
Define $\Db_{k-1,h}:=\nabla^{2}L_{k,h}\big(\vPhib(s_{k,h},a_{k,h})\widehat{\thetab}_{k-1,h}\big)-\nabla^{2}L_{k,h}\big(\vPhib(s_{k,h},a_{k,h})\thetab_{h}^{\star}\big)$.
For simplicity we write $\vPhib_{k,h}:=\vPhib(s_{k,h},a_{k,h})$.
By triangular inequality, 
\[
\begin{split} & \norm{\sum_{\tilde{s}\in\Scal_{k,h}}\widehat{V}_{k,h+1}(\tilde{s})\vPhib_{k,h}^{\top}\nabla^{2}L_{\nu,h}\big(\vPhib_{k,h}\widehat{\thetab}_{k-1,h}\big)\eb_{\tilde{s}}}_{\widehat{\Hb}_{k-1,h}^{-1}}\\
 & \le\norm{\sum_{\tilde{s}\in\Scal_{k,h}}\widehat{V}_{k,h+1}(\tilde{s})\vPhib_{k,h}^{\top}\nabla^{2}L_{\nu,h}\big(\vPhib_{k,h}\thetab_{h}^{\star}\big)\eb_{\tilde{s}}}_{\widehat{\Hb}_{k-1,h}^{-1}}+\norm{\sum_{\tilde{s}\in\Scal_{k,h}}\widehat{V}_{k,h+1}(\tilde{s})\vPhib_{k,h}^{\top}\Db_{k-1,h}\eb_{\tilde{s}}}_{\widehat{\Hb}_{k-1,h}^{-1}}
\end{split}
\]
For the second term, 
\begin{equation}
\norm{\sum_{\tilde{s}\in\Scal_{k,h}}\widehat{V}_{k,h+1}(\tilde{s})\vPhib_{k,h}^{\top}\Db_{k-1,h}\eb_{\tilde{s}}}_{\widehat{\Hb}_{k-1,h}^{-1}}\le\norm{\widehat{\Hb}_{k-1,h}^{-1/2}\vPhib_{k,h}^{\top}}_{1,2}\norm{\sum_{\tilde{s}\in\Scal_{k,h}}\widehat{V}_{k,h+1}(\tilde{s})\Db_{k-1,h}\eb_{\tilde{s}}}_{1},\label{eq:confidence_second_term}
\end{equation}
where $\|\Ab\|_{1,2}=\sup\{\|\Ab\xb\|_{2}:\|\xb\|_{1}\le1\}$ is the
matrix norm. We carefully bound the $\ell_{1}$ norm using the third
derivative of $L_{k,h}$. By the mean value theorem, there exists
$\bar{\etab}$ such that 
\[
\begin{split}
&\norm{\sum_{\tilde{s}\in\Scal_{k,h}}\widehat{V}_{k,h+1}(\tilde{s})\Db_{k-1,h}\eb_{\tilde{s}}}_{1}\\
&=  \norm{\sum_{\tilde{s}\in\Scal_{k,h}}\widehat{V}_{k,h+1}(\tilde{s})\Db_{k-1,h}\big(\nabla(\nabla^{2}L_{k,h}(\bar{\etab})\eb_{\tilde{s}})\big)^{\top}\vPhib_{k,h}(\widehat{\thetab}_{k-1,h}-\thetab_{h}^{\star})}_{1}\\
&=  \sum_{s\in\Scal_{k,h}}\abs{\sum_{\tilde{s}\in\Scal_{k,h}}\widehat{V}_{k,h+1}(\tilde{s})\eb_{s}^{\top}\big(\nabla(\nabla^{2}L_{k,h}(\bar{\etab})\eb_{\tilde{s}})\big)^{\top}\vPhib_{k,h}(\widehat{\thetab}_{k-1,h}-\thetab_{h}^{\star})}\\
&\le  \sum_{s\in\Scal_{k,h}}\norm{\sum_{\tilde{s}\in\Scal_{k,h}}\widehat{V}_{k,h+1}(\tilde{s})\big(\nabla(\nabla^{2}L_{k,h}(\bar{\etab})\eb_{\tilde{s}})\big)\eb_{s}}_{1}\norm{\vPhib_{k,h}(\widehat{\thetab}_{k-1,h}-\thetab_{h}^{\star})}_{\infty},
\end{split}
\]
where the last inequality uses H{ö}lder's inequality. For each $s\in\Scal_{k,h}$,
\[
\begin{split}\norm{\sum_{\tilde{s}\in\Scal_{k,h}}\widehat{V}_{k,h+1}(\tilde{s})\big(\nabla(\nabla^{2}L_{k,h}(\bar{\etab})\eb_{\tilde{s}})\big)\eb_{s}}_{1}= & \norm{\sum_{\tilde{s}\in\Scal_{k,h}}\widehat{V}_{k,h+1}(\tilde{s})D_{s}\nabla^{2}L_{k,h}(\bar{\etab})\eb_{\tilde{s}}}_{1}\\
= & \sum_{s^{\prime}\in\Scal_{k,h}}\abs{\sum_{\tilde{s}\in\Scal_{k,h}}\widehat{V}_{k,h+1}(\tilde{s})\eb_{s^{\prime}}^{\top}D_{s}\nabla^{2}L_{k,h}(\bar{\etab})\eb_{\tilde{s}}}\\
\le & \max_{s\in\Scal_{k,h}}\widehat{V}_{k,h+1}(s)\sum_{s^{\prime},\tilde{s}\in\Scal_{k,h}}\abs{\eb_{\tilde{s}}^{\top}D_{s}\nabla^{2}L_{k,h}(\bar{\etab})\eb_{s^{\prime}}}.
\end{split}
\]
Let us write $\nabla L_{k,h}(\bar{\etab})=(p_{1},\ldots,p_{|\Scal_{k,h}|})^{\top}$.
Then, (23) in \citet{jezequel2021mixability} implies 
\[
\begin{split}\eb_{\tilde{s}}^{\top}D_{s}\nabla^{2}L_{k,h}(\bar{\etab})\eb_{s^{\prime}}= & p_{s}\II(s=\tilde{s}=s^{\prime})\\
 & -p_{s}p_{\tilde{s}}\II(s=s^{\prime})-p_{s}p_{s^{\prime}}\II(s=\tilde{s})-p_{s}p_{s^{\prime}}\II(s^{\prime}=\tilde{s})\\
 & +2p_{s}p_{s^{\prime}}p_{\tilde{s}}.
\end{split}
\]
Thus, 
\[
\begin{split}\sum_{s^{\prime},\tilde{s}\in\Scal_{k,h}}\abs{\eb_{\tilde{s}}^{\top}D_{s}\nabla^{2}L_{k,h}(\bar{\etab})\eb_{s^{\prime}}}= & p_{s}+3p_{s}+2p_{s}=6p_{s}.\end{split}
\]
Gathering the terms, 
\[
\begin{split}\norm{\sum_{\tilde{s}\in\Scal_{k,h}}\widehat{V}_{k,h+1}(\tilde{s})\Db_{k-1,h}\eb_{\tilde{s}}}_{1}\le & 6\max_{s\in\Scal_{k,h}}\widehat{V}_{k,h+1}(s)\sum_{s\in\Scal_{k,h}}p_{s}\norm{\vPhib_{k,h}(\widehat{\thetab}_{k-1,h}-\thetab_{h}^{\star})}_{\infty}\\
= & 6\max_{s\in\Scal_{k,h}}\widehat{V}_{k,h+1}(s)\norm{\vPhib_{k,h}(\widehat{\thetab}_{k-1,h}-\thetab_{h}^{\star})}_{\infty}
\end{split}
\]
By Cauchy-Schwartz inequality, 
\[
\begin{split}\norm{\vPhib_{k,h}(\widehat{\thetab}_{k-1,h}-\thetab_{h}^{\star})}_{\infty}= & \max_{s\in\Scal_{k,h}}\abs{\eb_{s}^{\top}\vPhib_{k,h}(\widehat{\thetab}_{k-1,h}-\thetab_{h}^{\star})}\\
\le & \max_{s\in\Scal_{k,h}}\norm{\widehat{\Hb}_{k-1,h}^{-1/2}\vPhib_{k,h}\eb_{s}}_{2}\norm{\widehat{\thetab}_{k-1,h}-\thetab_{h}^{\star}}_{\widehat{\Hb}_{k-1,h}}\\
\le & \max_{s\in\Scal_{k,h}}\norm{\vPhib_{k,h}\eb_{s}}_{\widehat{\Hb}_{k-1,h}^{-1}}\beta_{k-1},
\end{split}
\]
where the second inequality uses the confidence ellipsoid (Theorem
\ref{thm:self}). Plugging in \eqref{eq:confidence_second_term} 
\begin{align*}
&\norm{\sum_{\tilde{s}\in\Scal_{k,h}}\widehat{V}_{k,h+1}(\tilde{s})\vPhib_{k,h}^{\top}\Db_{k-1,h}\eb_{\tilde{s}}}_{\widehat{\Hb}_{k-1,h}^{-1}} \\
& \le6\max_{s\in\Scal_{k,h}}\widehat{V}_{k,h+1}(s)\norm{\widehat{\Hb}_{k-1,h}^{-1/2}\vPhib_{k,h}^{\top}}_{1,2}\max_{s\in\Scal_{k,h}}\norm{\widehat{\Hb}_{k-1,h}^{-1/2}\vPhib_{k,h}\eb_{s}}_{2}\beta_{k-1}.
\end{align*}
By the bound for the matrix norm (Lemma \ref{lem:matrix_norm_bound}),
\[
\norm{\sum_{\tilde{s}\in\Scal_{k,h}}\widehat{V}_{k,h+1}(\tilde{s})\vPhib_{k,h}^{\top}\Db_{k-1,h}\eb_{\tilde{s}}}_{\widehat{\Hb}_{k-1,h}^{-1}}\le6\max_{s\in\Scal_{k,h}}\widehat{V}_{k,h+1}(s)\max_{s\in\Scal_{k,h}}\norm{\vPhib_{k,h}\eb_{s}}_{\widehat{\Hb}_{k-1,h}^{-1}}^{2}\beta_{k-1}.
\]
By the Hessian lower bound (Lemma \ref{lem:Hessian_bound}), and the
definition of the Fisher information lower bound, 
\[
\norm{\vPhib_{k,h}\eb_{s}}_{\widehat{\Hb}_{k-1,h}^{-1}}^{2}\le\frac{1}{(1-e^{-1})^{2}}\norm{\vPhib_{k,h}\eb_{s}}_{(\Hb_{k-1,h}^{\star})^{-1}}^{2}\le\frac{1}{(1-e^{-1})^{2}\kappa}\norm{\vPhib_{k,h}\eb_{s}}_{\Ab_{k-1,h}^{-1}}^{2}
\]
Thus,
\[
\norm{\sum_{\tilde{s}\in\Scal_{k,h}}\widehat{V}_{k,h+1}(\tilde{s})\vPhib_{k,h}^{\top}\Db_{k-1,h}\eb_{\tilde{s}}}_{\widehat{\Hb}_{k-1,h}^{-1}}\le\frac{6\beta_{k-1}\max_{s\in\Scal_{k,h}}\widehat{V}_{k,h+1}(s)}{(1-e^{-1})\kappa^{1}}\max_{s\in\Scal_{k,h}}\norm{\vPhib_{k,h}\eb_{s}}_{\Ab_{k-1,h}^{-1}}^{2}.
\]
Since $\widehat{V}_{k,h+1}(s)\le H$ for all $s\in\Scal_{k,h}$, the proof is completed.
\end{proof}

\subsection{Online regret to self-normalized bound}

\label{sec:proof_self} To prove the ellipsoid confidence for our
proposed estimator $\widehat{\thetab}_{k,h}$, we first prove the
following result which relates the online regret to the self-normalized
bound. \begin{lemma} \label{lem:regret_to_self} (Online regret to
self-normalized bound) Define $C_{\vPhib,\thetab}:=(e-1)(6+8B_{\vPhib}B_{\thetab}+2B_{\vPhib}^{2}B_{\thetab}^{2})$.
For each $h\in[H]$, with probability at least $1-\delta$, 
\[
\begin{split}\sum_{\nu=1}^{k}\ell_{\nu,h}(\thetab_{\nu-1,h})-\ell_{\nu,h}(\thetab_{h}^{\star})\ge & \frac{1}{2C_{\vPhib,\thetab}}\sum_{\nu=1}^{k}\Big(\nabla\ell_{\nu,h}(\thetab_{\nu-1,h})^{\top}(\thetab_{\nu-1,h}-\thetab_{h}^{\star})\Big)^{2}\\
 & -\frac{16B_{\vPhib}^{2}B_{\thetab}^{2}}{C_{\vPhib,\thetab}}\log\frac{d}{\delta}-2(e-1)(6+8B_{\vPhib}B_{\thetab})\log\frac{1}{\delta},
\end{split}
\]
for any $k\in[K]$ and $h\in[H]$. \end{lemma} 
\begin{proof}
By the definition of Bregman divergence, 
\[
\ell_{\nu,h}(\thetab_{\nu-1,h})=\ell_{\nu,h}(\thetab_{h}^{\star})+\nabla\ell_{\nu,h}(\thetab_{h}^{\star})^{\top}(\thetab_{\nu-1,h}-\thetab_{h}^{\star})+B_{\ell_{\nu,h}}(\thetab_{\nu-1,h},\thetab_{h}^{\star})
\]
Let us write the residual 
\[
\xi_{\nu,h}:=\big\{\eb_{s_{\nu,h+1}}-\EE\big[\eb_{s_{\nu,h+1}}\big|s_{\nu,h},a_{\nu,h}\big]\big\}^{\top}\vPhib(s_{\nu,h},a_{\nu,h})(\thetab_{\nu-1,h}-\thetab_{h}^{\star})
\]
Then rearranging the terms, 
\[
\begin{aligned} & \ell_{\nu,h}(\thetab_{\nu,h})-\ell_{\nu,h}(\thetab_{h}^{\star})=-\xi_{\nu,h}+B_{\ell_{\nu,h}}(\thetab_{\nu-1,h},\thetab_{h}^{\star})\end{aligned}
\]
By Freedman's concentration inequality (Lemma \ref{lem:freedman}),
for any $v\in[0,\frac{1}{2B_{\vPhib}B_{\thetab}}]$, with probability
at least $1-\delta$ 
\[
\sum_{\nu=1}^{k}\xi_{\nu,h}\le(e-2)v\sum_{\nu=1}^{k}\EE\Big[\xi_{\nu,h}^{2}\bigg|s_{\nu,h},a_{\nu,h}\Big]+\frac{1}{v}\log\frac{1}{\delta},
\]
for all $k$ and for each $h\in[H]$. Because $\|\thetab_{\nu-1,h}-\thetab_{h}^{\star}\|_{\nabla^{2}\ell_{\nu,h}(\thetab_{h}^{\star})}^{2}=\EE[\xi_{\nu,h}^{2}|s_{\nu,h},a_{\nu,h}],$
summing up, 
\begin{align*}
&\sum_{\nu=1}^{k}\ell_{\nu,h}(\thetab_{\nu-1,h})-\ell_{\nu,h}(\thetab_{h}^{\star})\\
&\ge-(e-2)v\sum_{\nu=1}^{k}\|\thetab_{\nu-1,h}-\thetab_{h}^{\star}\|_{\nabla^{2}\ell_{\nu,h}(\thetab_{h}^{\star})}^{2}\\
&\quad+\sum_{\nu=1}^{k}B_{\ell_{\nu,h}}(\thetab_{\nu-1,h},\thetab_{h}^{\star})-\frac{1}{v}\log\frac{1}{\delta}.
\end{align*}
By the second order of the Taylor expansion, 
\[
B_{\ell_{\nu,h}}(\thetab_{\nu-1,h},\thetab_{h}^{\star})=\frac{1}{2}\int_{0}^{1}\psi_{\nu,h}^{\prime\prime}(t)tdt,
\]
where $\ensuremath{\psi_{\nu,h}(t):=L_{\nu,h}\big(\vPhib_{\nu,h}\thetab_{\nu-1,h}+t\vPhib_{\nu,h}(\thetab_{h}^{\star}-\thetab_{\nu-1,h})\big).}$By
Proposition \ref{prop:multi_lower_bound}, 
\[
\frac{1}{2}\int_{0}^{1}\psi_{\nu,h}^{\prime\prime}(t)tdt\ge\frac{\psi_{\nu,h}^{\prime\prime}(1)}{6+8B_{\vPhib}B_{\thetab}}=\frac{\|\thetab_{\nu-1,h}-\thetab_{h}^{\star}\|_{\nabla^{2}\ell_{\nu,h}(\thetab_{h}^{\star})}^{2}}{6+8B_{\vPhib}B_{\thetab}}
\]
Thus, 
\begin{equation}
\sum_{\nu=1}^{k}\ell_{\nu,h}(\thetab_{\nu-1,h})-\ell_{\nu,h}(\thetab_{h}^{\star})\ge\big(1-(e-2)(6+8B_{\vPhib}B_{\thetab})v\big)\sum_{\nu=1}^{k}B_{\ell_{\nu,h}}(\thetab_{\nu-1,h},\thetab_{h}^{\star})-\frac{1}{v}\log\frac{1}{\delta}.\label{eq:likelihood_lower_breg}
\end{equation}
By Lemma \ref{lem:Information_matrix_bound}, 
\begin{align*}
 & \sum_{\nu=1}^{k}\big(\nabla\ell_{\nu,h}(\thetab_{\nu-1,h})^{\top}(\thetab_{\nu-1,h}-\thetab_{h}^{\star})\big)^{2}\\
 & \le2(e-1)B_{\vPhib}^{2}B_{\thetab}^{2}\sum_{\nu=1}^{k}B_{\ell_{\nu,h}}(\thetab_{\nu-1,h},\thetab_{h}^{\star})\\
 & +(e-1)\sum_{\nu=1}^{k}\norm{\thetab_{\nu-1,h}-\thetab_{h}^{\star}}_{\nabla^{2}\ell_{\nu,h}(\thetab_{h}^{\star})}^{2}+16B_{\vPhib}^{2}B_{\thetab}^{2}\log\frac{d}{\delta}\\
 & \le(e-1)(6+8B_{\vPhib}B_{\thetab}+2B_{\vPhib}^{2}B_{\thetab}^{2})\sum_{\nu=1}^{k}B_{\ell_{\nu,h}}(\thetab_{\nu-1,h},\thetab_{h}^{\star})+16B_{\vPhib}^{2}B_{\thetab}^{2}\log\frac{d}{\delta}
\end{align*}
Rearranging the terms, 
\begin{small}
\[
\sum_{\nu=1}^{k}B_{\ell_{\nu,h}}(\thetab_{\nu-1,h},\thetab_{h}^{\star})\ge\frac{\sum_{\nu=1}^{k}\big(\nabla\ell_{\nu,h}(\thetab_{\nu-1,h})^{\top}(\thetab_{\nu-1,h}-\thetab_{h}^{\star})\big)^{2}}{(e-1)(6+8B_{\vPhib}B_{\thetab}+2B_{\vPhib}^{2}B_{\thetab}^{2})}-\frac{16B_{\vPhib}^{2}B_{\thetab}^{2}\log\frac{d}{\delta}}{(e-1)(6+8B_{\vPhib}B_{\thetab}+2B_{\vPhib}^{2}B_{\thetab}^{2})}.
\]
\end{small}
Plugging in \eqref{eq:likelihood_lower_breg}, 
\begin{align*}
&\sum_{\nu=1}^{k}\ell_{\nu,h}(\thetab_{\nu-1,h})-\ell_{\nu,h}(\thetab_{h}^{\star})\\
&\ge  \frac{\big(1-(e-2)(6+8B_{\vPhib}B_{\thetab})v\big)}{(e-1)(6+8B_{\vPhib}B_{\thetab}+2B_{\vPhib}^{2}B_{\thetab}^{2})}\sum_{\nu=1}^{k}\big(\nabla\ell_{\nu,h}(\thetab_{\nu-1,h})^{\top}(\thetab_{\nu-1,h}-\thetab_{h}^{\star})\big)^{2}\\
 &\quad -\frac{16B_{\vPhib}^{2}B_{\thetab}^{2}\log\frac{d}{\delta}}{(e-1)(6+8B_{\vPhib}B_{\thetab}+2B_{\vPhib}^{2}B_{\thetab}^{2})}-\frac{1}{v}\log\frac{1}{\delta}
\end{align*}
Setting $v=(1/2)(e-2)^{-1}(6+8B_{\vPhib}B_{\thetab})^{-1}$ , 
\begin{align*}
\sum_{\nu=1}^{k}\ell_{\nu,h}(\thetab_{\nu-1,h})-\ell_{\nu,h}(\thetab_{h}^{\star})\ge & \frac{\sum_{\nu=1}^{k}\big(\nabla\ell_{\nu,h}(\thetab_{\nu-1,h})^{\top}(\thetab_{\nu-1,h}-\thetab_{h}^{\star})\big)^{2}}{2(e-1)(6+8B_{\vPhib}B_{\thetab}+2B_{\vPhib}^{2}B_{\thetab}^{2})}\\
 & -\frac{16B_{\vPhib}^{2}B_{\thetab}^{2}\log\frac{d}{\delta}}{(e-1)(6+8B_{\vPhib}B_{\thetab}+2B_{\vPhib}^{2}B_{\thetab}^{2})}-2(e-1)(6+8B_{\vPhib}B_{\thetab})\log\frac{1}{\delta}
\end{align*}
which completes the proof. 
\end{proof}

\subsection{Information Matrix Bound}
The following key concentration result for our analysis in an immediate
consequence of~\eqref{eq:diff}. 
\begin{lemma}[Information matrix bound] \label{lem:Information_matrix_bound} For each step $h\in[H]$
and episode $\nu\in[K]$, let $\ell_{\nu,h}$ denote the negative
log-likelihood defined in \eqref{eq:log_likelihood}. 
Then with probability
at least $1-\delta$, 
\begin{align*}
&\sum_{\nu=1}^{k}\big(\nabla\ell_{\nu,h}(\thetab_{\nu-1,h})^{\top}(\thetab_{\nu-1,h}-\thetab_{h}^{\star})\big)^{2}\\
&\le  2(e-1)B_{\vPhib}^{2}B_{\thetab}^{2}\sum_{\nu=1}^{k}B_{\ell_{\nu,h}}(\thetab_{h}^{\star},\thetab_{\nu-1,h})\\
&\quad+(e-1)\sum_{\nu=1}^{k}\|\thetab_{\nu-1,h}-\thetab_{h}^{\star}\|_{\nabla^{2}\ell_{\nu,h}(\thetab_{h}^{\star})}^{2}+16B_{\vPhib}^{2}B_{\thetab}^{2}\log\frac{1}{\delta}.
\end{align*}
\end{lemma} Lemma~\ref{lem:Information_matrix_bound} furnishes
a \emph{lower bound} on $\sum_{\nu=1}^{k}\|\thetab_{\nu-1,h}-\thetab_{h}^{\star}\|_{\nabla^{2}\ell_{\nu,h}(\thetab_{h}^{\star})}^{2}$
that is independent of $\kappa$. The proof of Lemma \ref{lem:online_regret}
also involves creatively combining standard argument for the online
step update with a new lower bound for self-concordant function~(Lemma
\ref{lem:con_hessian_upper}). 
\begin{proof}
For $\nu\in[k]$ and $h\in[H]$, 
\begin{align*}
&|\nabla\ell_{\nu,h}(\thetab_{\nu-1,h})^{\top}(\thetab_{\nu-1,h}-\thetab_{h}^{\star})|\\
&\le  \big|\eb_{s_{\nu,h+1}}^{\top}\vPhib_{\nu,h}(\thetab_{\nu-1,h}-\thetab_{h}^{\star})\big|+\big|\nabla L_{\nu,h}(\vPhib_{\nu,h}\thetab_{\nu-1,h})^{\top}\vPhib_{\nu,h}(\thetab_{\nu-1,h}-\thetab_{h}^{\star})\big|\\
&\le  \big(\|\eb_{s_{\nu,h+1}}^{\top}\vPhib_{\nu,h}\|_{2}+\|\nabla L_{\nu,h}(\vPhib_{\nu,h}\thetab_{\nu-1,h})^{\top}\vPhib_{\nu,h}\|_{2}\big)\|\thetab_{\nu-1,h}-\thetab_{h}^{\star}\|_{2}\\
&\le  4B_{\vPhib}B_{\thetab}.
\end{align*}
By the concentration inequality (Lemma~\ref{lem:chernoff}), for
any $k\ge1$, with probability at least $1-\delta$, 
\begin{align*}
&\sum_{\nu=1}^{k}\big(\nabla\ell_{\nu,h}(\thetab_{\nu-1,h})^{\top}(\thetab_{\nu-1,h}-\thetab_{h}^{\star})\big)^{2}\\
&\le(e-1)\sum_{\nu=1}^{k}\EE\big[\big(\nabla\ell_{\nu,h}(\thetab_{\nu-1,h})^{\top}(\thetab_{\nu-1,h}-\thetab_{h}^{\star})\big)^{2}\big|\Fcal_{\nu,h}\big]+16B_{\vPhib}^{2}B_{\thetab}^{2}\log\frac{1}{\delta}.
\end{align*}
For each $\nu\in[k]$ and $h\in[H]$, we have $\EE[\nabla\ell_{\nu,h}(\thetab_{h}^{\star})|\Fcal_{\nu,h}]=0$
and $\EE[\nabla\ell_{\nu,h}(\thetab_{h}^{\star})\nabla\ell_{\nu,h}(\thetab_{h}^{\star})^{\top}|\Fcal_{\nu,h}]=\nabla^{2}\ell_{\nu,h}(\thetab_{h}^{\star})$.
Thus, 
\begin{align*}
 & \EE\big[\big(\nabla\ell_{\nu,h}(\thetab_{\nu-1,h})^{\top}(\thetab_{\nu-1,h}-\thetab_{h}^{\star})\big)^{2}\big|\Fcal_{\nu,h}\big]\\
 & =\EE\big[\Big(\big(\nabla\ell_{\nu,h}(\thetab_{\nu-1,h})-\nabla\ell_{\nu,h}(\thetab_{h}^{\star})+\nabla\ell_{\nu,h}(\thetab_{h}^{\star})\big)^{\top}(\thetab_{\nu-1,h}-\thetab_{h}^{\star})\Big)^{2}\big|\Fcal_{\nu,h}\big]\\
 & =\EE\big[\Big(\big(\vPhib_{\nu,h}^{\top}\big(\nabla L_{\nu,h}(\vPhib_{\nu,h}\thetab_{\nu-1,h})-\nabla L_{\nu,h}(\vPhib_{\nu,h}\thetab_{h}^{\star})\big)+\nabla\ell_{\nu,h}(\thetab_{h}^{\star})\big)^{\top}(\thetab_{\nu-1,h}-\thetab_{h}^{\star})\Big)^{2}\big|\Fcal_{\nu,h}\big]\\
 & =\Big(\big(\nabla L_{\nu,h}(\vPhib_{\nu,h}\thetab_{\nu-1,h})-\nabla L_{\nu,h}(\vPhib_{\nu,h}\thetab_{h}^{\star})\big)^{\top}\vPhib_{\nu,h}(\thetab_{\nu-1,h}-\thetab_{h}^{\star})\Big)^{2}+\|\thetab_{\nu-1,h}-\thetab_{h}^{\star}\|_{\nabla^{2}\ell_{\nu,h}(\thetab_{h}^{\star})}^{2}.
\end{align*}
By the property of the mean-variance property, $\nabla L_{\nu,h}(\vPhib_{\nu,h}\thetab_{\nu-1,h})=\EE_{\vPhib_{\nu,h}\thetab_{\nu-1,h}}[\eb_{s}]$.
Thus, 
\begin{align*}
 & \big(\nabla L_{\nu,h}(\vPhib_{\nu,h}\thetab_{\nu-1,h})-\nabla L_{\nu,h}(\vPhib_{\nu,h}\thetab_{h}^{\star})\big)^{\top}\vPhib_{\nu,h}(\thetab_{\nu-1,h}-\thetab_{h}^{\star})\\
 & =\EE_{\vPhib_{\nu,h}\thetab_{\nu-1,h}}\big[\eb_{s}^{\top}\vPhib_{\nu,h}(\thetab_{\nu-1,h}-\thetab_{h}^{\star})\big]-\EE_{\vPhib_{\nu,h}\thetab_{h}^{\star}}\big[\eb_{s}^{\top}\vPhib_{\nu,h}(\thetab_{\nu-1,h}-\thetab_{h}^{\star})\big]\\
 & \le2B_{\vPhib}B_{\thetab}\cdot TV(\PP_{\vPhib_{\nu,h}\thetab_{\nu-1,h}},\PP_{\vPhib_{\nu,h}\thetab_{h}^{\star}}),
\end{align*}
By Pinsker's inequality, 
\begin{align*}
 & \Big|\big(\nabla L_{\nu,h}(\vPhib_{\nu,h}\thetab_{\nu-1,h})-\nabla L_{\nu,h}(\vPhib_{\nu,h}\thetab_{h}^{\star})\big)^{\top}\vPhib_{\nu,h}(\thetab_{\nu-1,h}-\thetab_{h}^{\star})\Big|\\
 & \le B_{\vPhib}B_{\thetab}\sqrt{2KL(\PP_{\vPhib_{\nu,h}\thetab_{\nu-1,h}},\PP_{\vPhib_{\nu,h}\thetab_{h}^{\star}})}\\
 & =B_{\vPhib}B_{\thetab}\sqrt{2B_{\ell_{\nu,h}}(\thetab_{h}^{\star},\thetab_{\nu-1,h})},
\end{align*}
where the last inequality holds by the relationship between the Bregman
divergence and KL-divergence for multinomial distribution. Thus, 
\[
\Big(\big(\nabla L_{\nu,h}(\vPhib_{\nu,h}\thetab_{\nu-1,h})-\nabla L_{\nu,h}(\vPhib_{\nu,h}\thetab_{h}^{\star})\big)^{\top}\vPhib_{\nu,h}(\thetab_{\nu-1,h}-\thetab_{h}^{\star})\Big)^{2}\le2B_{\vPhib}^{2}B_{\thetab}^{2}B_{\ell_{\nu,h}}(\thetab_{h}^{\star},\thetab_{\nu-1,h}),
\]
which completes the proof. 
\end{proof}

\subsection{Concentration Inequalities}

\begin{lemma} \label{lem:freedman} (Freedman's inequality) \citep[Lemma 3]{lee2024improved}
Let $\{D_{s}:s\ge1\}$ denote a martingale difference sequence adapted
to $\Fcal_{s}$ with $\max|D_{s}|\le R$ almost surely. Then for any
$\delta\in(0,1)$ and any $\eta\in[0,1/R]$, with probability at least
$1-\delta$, 
\[
\sum_{s=1}^{t}D_{s}\le(e-2)\eta\sum_{s=1}^{t}\EE\big[D_{s}^{2}\big|\Fcal_{s-1}\big]+\frac{1}{\eta}\log\frac{1}{\delta},
\]
for all $t\ge1$. \end{lemma}

\begin{lemma} \label{lem:chernoff} (Chernoff Bound) Let $\{X_{s}:s\ge1\}$
denote a sequence of nonnegative random variables adapted to $\Fcal_{s}$
with $\max X_{s}\le R$ almost surely. Then for any $\delta\in(0,1)$,
with probability at least $1-\delta$, 
\[
\sum_{s=1}^{t}X_{s}\le(e-1)\sum_{s=1}^{t}\EE\big[X_{s}\big|\Fcal_{s-1}\big]+R\log\frac{1}{\delta},
\]
and 
\[
\sum_{s=1}^{t}X_{s}\ge(1-e^{-1})\sum_{s=1}^{t}\EE\big[X_{s}\big|\Fcal_{s-1}\big]-R\log\frac{1}{\delta},
\]
for all $t\ge1$. 
\end{lemma}
\begin{proof}
\medspace{} For any $x>0$ and $y>0$, by Markov inequality, 
\[
\PP\Big(\sum_{s=1}^{t}X_{s}>(e-1)\sum_{s=1}^{t}\EE\big[X_{s}|\Fcal_{s-1}\big]+R\log\frac{1}{\delta}\Big)\le\delta\EE\exp\Big(\sum_{s=1}^{t}\frac{X_{s}}{R}-(e-1)\sum_{s=1}^{t}\frac{\EE\big[X_{s}|\Fcal_{s-1}\big]}{R}\Big)
\]
For $x\in[0,1]$, we have $1+x\le e^{x}\le1+(e-1)x$. For each $s\in[t]$,
\[
\begin{split}\EE\Big[\exp\Big(\frac{X_{s}}{R}\Big)\Big|\Fcal_{s-1}\Big]\le & 1+(e-1)\frac{\EE\Big[X_{s}\Big|\Fcal_{s-1}\Big]}{R}\\
\le & \exp\left((e-1)\frac{\EE\Big[X_{s}\Big|\Fcal_{s-1}\Big]}{R}\right).
\end{split}
\]
Thus, the sequence 
\[
\left\{ \exp\left(\sum_{s=1}^{t}\frac{X_{s}}{R}-(e-1)\sum_{s=1}^{t}\frac{\EE\big[X_{s}|\Fcal_{s-1}\big]}{R}\right):t\ge1\right\} ,
\]
is a super-martingale. Following the stopping time argument (see e.g.,
proof of Theorem 1 in \citealt{abbasi2011improved}) we have 
\[
\PP\left(\bigcup_{t\ge1}\left\{ \sum_{s=1}^{t}X_{s}>(e-1)\sum_{s=1}^{t}\EE\big[X_{s}|\Fcal_{s-1}\big]+R\log\frac{1}{\delta}\right\} \right)\le\delta.
\]
\end{proof}

\begin{lemma}[Matrix Chernoff Bound \citep{tropp2012user}] \label{lem:chernoff_matrix-1}
Let $\{\Xb_{s}:s\geq1\}$ be a sequence of nonnegative semi-definite
random matrices in $\ensuremath{\mathbb{R}^{d\times d}},$ adapted
to $\ensuremath{\mathcal{F}_{s}},$ with $\lambda_{\max}(\Xb_{s})\leq R$
almost surely. Then, for any $\delta\in(0,1)$, the following inequalities
hold with probability at least $1-\delta$:
\[
\sum_{s=1}^{t}\Xb_{s}\preceq(e-1)\sum_{s=1}^{t}\mathbb{E}\big[\Xb_{s}\big|\mathcal{F}_{s-1}\big]+R\log\frac{d}{\delta}
\]
and 
\[
\sum_{s=1}^{t}\Xb_{s}\succeq(1-e^{-1})\sum_{s=1}^{t}\mathbb{E}\big[\Xb_{s}\big|\mathcal{F}_{s-1}\big]-R\log\frac{d}{\delta}
\]
for all $t\geq1$. \end{lemma}

\begin{lemma}[Theorem 1 in \citealt{faury2020improved}] \label{lem:faury_self}
Let $\{\Fcal_{t}\}_{t=1}^{\infty}$ be a filtration. Let $\{x_{t}\}_{t=1}^{\infty}$
be a stochastic process in $B_{2}(d)$ such that $x_{t}$ is $\Fcal_{t}$
measurable. Let $\{\epsilon_{t+1}\}_{t=2}^{\infty}$ be a martingale
difference sequence such that $\epsilon_{t+1}$ is $\Fcal_{t+1}$
measurable. Furthermore, assume that conditionally on $\Fcal_{t}$
we have $|\epsilon_{t+1}|\le1$ almost surely, and note $\sigma_{t}^{2}:=\EE[\epsilon_{t+1}^{2}|\Fcal_{t}]$.
Let $\lambda>0$ and for any $t\geq1$ define: 
\[
\Hb_{t}:=\sum_{s=1}^{t-1}\sigma_{s}^{2}x_{s}x_{s}^{\top}+\lambda I_{d},\qquad S_{t}:=\sum_{s=1}^{t-1}\epsilon_{s}x_{s}.
\]

Then for any $\delta\in(0,1]$: 
\[
\PP\!\left(\exists t\geq1,\ \|S_{t}\|_{\Hb_{t}^{-1}}\geq\frac{\sqrt{\lambda}}{2}+\frac{2}{\sqrt{\lambda}}\log\left(\frac{\det(\Hb_{t})^{1/2}\lambda^{-d/2}}{\delta}\right)+\frac{2}{\sqrt{\lambda}}d\log(2)\right)\leq\delta.
\]
\end{lemma}

\begin{lemma}[Anti-concentration inequality for nonnegative bounded
random variables] \label{lem:anti} Let $\{X_{i}:i\ge1\}$ denote
the sequence of nonnegative random variables adapted to filtration
$\{\Fcal_{i}:i\ge1\}$ such that $0\le X_{i}\le b$. Let $\mu_{k}:=\EE[\sum_{i=1}^{k}X_{i}]$.
Then for any $\delta\in(0,1)$, 
\[
\PP\bigg(S_{k}\ge\mu_{k}-\frac{\delta}{1-\delta}\sqrt{2kb^2\log\frac{1}{\delta}}\bigg)\ge\delta.
\]
\end{lemma} 
\begin{proof}
Let us write $S_{k}:=\sum_{i=1}^{k}X_{i}$ and $\mu_{k}:=\EE[S_{k}]$.
Then for any $k\ge1$, for $\theta\in\RR$ to be determined later, 
\begin{align*}
\mu_{k}= & \EE[S_{k}\II(S_{k}\le\theta)]+\EE[S_{k}\II(S_{k}>\theta)]\\
\le & \theta\PP(S_{k}\le\theta)+\EE[S_{k}\II(S_{k}>\theta)].
\end{align*}
Rearranging the terms, 
\[
\mu_{k}-\theta\le-\theta\PP(S_{k}>\theta)+\EE[S_{k}\II(S_{k}>\theta)].
\]
By the definition of the conditional expectation, 
\begin{equation}
\mu_{k}-\theta\le\PP(S_{k}>\theta)\big(\EE[S_{k}|S_{k}>\theta]-\theta\big)\label{eq:anti_con}
\end{equation}
For any $\lambda>0$, by the entropic variational principle (Lemma
~\ref{lem:variation_principle}) 
\[
\EE\big[S_{k}\big|S_{k}>\theta\big]\le\frac{\log\EE[e^{\lambda S_{k}}]+KL(\PP_{S_{k}|S_{k}>\theta},\PP_{S_{k}})}{\lambda}.
\]
The Kullback-Libeler divergence,
\[
KL(\PP_{S_{k}|S_{k}>\theta},\PP_{S_{k}})=\EE\big[\log\frac{\II(S_{k}>\theta)}{\PP(S_{k}>\theta)}\big|S_{k}\ge\mu_{k}\big]=\log\frac{1}{\PP(S_{k}>\theta)}.
\]
It follows that
\[
\EE\big[S_{k}\big|S_{k}>\theta\big]\le\frac{\log\EE[e^{\lambda S_{k}}]+\log\frac{1}{\PP(S_{k}>\theta)}}{\lambda}.
\]
Note that $\log\EE[e^{\lambda S_{k}}]=\log\EE[e^{\lambda S_{k}-\lambda\mu_{k}}]+\lambda\mu_{k}$.
For each $i=1,\ldots,k$ let $Y_{i}$ denote a random variable such
that the distribution of $Y_{i}|\Fcal_{i-1}$ is equal to $X_{i}|\Fcal_{i-1}$
and $X_{i},Y_{i}$ are conditionally independent given $\Fcal_{i-1}$.
Let us write the conditional expectation $\EE_{i}:=\EE[\cdot|\Fcal_{i-1}]$.
By the tower property of the conditional expectation, $\mu_{k}=\EE\big[\sum_{i=1}^{k}\EE_{i}[X_{i}]\big]=\EE[\sum_{i=1}^{k}\EE_{i}[Y_{i}]]=\EE[\sum_{i=1}^{k}Y_{i}]$
and by Jensen's inequality,
\[
\EE[e^{\lambda S_{k}-\lambda\mu_{k}}]\le\EE[e^{\lambda\sum_{i=1}^{k}(X_{i}-Y_{i})}].
\]
Define $M_{i}(\lambda)=\exp\big(\lambda\sum_{j=1}^{i}(X_{j}-Y_{j})\big)$
and $M_{0}(\lambda)=1$. By the tower property of the conditional
expectation, because $Y_{1},\ldots,Y_{k-1}$ are independent of $X_{k}$
and $Y_{k}$ 
\begin{align*}
\EE[M_{k}(\lambda)]= & \EE\bigg[\EE[M_{k}(\lambda)|\Fcal_{k-1}]\bigg]\\
= & \EE\bigg[\EE[M_{k}(\lambda)|Y_{1},\ldots,Y_{k-1},\Fcal_{k-1}]\bigg]\\
= & \EE\bigg[M_{k-1}(\lambda)\EE[\exp(\lambda X_{k}-\lambda Y_{k})|Y_{1},\ldots,Y_{k-1},\Fcal_{k-1}]\bigg]\\
= & \EE\bigg[M_{k-1}(\lambda)\EE[\exp(\lambda X_{k}-\lambda Y_{k})|\Fcal_{k-1}]\bigg].
\end{align*}
Since $X_{k}$ and $Y_{k}$ are conditionally identically distributed
and independent given $\Fcal_{k-1}$, for a Radamacher random variable
$\epsilon_{k}$, 
\begin{align*}
\EE[\exp(\lambda X_{k}-\lambda Y_{k})|\Fcal_{k-1}]= & \EE\big[\exp\big(\lambda\epsilon_{k}(X_{k}-Y_{k})\big)\big|\Fcal_{k-1}\big]\\
\le & \exp(\frac{\lambda b^{2}}{2}),
\end{align*}
where the last inequality uses Hoeffing's inequality with the fact
that $\epsilon_{k}(X_{k}-Y_{k})\in[-b,b]$. Thus,
\begin{align*}
\EE[M_{k}(\lambda)]\le & \EE\big[M_{k-1}(\lambda)\big]\exp(\frac{\lambda^{2}b^{2}}{2})\\
\vdots & \vdots\\
\le & \EE\big[M_{0}(\lambda)\big]\exp(\frac{\lambda^{2}kb^{2}}{2}).
\end{align*}
It follows that 
\[
\EE\big[S_{k}\big|S_{k}>\theta\big]\le\mu_{k}+\frac{\lambda kb^{2}}{2}+\frac{1}{\lambda}\log\frac{1}{\PP(S_{k}>\theta)}.
\]
Setting $\lambda=\sqrt{\frac{2\log\frac{1}{\PP(S_{k}>\theta)}}{kb^{2}}}$,
\[
\EE\big[S_{k}\big|S_{k}>\theta\big]\le\mu_{k}+\sqrt{2kb^{2}\log\frac{1}{\PP(S_{k}>\theta)}}
\]
 Plugging in \eqref{eq:anti_con}, 
\[
\mu_{k}-\theta\le\PP(S_{k}>\theta)\big(\mu_{k}+\sqrt{2kb^{2}\log\frac{1}{\PP(S_{k}>\theta)}}-\theta\big)
\]
Rearranging the terms,
\[
\theta\ge\mu_{k}-\frac{\PP(S_{k}>\theta)}{1-\PP(S_{k}>\theta)}\sqrt{2kb^{2}\log\frac{1}{\PP(S_{k}>\theta)}}.
\]
Given $\delta\in(0,1)$, setting $\theta=\mu_{k}-\frac{\delta}{(1-\delta)}\sqrt{2kb^{2}\log\frac{1}{\delta}}$ yields
\[
\frac{\delta}{(1-\delta)}\sqrt{2kb^{2}\log\frac{1}{\delta}}\le\frac{\PP(S_{k}>\theta)}{1-\PP(S_{k}>\theta)}\sqrt{2kb^{2}\log\frac{1}{\PP(S_{k}>\theta)}}.
\]
Because $p\mapsto\frac{p}{1-p}\sqrt{2kb^{2}\log\frac{1}{p}}$ is nondecreasing
in $p\in(0,1)$, we conclude $\delta\le\PP(S_{k}>\theta)$. 
\end{proof}

\subsection{Matrix Potential Lemma}
\begin{lemma}[Matrix Elliptical Potential Lemma]
\label{lem:matrix_potent}  
Let $\{\Db_{\nu}\in\RR^{d\times d}:\nu\ge0\}$ denote a sequence of positive
definite symmetric matrices with $\Tr(\Db_{\nu})\le b$ for all $\nu\ge1$
for some $b>0$. Let $\Mb_{n}=\sum_{\nu=1}^{n}\Db_{\nu}+\Db_{0}$
such that $\lambda_{\min}(\Db_{0})\ge\max\{1,\lambda_{\max}(\Db_{\nu})\}$.
Then, for any $n\ge1$
\[
\sum_{\nu=1}^{n}\Tr(\Mb_{n-1}^{-1/2}\Db_{n}\Mb_{n-1}^{-1/2})\le2d\log(n+1).
\]
\end{lemma}

\begin{proof}
Let us fix $n\ge1$ throughout the proof. Then
\begin{equation*}\begin{aligned}
\det(\Mb_{n})= & \det(\Mb_{n-1}+\Db_{n})\\
= & \det(\Mb_{n-1})\det(\Ib_{d}+\Mb_{n-1}^{-1/2}\Db_{n}\Mb_{n-1}^{-1/2})\\
\vdots & \vdots\\
= & \det(\Db_{0})\prod_{\nu=1}^{n}\det(\Ib_{d}+\Mb_{\nu-1}^{-1/2}\Db_{\nu}\Mb_{\nu-1}^{-1/2}).
\end{aligned}\end{equation*}
Taking logarithm on both sides,
\[
\log\frac{\det(\Mb_{n})}{\det(\Db_{0})}=\sum_{\nu=1}^{n}\log\det(\Ib_{d}+\Mb_{\nu-1}^{-1/2}\Db_{\nu}\Mb_{\nu-1}^{-1/2})
\]
Let $\lambda_{\nu,1},\ldots,\lambda_{\nu,d}$ denote the eigenvalues
of $\Mb_{\nu-1}^{-1/2}\Db_{\nu}\Mb_{\nu-1}^{-1/2}$. For each $\nu\in[n]$,
because $\Db_{0}\succeq\sup_{\nu\ge1}\max\{1,\lambda_{\max}(\Db_{\nu})\}\succeq\Db_{\nu}$
\begin{equation*}\begin{aligned}
(\Mb_{\nu-1}^{-1/2}\Db_{\nu}\Mb_{\nu-1}^{-1/2})^{2}= & \Mb_{\nu-1}^{-1/2}\Db_{\nu}\Mb_{\nu-1}^{-1}\Db_{\nu}\Mb_{\nu-1}^{-1/2}\\
\preceq & \Mb_{\nu-1}^{-1/2}\Db_{\nu}\Db_{0}^{-1}\Db_{\nu}\Mb_{\nu-1}^{-1/2}\\
\preceq & \Mb_{\nu-1}^{-1/2}\Db_{\nu}\Db_{\nu}^{-1}\Db_{\nu}\Mb_{\nu-1}^{-1/2}\\
= & \Mb_{\nu-1}^{-1/2}\Db_{\nu}\Mb_{\nu-1}^{-1/2}.
\end{aligned}\end{equation*}
By the property of contraction matrix (Lemma \ref{lem:contraction_matrix}),
$\lambda_{\nu,i}\in[0,1]$ for $i\in[d]$. Thus,
\begin{equation*}\begin{aligned}
\log\frac{\det(\Mb_{n})}{\det(\Db_{0})}= & \sum_{\nu=1}^{n}\log\det(\Ib_{d}+\Mb_{\nu-1}^{-1/2}\Db_{\nu}\Mb_{\nu-1}^{-1/2})\\
= & \sum_{\nu=1}^{n}\log\prod_{i=1}^{d}(1+\lambda_{\nu,i})\\
= & \sum_{\nu=1}^{n}\sum_{i=1}^{d}\log(1+\lambda_{\nu,i})\\
\ge & \frac{1}{2}\sum_{\nu=1}^{n}\sum_{i=1}^{d}\lambda_{\nu,i}.
\end{aligned}\end{equation*}
where the last inequality holds because $\log(1+x)\ge x/2$ for $x\in[0,1]$.
Equating $\sum_{i=1}^{d}\lambda_{\nu,i}=\Tr(\Mb_{\nu-1}^{-1/2}\Db_{\nu}\Mb_{\nu-1}^{-1/2})$,
\[
\sum_{\nu=1}^{n}\Tr(\Mb_{\nu-1}^{-1/2}\Db_{\nu}\Mb_{\nu-1}^{-1/2})\le\log\frac{\det(\Mb_{n})}{\det(\Db_{0})}
\]
By GM-AM inequality, $\{\det(\Mb_{n})\}^{1/d}\le\Tr(\Mb_{n})/d$ and,
\[
\begin{aligned}
\log\frac{\det(\Mb_{n})}{\det(\Db_{0})}&=\log\det(\Db_0^{-1/2}\Mb_{n}\Db_0^{-1/2})\\
&\le d\log\frac{\Tr(\Db_0^{-1/2}\Mb_{n}\Db_0^{-1/2})}{d}.  
\end{aligned}
\]
Since $\Db_0 \succeq \Db_\nu$ for all $\nu\ge1$, 
\[
\begin{aligned}
\Tr(\Db_0^{-1/2}\Mb_{n}\Db_0^{-1/2}) &=d+\sum_{\nu=1}^{n}\Tr(\Db_0^{-1/2}\Db_{\nu}\Db_0^{-1/2}) \\
&\le d(n+1),
\end{aligned}
\]
which completes the proof.
\end{proof}

\subsection{Bounds for the Multinomial Transition Kernel Difference}

\label{subsec:2nd_diff} \begin{lemma}[The Second Order Difference
Bound for Transition Kernel Difference] \label{lem:diff_second}
(Second Order Difference between the two transition kernels) Let $\PP(\eb_{s^{\prime}}|s,a,\thetab)$
denote the multinomial transition kernel defined in \eqref{eq:transit_kernel}.
Then for any sequence $v_{s}\in\RR$ for $s\in\tilde{\Scal}\subseteq\Scal$
and $\thetab_{1},\thetab_{2}\in\RR^{d}$, 
\[
\begin{aligned} & \sum_{s^{\prime}\in\tilde{\Scal}}v_{s^{\prime}}\big(\PP(s^{\prime}|s,a,\thetab_{1})-\PP(s^{\prime}|s,a,\thetab_{2})\big)\\
 & \le\sum_{\tilde{s}\in\tilde{\Scal}}v_{\tilde{s}}\eb_{\tilde{s}}^{\top}\nabla^{2}L_{\tilde{\Scal}}\big(\vPhib(s,a)\thetab_{2}\big)\vPhib(s,a)(\thetab_{1}-\thetab_{2})\\
 & +\max_{\tilde{s}\in\tilde{\Scal}}|v_{\tilde{s}}|\max_{\tilde{s}\in\tilde{\Scal}}\big((\thetab_{1}-\thetab_{2})^{\top}\vPhib(s,a)^{\top}\eb_{\tilde{s}}\big)^{2}
\end{aligned}
\]
\end{lemma}

\textit{Proof.} For $\thetab\in\RR^{d}$ define the residual $\xib_{s^{\prime},\thetab}:=\eb_{s^{\prime}}-\EE[\eb_{s^{\prime}}|s,a,\thetab]$
of the multinomial distribution given $\thetab$. Applying the second
order Taylor expansion on 
\[
\PP(s^{\prime}|s,a,\thetab)=\exp\big(\eb_{s^{\prime}}^{\top}\vPhib(s,a)\thetab-g(\vPhib(s,a)\thetab)\big),
\]
for any $\thetab_{1},\thetab_{2}\in\RR^{d}$, there exists $v\in[0,1]$
such that $\bar{\thetab}:=v\thetab_{1}+(1-v)\thetab_{2}$ and 
\[
\begin{aligned} & \PP(s^{\prime}|s,a,\thetab_{1})-\PP(s^{\prime}|s,a,\thetab_{2})\\
 & =\xib_{s^{\prime},\thetab_{2}}^{\top}\vPhib(s,a)(\thetab_{1}-\thetab_{2})\PP(s^{\prime}|s,a,\thetab_{2})\\
 & +\frac{1}{2}(\thetab_{1}-\thetab_{2})^{\top}\vPhib(s,a)^{\top}\big\{\xib_{s^{\prime},\bar{\thetab}}\xib_{s^{\prime},\bar{\thetab}}^{\top}-\nabla^{2}L_{\nu,h}\big(\vPhib(s,a)^{\top}\bar{\thetab}\big)\big\}\vPhib(s,a)(\thetab_{1}-\thetab_{2})\PP(s^{\prime}|s,a,\bar{\thetab}).
\end{aligned}
\]
Summing up over $s^{\prime}\in\Scal$, 
\begin{small}
\[
\begin{split} 
& \sum_{s^{\prime}\in\Scal}v_{s^{\prime}}\big(\PP(s^{\prime}|s,a,\thetab_{1})-\PP(s^{\prime}|s,a,\thetab_{2})\big)\\
 & =\sum_{s^{\prime}\in\Scal}v_{s^{\prime}}\xib_{s^{\prime},\thetab_{2}}^{\top}\vPhib(s,a)(\thetab_{1}-\thetab_{2})\PP(s^{\prime}|s,a,\thetab_{2})\\
 & +\frac{1}{2}(\thetab_{1}-\thetab_{2})^{\top}\vPhib(s,a)^{\top}\bigg(\sum_{s^{\prime}\in\Scal}v_{s^{\prime}}\Big(\xib_{s^{\prime},\bar{\thetab}}\xib_{s^{\prime},\bar{\thetab}}^{\top}-v_{s^{\prime}}\nabla^{2}L_{\nu,h}\big(\vPhib(s,a)^{\top}\bar{\thetab}\big)\Big)\bigg)\vPhib(s,a)(\thetab_{1}-\thetab_{2})\PP(s^{\prime}|s,a,\bar{\thetab}).
\end{split}
\]
\end{small}
For the first term, because $\EE[\xib_{s^{\prime},\thetab}|s,a,\thetab]=\boldsymbol{0}$.
\[
\begin{split} & \sum_{s^{\prime}\in\Scal}v_{s^{\prime}}\xib_{s^{\prime},\thetab_{2}}^{\top}\vPhib(s,a)(\thetab_{1}-\thetab_{2})\PP(s^{\prime}|s,a,\thetab_{2})\\
 & =\sum_{\tilde{s}\in\Scal}v_{\tilde{s}}\eb_{\tilde{s}}^{\top}\EE\big[\eb_{s^{\prime}}\xib_{s^{\prime},\thetab_{2}}^{\top}|s,a,\thetab_{2}\big]\vPhib(s,a)(\thetab_{1}-\thetab_{2})\\
 & =\sum_{\tilde{s}\in\Scal}v_{\tilde{s}}\eb_{\tilde{s}}^{\top}\EE\big[\xib_{s^{\prime},\thetab_{2}}\xib_{s^{\prime},\thetab_{2}}^{\top}|s,a,\thetab_{2}\big]\vPhib(s,a)(\thetab_{1}-\thetab_{2})\\
 & =\sum_{\tilde{s}\in\Scal}v_{\tilde{s}}\eb_{\tilde{s}}^{\top}\nabla^{2}L_{\nu,h}\big(\vPhib(s,a)\thetab_{2}\big)\vPhib(s,a)(\thetab_{1}-\thetab_{2}).
\end{split}
\]
For the second term, because $\EE[\xib_{s^{\prime},\bar{\thetab}}\xib_{s^{\prime},\bar{\thetab}}^{\top}|s,a,\bar{\thetab}]-\nabla^{2}L_{\nu,h}\big(\vPhib(s,a)^{\top}\bar{\thetab}\big)=\Ob$,
\[
\begin{split} & \sum_{s^{\prime}\in\Scal}v_{s^{\prime}}\Big(\xib_{s^{\prime},\bar{\thetab}}\xib_{s^{\prime},\bar{\thetab}}^{\top}-\nabla^{2}L_{\nu,h}\big(\vPhib(s,a)^ {}\bar{\thetab}\big)\Big)\PP(s^{\prime}|s,a,\bar{\thetab})\\
 & =\sum_{\tilde{s}\in\Scal}v_{\tilde{s}}\eb_{\tilde{s}}^{\top}\EE\big[\eb_{s^{\prime}}\big(\xib_{s^{\prime},\bar{\thetab}}\xib_{s^{\prime},\bar{\thetab}}^{\top}-\nabla^{2}L_{\nu,h}\big(\vPhib(s,a)^ {}\bar{\thetab}\big)\big)\big|s,a,\bar{\thetab}\big]\\
 & =\sum_{\tilde{s}\in\Scal}v_{\tilde{s}}\eb_{\tilde{s}}^{\top}\EE\big[\xib_{s^{\prime},\bar{\thetab}}\big(\xib_{s^{\prime},\bar{\thetab}}\xib_{s^{\prime},\bar{\thetab}}^{\top}-\nabla^{2}L_{\nu,h}\big(\vPhib(s,a)^ {}\bar{\thetab}\big)\big)\big|s,a,\bar{\thetab}\big]\\
 & =\sum_{\tilde{s}\in\Scal}v_{\tilde{s}}\eb_{\tilde{s}}^{\top}\EE\big[\xib_{s^{\prime},\bar{\thetab}}\xib_{s^{\prime},\bar{\thetab}}\xib_{s^{\prime},\bar{\thetab}}^{\top}\big|s,a,\bar{\thetab}\big]\\
 & \preceq\max_{\tilde{s}\in\Scal}|v_{\tilde{s}}|\EE[\|\xib_{s^{\prime},\bar{\thetab}}\|_{1}\xib_{s^{\prime},\bar{\thetab}}\xib_{s^{\prime},\bar{\thetab}}^{\top}\big|s,a,\bar{\thetab}\big],
\end{split}
\]
where the last matrix ordering holds by H{ö}lder's inequality. Because
$\|\xib_{s^{\prime},\bar{\thetab}}\|_{1}\le2$ almost surely, 
\[
\begin{split} & \sum_{s^{\prime}\in\Scal}v_{s^{\prime}}\Big(\xib_{s^{\prime},\bar{\thetab}}\xib_{s^{\prime},\bar{\thetab}}^{\top}-\nabla^{2}L_{\nu,h}\big(\vPhib(s,a)^ {}\bar{\thetab}\big)\Big)\PP(s^{\prime}|s,a,\bar{\thetab})\\
 & \preceq2\max_{\tilde{s}\in\Scal}|v_{\tilde{s}}|\EE[\xib_{s^{\prime},\bar{\thetab}}\xib_{s^{\prime},\bar{\thetab}}^{\top}\big|s,a,\bar{\thetab}\big]\\
 & =2\max_{\tilde{s}\in\Scal}|v_{\tilde{s}}|\nabla^{2}L_{\nu,h}\big(\vPhib(s,a)\bar{\thetab}\big).
\end{split}
\]
Thus, 
\[
\begin{split} & \sum_{s^{\prime}\in\Scal}v_{s^{\prime}}\big(\PP(s^{\prime}|s,a,\thetab_{1})-\PP(s^{\prime}|s,a,\thetab_{2})\big)\\
 & \le\sum_{\tilde{s}\in\Scal}v_{\tilde{s}}\eb_{\tilde{s}}^{\top}\nabla^{2}L_{\nu,h}\big(\vPhib(s,a)\thetab_{2}\big)\vPhib(s,a)(\thetab_{1}-\thetab_{2})+\max_{\tilde{s}\in\Scal}|v_{\tilde{s}}|\big\Vert\vPhib(s,a)(\thetab_{1}-\thetab_{2})\big\Vert_{\nabla^{2}L_{\nu,h}\big(\vPhib(s,a)\bar{\thetab}\big)}^{2}.
\end{split}
\]
Because 
\[
\begin{aligned}\nabla^{2}L_{\nu,h}\big(\vPhib(s,a)\bar{\thetab}\big) & =\text{diag}\Big(\nabla L_{\nu,h}\big(\vPhib(s,a)\bar{\thetab}\big)\Big)-\nabla L_{\nu,h}\big(\vPhib(s,a)\bar{\thetab}\big)\nabla L_{\nu,h}\big(\vPhib(s,a)\bar{\thetab}\big)^{\top}\\
 & \preceq\text{diag}\Big(\nabla L_{\nu,h}\big(\vPhib(s,a)\bar{\thetab}\big)\Big),
\end{aligned}
\]
we obtain, 
\[
\begin{split} & \big\Vert\vPhib(s,a)(\thetab_{1}-\thetab_{2})\big\Vert_{\nabla^{2}L_{\nu,h}\big(\vPhib(s,a)\bar{\thetab}\big)}^{2}\\
 & =(\thetab_{1}-\thetab_{2})^{\top}\vPhib(s,a)^{\top}\nabla^{2}L_{\nu,h}\big(\vPhib(s,a)\bar{\thetab}\big)\vPhib(s,a)(\thetab_{1}-\thetab_{2})\\
 & =\sum_{s_{1},s_{2}\in\Scal}(\thetab_{1}-\thetab_{2})^{\top}\vPhib(s,a)^{\top}\eb_{s_{1}}\eb_{s_{1}}^{\top}\nabla^{2}L_{\nu,h}\big(\vPhib(s,a)\bar{\thetab}\big)\eb_{s_{2}}\eb_{s_{2}}^{\top}\vPhib(s,a)(\thetab_{1}-\thetab_{2})\\
 & \le\max_{s_{1}\in\Scal}\big((\thetab_{1}-\thetab_{2})^{\top}\vPhib(s,a)^{\top}\eb_{s_{1}}\big)^{2}\sum_{s_{1},s_{2}\in\Scal}|\eb_{s_{1}}^{\top}\nabla^{2}L_{\nu,h}\big(\vPhib(s,a)\bar{\thetab}\big)\eb_{s_{2}}|\\
 & \le\max_{s_{1}\in\Scal}\big((\thetab_{1}-\thetab_{2})^{\top}\vPhib(s,a)^{\top}\eb_{s_{1}}\big)^{2}\sum_{s_{1},s_{2}\in\Scal}\Big|\eb_{s_{1}}^{\top}\text{diag}\Big(\nabla L_{\nu,h}\big(\vPhib(s,a)\bar{\thetab}\big)\Big)\eb_{s_{2}}\Big|\\
 & =\max_{s_{1}\in\Scal}\big((\thetab_{1}-\thetab_{2})^{\top}\vPhib(s,a)^{\top}\eb_{s_{1}}\big)^{2}\norm{\nabla L_{\nu,h}\big(\vPhib(s,a)\bar{\thetab}\big)}_{1}\\
 & =\max_{s_{1}\in\Scal}\big((\thetab_{1}-\thetab_{2})^{\top}\vPhib(s,a)^{\top}\eb_{s_{1}}\big)^{2},
\end{split}
\]
which completes the proof. \hfill{}$\Box$

\begin{lemma} (First Order Difference between the two transition
kernels) \label{lem:diff_first} For $k\in[K]$ and $h\in[H]$, for
the estimator $\widehat{\thetab}_{k,h}$ , we have Then for any sequence
$v_{\tilde{s}}\in\RR$ for $\tilde{s}\in\Scal$ and $\thetab_{1},\thetab_{2}\in\RR^{d}$,
\[
\sum_{s^{\prime}\in\Scal}v_{s^{\prime}}\big(\PP(s^{\prime}|s,a,\thetab_{1})-\PP(s^{\prime}|s,a,\thetab_{2})\big)\le\frac{\max_{\tilde{s}\in\Scal}|v_{\tilde{s}}|}{2}\max_{\tilde{s}\in\Scal}(\thetab_{1}-\thetab_{2})^{\top}\vPhib(s,a)^{\top}\eb_{\tilde{s}}.
\]
\end{lemma}
\begin{proof}
\medspace{} By the definition of total variation, 
\[
\sum_{s^{\prime}\in\Scal}v_{s^{\prime}}\big(\PP(s^{\prime}|s,a,\thetab_{1})-\PP(s^{\prime}|s,a,\thetab_{2})\big)\le\max_{\tilde{s}\in\Scal}|v_{\tilde{s}}|\|\PP(\cdot|s,a,\thetab_{1})-\PP(\cdot|s,a,\thetab_{2})\|_{1}.
\]
Applying Pinsker's inequality, 
\[
\|\PP(\cdot|s,a,\thetab_{1})-\PP(\cdot|s,a,\thetab_{2})\|_{1}\le\sqrt{\frac{1}{2}KL\big(\PP(\cdot|s,a,\thetab_{1}),\PP(\cdot|s,a,\thetab_{2})\big)}.
\]
Because $\PP(\cdot|s,a,\thetab)$ is multinomial probability measure,
\[
\begin{split} & KL\big(\PP(\cdot|s,a,\thetab_{1}),\PP(\cdot|s,a,\thetab_{2})\big)\\
 & =\EE\big[\eb_{s^{\prime}}^{\top}\vPhib(s,a)(\thetab_{1}-\thetab_{2})\big|s,a,\thetab_{1}\big]-g\big(\vPhib(s,a)\thetab_{1}\big)+g\big(\vPhib(s,a)\thetab_{2}\big)\\
 & =\nabla L_{\nu,h}\big(\vPhib(s,a)\thetab_{1}\big)^{\top}\vPhib(s,a)(\thetab_{1}-\thetab_{2})-g\big(\vPhib(s,a)\thetab_{1}\big)+g\big(\vPhib(s,a)\thetab_{2}\big),
\end{split}
\]
where the last equality uses the identity for the mean of multinomial
(Proposition \ref{prop:mean_var}). By second order Taylor series,
there exists $\bar{\thetab}$ such that 
\[
\begin{split}
&KL\big(\PP(\cdot|s,a,\thetab_{1}),\PP(\cdot|s,a,\thetab_{2})\big)\\
& =\frac{1}{2}(\thetab_{1}-\thetab_{2})^{\top}\vPhib(s,a)^{\top}\nabla^{2}L_{\nu,h}(\vPhib(s,a)\bar{\thetab})\vPhib(s,a)(\thetab_{1}-\thetab_{2})\\
&\le  \frac{1}{2}(\thetab_{1}-\thetab_{2})^{\top}\vPhib(s,a)^{\top}\text{diag}\big(\nabla L_{\nu,h}(\vPhib(s_{k,h},a_{k,h})\bar{\thetab}\big)\vPhib(s,a)(\thetab_{1}-\thetab_{2})\\
&=  \frac{1}{2}\sum_{\tilde{s}\in\Scal}\eb_{\tilde{s}}^{\top}\nabla L_{\nu,h}(\vPhib(s,a)\bar{\thetab})\big((\thetab_{1}-\thetab_{2})^{\top}\vPhib(s,a)^{\top}\eb_{\tilde{s}}\big)^{2}
\end{split}
\]
where the first inequality uses the fact that $\nabla^{2}L_{\nu,h}(\etab)=\text{diag}\big(\nabla L_{\nu,h}(\etab)\big)-\nabla L_{\nu,h}(\etab)\nabla L_{\nu,h}(\etab)^{\top}\preceq\text{diag}\big(\nabla L_{\nu,h}(\etab)\big)$.
Because $\|\nabla L_{\nu,h}(\vPhib(s,a)\bar{\thetab})\|_{1}=1$,by
H{ö}lder's inequality, 
\[
\begin{split} & \frac{1}{2}\sum_{\tilde{s}\in\Scal}\eb_{\tilde{s}}^{\top}\nabla L_{\nu,h}(\vPhib(s,a)\bar{\thetab})\big((\thetab_{1}-\thetab_{2})^{\top}\vPhib(s,a)^{\top}\eb_{\tilde{s}}\big)^{2}\\
 & \le\frac{1}{2}\max_{\tilde{s}\in\Scal}\big((\thetab_{1}-\thetab_{2})^{\top}\vPhib(s,a)^{\top}\eb_{\tilde{s}}\big)^{2}\|\nabla L_{\nu,h}(\vPhib(s,a)\bar{\thetab})\|_{1}\\
 & =\frac{1}{2}\max_{\tilde{s}\in\Scal}\big((\thetab_{1}-\thetab_{2})^{\top}\vPhib(s,a)^{\top}\eb_{\tilde{s}}\big)^{2}.
\end{split}
\]
which completes the proof. 
\end{proof}

\begin{proof}
\medspace{} For the first term, 
\[
\begin{split} & \norm{\vPhib_{k,h}^{\top}\nabla^{2}L_{\nu,h}\big(\vPhib_{k,h}\thetab_{h}^{\star}\big)\Big(\sum_{\tilde{s}\in\Scal_{k,h}}\widehat{V}_{k,h+1}(\tilde{s})\eb_{\tilde{s}}\Big)}_{\widehat{\Hb}_{k-1,h}^{-1}}\\
 & \le H\max_{\xb:\|\xb\|_{\infty}\le1}\norm{\vPhib_{k,h}\nabla^{2}L_{\nu,h}\big(\vPhib_{k,h}\thetab_{h}^{\star}\big)\xb}_{\widehat{\Hb}_{k-1,h}^{-1}}\\
 & \le H\|\widehat{\Hb}_{k-1,h}^{-1/2}\vPhib_{k,h}\big(\nabla^{2}L_{k,h}(\vPhib_{k,h}\thetab_{h}^{\star})\big)^{1/2}\|_{2}\max_{\xb:\|\xb\|_{\infty}\le1}\|\Big(\nabla^{2}L_{k,h}(\vPhib_{k,h}\thetab_{h}^{\star})\Big)^{1/2}\xb\|_{2}\\
 & =H\sigma_{k,h}\|\widehat{\Hb}_{k-1,h}^{-1/2}\vPhib_{k,h}\big(\nabla^{2}L_{k,h}(\vPhib_{k,h}\thetab_{h}^{\star})\big)^{1/2}\|_{2}\\
 & \le H\sigma_{k,h}\|\widehat{\Hb}_{k-1,h}^{-1/2}\vPhib_{k,h}\big(\nabla^{2}L_{k,h}(\vPhib_{k,h}\thetab_{h}^{\star})\big)^{1/2}\|_{F}\\
 & =H\sigma_{k,h}\sqrt{\Tr\Big(\nabla^{2}\ell_{k,h}(\thetab_{h}^{\star})\widehat{\Hb}_{k-1,h}^{-1}\Big)}
\end{split}
\]
where the last equality uses the fact that $\nabla^{2}\ell_{k,h}(\thetab_{h}^{\star})=\vPhib_{k,h}^{\top}\nabla^{2}L_{k,h}(\vPhib_{k,h}\thetab_{h}^{\star})\vPhib_{k,h}$.
By Lemma \ref{lem:Hessian_bound}, we have $\widehat{\Hb}_{k,h}\succeq\Hb_{k,h}^{\star}:=\sum_{\nu=1}^{k}\nabla^{2}\ell_{\nu,h}(\thetab_{h}^{\star})+4B_{\vPhib}^{2}\Ib_{d}$
and 
\[
\norm{\sum_{\tilde{s}\in\Scal_{k,h}}\widehat{V}_{k,h+1}^ {}(\tilde{s})\vPhib_{k,h}^{\top}\nabla^{2}L_{\nu,h}\big(\vPhib_{k,h}\thetab_{h}^{\star}\big)\eb_{\tilde{s}}}_{\widehat{\Hb}_{k-1,h}^{-1}}\le H\sqrt{\sigma_{k,h}^{2}\Tr\Big(\nabla^{2}\ell(\thetab_{h}^{\star})\big(\Hb_{k-1,h}^{\star}\big)^{-1}\Big)}.
\]
\end{proof}

\subsection{Matrix Algebras}

\begin{proposition} \label{prop:inv_matrix} (Inverse Matrix Differentiation)
Let $\Fb:\RR^{d}\to\RR^{d\times d}$ denote an invertible matrix map.
Then for any $\xb\in\RR^{d}$ and $i\in[d]$, 
\[
D_{i}\big\{\Fb(\xb)^{-1}\big\}=-\Fb(\xb)^{-1}D_{i}\Fb(\xb)\Fb(\xb)^{-1}.
\]
\end{proposition}
\begin{proof}
\medspace{} Note that $D_{i}\{\Fb(\xb)^{-1}\Fb(\xb)\}=\Ob$. By the
chain rule, 
\[
D_{i}\big\{\Fb(\xb)^{-1}\big\}\Fb(\xb)+\Fb(\xb)^{-1}D_{i}\Fb(\xb)=\Ob,
\]
and rearranging the term completes the proof. 
\end{proof}
\begin{lemma}[A Property of a contraction
matrix]
\label{lem:contraction_matrix} Let $\Mb\in\RR^{d\times d}$ denote a positive definite symmetric matrix such that $\Mb^{2}\preceq\Mb$. Then the eigenvalues $\lambda_{1},\ldots,\lambda_{d}$
of $\Mb$ are in $[0,1]$. \end{lemma}
\begin{proof}
\medspace{} Since $\Mb^{-1/2}$ exists, we have $\Mb^{-1/2}\Mb^{2}\Mb^{-1/2}\preceq\Mb^{-1/2}\Mb\Mb^{-1/2}$
and thus $\Mb\preceq\Ib_{d}$ which completes the proof. 
\end{proof}
\begin{lemma}[A matrix norm bound] \label{lem:matrix_norm_bound}
Let $\Ab\in\RR^{d\times d}$ denote a positive definite symmetric
matrix and $\Bb\in\RR^{d\times s}$ denote an arbitral matrix. Then,
\[
\norm{\Ab^{-1/2}\Bb}_{1,2}:=\sup_{\xb:\|\xb\|_{1}\le1}\norm{\Ab^{-1/2}\Bb\xb}_{2}\le\max_{s^{\prime}\in[s]}\norm{\Ab^{-1/2}\Bb\eb_{s^{\prime}}}_{2}
\]
\end{lemma}
\begin{proof}
\medspace{} Writing $\xb=(x_{1},\ldots,x_{s})$, 
\[
\begin{split}\sup_{\xb:\|\xb\|_{1}\le1}\norm{\Ab^{-1/2}\Bb\xb}_{2}^{2}= & \sup_{\xb:\|\xb\|_{1}\le1}\xb^{\top}\Bb^{\top}\Ab^{-1}\Bb\xb\\
= & \sup_{\xb:\|\xb\|_{1}\le1}\sum_{s_{1},s_{2}\in[s]}x_{s_{1}}x_{s_{2}}\eb_{s_{1}}^{\top}\Bb^{\top}\Ab^{-1}\Bb\eb_{s_{2}}^{\top}\\
\le & \max_{s_{1},s_{2}\in[s]}\abs{\eb_{s_{1}}^{\top}\Bb^{\top}\Ab^{-1}\Bb\eb_{s_{2}}^{\top}}\sup_{\xb:\|\xb\|_{1}\le1}\sum_{s_{1},s_{2}\in[s]}|x_{s_{1}}x_{s_{2}}|\\
\le & \max_{s_{1},s_{2}\in[s]}\abs{\eb_{s_{1}}^{\top}\Bb^{\top}\Ab^{-1}\Bb\eb_{s_{2}}^{\top}}.
\end{split}
\]
By Cauchy-Schwarz inequality, 
\[
\begin{split}\max_{s_{1},s_{2}\in[s]}\abs{\eb_{s_{1}}^{\top}\Bb^{\top}\Ab^{-1}\Bb\eb_{s_{2}}^{\top}}\le & \max_{s_{1},s_{2}\in[s]}\norm{\Ab^{-1/2}\Bb\eb_{s_{1}}}_{2}\norm{\Ab^{-1/2}\Bb\eb_{s_{2}}}_{2}\\
= & \max_{s^{\prime}\in[s]}\norm{\Ab^{-1/2}\Bb\eb_{s^{\prime}}}_{2}^{2},
\end{split}
\]
which completes the proof. 
\end{proof}

\subsection{Divergence Between the Two Interaction Measures}

\begin{lemma} \label{lem:KL} (Kullback-Liebler Divergence Between
the Two Interaction Measures) Given an MNL-MDP instance $\Mcal(\Scal,\Acal,H,\Thetab,\vPhib,r)$
and an algorithm $\Pi_{K}$, let 
\[
\begin{split}\yb_{k,h} & :=(s_{1,1,}a_{1,1},s_{1,2},a_{1,2},\ldots,s_{1,H+1},s_{2,1},a_{2,1},\ldots,s_{k,h},a_{k,h}),\\
\zb_{k,h} & :=(s_{1,1,}a_{1,1},s_{1,2},a_{1,2},\ldots,s_{1,H+1},s_{2,1},a_{2,1},\ldots,s_{k,h}),
\end{split}
\]
denote the sequence of action-state random variables generated by
the interaction measure $\PP_{\Thetab,\Pi_{K}}$. Then for any parameters
$\Thetab_{1},\Thetab_{2}\in\{(\thetab_{1},\ldots,\thetab_{h})\in\RR^{d\times H}:\max_{h\in[H]}\|\thetab_{h}\|_{2}\le B_{\thetab}\}$,
the Kullback-Leibler divergence, 
\[
\begin{split} & KL(\PP_{\Thetab_{1},\Pi_{K}},\PP_{\Thetab_{2},\Pi_{K}})\\
 & \le\frac{\EE_{\Thetab_{1},\Pi_{K}}\big[\sum_{k=1}^{K}\sum_{h=1}^{H}\|\vPhib(s_{k,h},a_{k,h})\big(\Thetab_{1}-\Thetab_{2}\big)\eb_{h}\|_{\infty}^{2}\sigma_{k,h}^{2}]}{2}\\
 & \quad+\frac{\EE_{\Thetab_{1},\Pi_{K}}\big[\sum_{k=1}^{K}\sum_{h=1}^{H}\|\vPhib(s_{k,h},a_{k,h})\big(\Thetab_{1}-\Thetab_{2}\big)\eb_{h}\|_{\infty}^{3}\big]}{\sqrt{6\kappa}},
\end{split}
\]
where $\sigma_{k,h}^{2}:=\max_{\xb:\|\xb\|_{\infty}\le1}\xb^{\top}\nabla^{2}L_{k,h}\big(\vPhib(s_{k,h},a_{k,h})\Thetab_{1}\eb_{H}\big)\xb$
and $L_{k,h}(\ub):=\log\big(\sum_{s\in\Scal_{h}(s_{k,h},a_{k,h})}\exp(\eb_{s}^{\top}\ub)\big)$
is the log-sum function. \end{lemma}
\begin{proof}
By definition of $KL(\cdot,\cdot)$, 
\begin{small}
\[
\begin{split} & KL(\PP_{\Thetab_{1},\Pi_{K}},\PP_{\Thetab_{2},\Pi_{K}})\\
 & =\int\!\log\frac{d\PP_{\Thetab_{1},\Pi_{K}}(\zb_{K,H+1})}{d\PP_{\Thetab_{2},\Pi_{K}}(\zb_{K,H+1})}d\PP_{\Thetab_{1},\Pi_{K}}(\zb_{K,H+1})\\
 & =\int\!\Big(\!\int\!\log\frac{d\PP_{\Thetab_{1},\Pi_{K}}(s_{K,H+1}|\yb_{K,H})}{d\PP_{\Thetab_{2},\Pi_{K}}(s_{K,H+1}|\yb_{K,H})}\!+\!\log\frac{d\PP_{\Thetab_{1},\Pi_{K}}(\yb_{K,H})}{d\PP_{\Thetab_{2},\Pi_{K}}(\yb_{K,H})}d\PP_{\Thetab_{1},\Pi_{K}}(s_{K,H+1}|\yb_{K,H})\Big)d\PP_{\Thetab_{1},\Pi_{K}}(\yb_{K,H}).
\end{split}
\]
\end{small}
For each $k\in[K]$ and $h\in[H]$, we write the log-sum function
$L_{k,h}:=L_{\Scal_{h}(s_{k,h},a_{k,h})}$ and the feature $\vPhib_{k,h}:=\vPhib(s_{k,h},a_{k,h})$.
Because $\PP_{\Thetab_{1},\Pi_{K}}(s_{k,h+1}|\yb_{k,h})=\PP_{\Thetab_{1},\Pi_{K}}(s_{k,h+1}|s_{k,h},a_{k,h})$
is a multinomial distribution, 
\[
\begin{split} & \int\log\frac{d\PP_{\Thetab_{1},\Pi_{K}}(s_{K,H+1}|\yb_{K,H})}{d\PP_{\Thetab_{2},\Pi_{K}}(s_{K,H+1}|\yb_{K,H})}d\PP_{\Thetab_{1},\Pi_{K}}(s_{K,H+1}|\yb_{K,H})\\
 & =\EE_{\Thetab_{1},\Pi_{K}}[\eb_{s_{K,H+1}}]^{\top}\vPhib_{K,H}\big(\Thetab_{1}-\Thetab_{2})\eb_{H}+L_{K,H}\big(\vPhib_{K,H}\Thetab_{2}\eb_{H}\big)-L_{K,H}\big(\vPhib_{K,H}\Thetab_{1}\eb_{H}\big)\\
 & =\nabla L_{\nu,h}\big(\vPhib_{K,H}\Thetab_{1}\eb_{H}\big)^{\top}\vPhib_{K,H}\big(\Thetab_{1}-\Thetab_{2})\eb_{H}+L_{K,H}\big(\vPhib_{K,H}\Thetab_{2}\eb_{H}\big)-L_{K,H}\big(\vPhib_{K,H}\Thetab_{1}\eb_{H}\big)
\end{split}
\]
where the last equality uses Proposition \ref{prop:mean_var}. By
the third order Taylor expansion, there exists $\bar{\thetab}_{H}\in\{\thetab\in\RR^{d}:\|\thetab\|_{2}\le B_{\thetab}\}$
such that 
\[
\begin{split} & \nabla L_{\nu,h}\big(\vPhib_{K,H}\Thetab_{1}\eb_{H}\big)^{\top}\vPhib_{K,H}\big(\Thetab_{1}-\Thetab_{2})\eb_{H}+L_{K,H}\big(\vPhib_{K,H}\Thetab_{2}\eb_{H}\big)-L_{K,H}\big(\vPhib_{K,H}\Thetab_{1}\eb_{H}\big)\\
 & =\frac{1}{2}\Big\Vert\vPhib_{K,H}\big(\Thetab_{1}-\Thetab_{2})\eb_{H}\Big\Vert_{\nabla^{2}L_{K,H}\big(\vPhib_{K,H}\Thetab_{1}\eb_{H}\big)}^{2}\\
 & +\frac{1}{6}\sum_{s\in\Scal_{K,H}}\eb_{s}^{\top}\vPhib_{K,H}\big(\Thetab_{1}-\Thetab_{2}\big)\eb_{H}\Big\Vert\vPhib_{K,H}\big(\Thetab_{1}-\Thetab_{2}\big)\eb_{H}\Big\Vert_{D_{s}\nabla^{2}L_{K,H}(\vPhib_{K,H}\bar{\thetab}_{H}\big)}^{2}.
\end{split}
\]
In the first term, 
\[
\begin{split} & \Big\Vert\vPhib_{K,H}\big(\Thetab_{1}-\Thetab_{2})\eb_{H}\Big\Vert_{\nabla^{2}L_{K,H}\big(\vPhib_{K,H}\Thetab_{1}\eb_{H}\big)}^{2}\\
 & \le\Big(\vPhib_{K,H}\big(\Thetab_{1}-\Thetab_{2})\eb_{H}\Big)^{\top}\nabla^{2}L_{K,H}\big(\vPhib_{K,H}\Thetab_{1}\eb_{H}\big)\Big(\vPhib_{K,H}\big(\Thetab_{1}-\Thetab_{2})\eb_{H}\Big)\\
 & \le\norm{\vPhib_{K,H}\big(\Thetab_{1}-\Thetab_{2})\eb_{H}}_{\infty}^{2}\Ubrace{\max_{\xb:\|\xb\|_{\infty}\le1}\xb^{\top}\nabla^{2}L_{K,H}\big(\vPhib_{K,H}\Thetab_{1}\eb_{H}\big)\xb}{\sigma_{K,H}^{2}}
\end{split}
\]
Because $L_{K,H}$ is $\sqrt{6}$-self concordant, the third order
term, 
\[
\begin{split} & \sum_{s\in\Scal_{K,H}}\eb_{s}^{\top}\vPhib_{K,H}\big(\Thetab_{1}-\Thetab_{2}\big)\eb_{H}\Big\Vert\vPhib_{K,H}\big(\Thetab_{1}-\Thetab_{2}\big)\eb_{H}\Big\Vert_{D_{s}\nabla^{2}L_{K,H}(\vPhib_{K,H}\bar{\thetab}_{H}\big)}^{2}\\
 & \le\sqrt{6}\|\vPhib_{K,H}\big(\Thetab_{1}-\Thetab_{2}\big)\eb_{H}\|_{2}\|\vPhib_{K,H}\big(\Thetab_{1}-\Thetab_{2}\big)\eb_{H}\|_{\nabla^{2}L_{K,H}(\vPhib_{K,H}\bar{\thetab}_{H}\big)}^{2}\\
 & \le\sqrt{6\kappa^{-1}}\|\vPhib_{K,H}\big(\Thetab_{1}-\Thetab_{2}\big)\eb_{H}\|_{\nabla^{2}L_{K,H}(\vPhib_{K,H}\bar{\thetab}_{H}\big)}^{3},
\end{split}
\]
where the last inequality uses Definition \ref{def:kappa}. 
Because
$\max_{\xb:\|\xb\|_{\infty}\le1}\xb^{\top}\nabla^{2}L_{\nu,h}\big(\vPhib(s_{K,H},a_{K,H})\bar{\thetab}_{H}\big)\xb\le1$,
\[
\begin{split} & \sqrt{6\kappa^{-1}}\|\vPhib_{K,H}\big(\Thetab_{1}-\Thetab_{2}\big)\eb_{H}\|_{\nabla^{2}L_{K,H}(\vPhib_{K,H}\bar{\thetab}_{H}\big)}^{3}\\
 & \le\sqrt{6\kappa^{-1}}\|\vPhib_{K,H}\big(\Thetab_{1}-\Thetab_{2}\big)\eb_{H}\|_{\infty}^{3}\bigg(\max_{\xb:\|\xb\|_{\infty}\le1}\xb^{\top}\nabla^{2}L_{\nu,h}\big(\vPhib(s_{K,H},a_{K,H})\bar{\thetab}_{H}\big)\xb\bigg)^{3/2}\\
 & \le\sqrt{6\kappa^{-1}}\|\vPhib_{K,H}\big(\Thetab_{1}-\Thetab_{2}\big)\eb_{H}\|_{\infty}^{3}.
\end{split}
\]
Recall that $\Delta_{k,h}:=\|\vPhib_{k,h}\big(\Thetab_{1}-\Thetab_{2}\big)\eb_{h}\|_{\infty}.$
Gathering the terms, 
\[
\int\log\frac{d\PP_{\tilde{\Ub}}(s_{K,H+1}|\yb_{K,H})}{d\PP_{\bar{\Ub}(h,i)}(s_{K,H+1}|\yb_{K,H})}d\PP_{\tilde{\Ub}}(s_{K,H+1}|\yb_{K,H})\le\frac{\Delta_{K,H}^{2}\sigma_{K,H}^{2}}{2}+\frac{\Delta_{K,H}^{3}}{\sqrt{6\kappa}}.
\]
Thus, 
\[
\begin{split} 
& KL(\PP_{\Thetab_{1},\Pi_{K}},\PP_{\Thetab_{2},\Pi_{K}})\\
 & \le\!\int\!\!\Big(\frac{\Delta_{K,H}^{2}\sigma_{K,H}^{2}}{2}\!+\!\frac{\Delta_{K,H}^{3}}{\sqrt{6\kappa}}\!+\!\int\log\frac{d\PP_{\Thetab_{1},\Pi_{K}}(\yb_{K,H})}{d\PP_{\Thetab_{2},\Pi_{K}}(\yb_{K,H})}d\PP_{\Thetab_{1},\Pi_{K}}(s_{K,H+1}|\yb_{K,H})\!\Big)d\PP_{\Thetab_{1},\Pi_{K}}(\yb_{K,H})\\
 & \le\frac{\EE_{\Thetab_{1},\Pi_{K}}\big[\Delta_{K,H}^{2}\sigma_{K,H}^{2}]}{2}+\frac{\EE_{\Thetab_{1},\Pi_{K}}\big[\Delta_{K,H}^{3}\big]}{\sqrt{6\kappa}}+\int\log\frac{d\PP_{\Thetab_{1},\Pi_{K}}(\yb_{K,H})}{d\PP_{\Thetab_{2},\Pi_{K}}(\yb_{K,H})}d\PP_{\Thetab_{1},\Pi_{K}}(\yb_{K,H}).
\end{split}
\]
Note that the last term, 
\[
\begin{split} & \int\log\frac{d\PP_{\Thetab_{1},\Pi_{K}}(\yb_{K,H})}{d\PP_{\Thetab_{2},\Pi_{K}}(\yb_{K,H})}d\PP_{\Thetab_{1},\Pi_{K}}(\yb_{K,H})\\
 & =\int\int\log\frac{d\pi_{K,H}(a_{K,H}|\zb_{K,H})}{d\pi_{K,H}(a_{K,H}|\zb_{K,H})}+\log\frac{d\PP_{\Thetab_{1},\Pi_{K}}(\zb_{K,H})}{d\PP_{\Thetab_{2},\Pi_{K}}(\zb_{K,H})}d\pi_{K,H}(a_{K,H}|\zb_{K,H})d\PP_{\Thetab_{1},\Pi_{K}}(\zb_{K,H})\\
= & \int\log\frac{d\PP_{\Thetab_{1},\Pi_{K}}(\zb_{K,H})}{d\PP_{\Thetab_{2},\Pi_{K}}(\zb_{K,H})}\Big(\int d\pi_{K,H}(a_{K,H}|\zb_{K,H})\Big)d\PP_{\Thetab_{1},\Pi_{K}}(\zb_{K,H})\\
= & \int\log\frac{d\PP_{\Thetab_{1},\Pi_{K}}(\zb_{K,H})}{d\PP_{\Thetab_{2},\Pi_{K}}(\zb_{K,H})}d\PP_{\Thetab_{1},\Pi_{K}}(\zb_{K,H})
\end{split}
\]
Recursively, 
\[
\begin{split} & KL(\PP_{\Thetab_{1},\Pi_{K}},\PP_{\Thetab_{2},\Pi_{K}})\\
 & =\int\log\frac{d\PP_{\Thetab_{1},\Pi_{K}}(\zb_{K,H+1})}{d\PP_{\Thetab_{2},\Pi_{K}}(\zb_{K,H+1})}d\PP_{\Thetab_{1},\Pi_{K}}(\zb_{K,H+1})\\
 & \le\frac{\EE_{\Thetab_{1},\Pi_{K}}\big[\Delta_{K,H}^{2}\sigma_{K,H}^{2}]}{2}+\frac{\EE_{\Thetab_{1},\Pi_{K}}\big[\Delta_{K,H}^{3}\big]}{\sqrt{6\kappa}}+\int\log\frac{d\PP_{\Thetab_{1},\Pi_{K}}(\zb_{K,H})}{d\PP_{\Thetab_{2},\Pi_{K}}(\zb_{K,H})}d\PP_{\Thetab_{1},\Pi_{K}}(\zb_{K,H})\\
 & \le\frac{\EE_{\Thetab_{1},\Pi_{K}}\big[\Delta_{K,H-1}^{2}\sigma_{K,H-1}^{2}+\Delta_{K,H}\sigma_{K,H}^{2}]}{2}+\frac{\EE_{\Thetab_{1},\Pi_{K}}\big[\Delta_{K,H-1}^{3}+\Delta_{K,H}^{3}\big]}{\sqrt{6\kappa}}\\
 & \quad+\int\log\frac{d\PP_{\Thetab_{1},\Pi_{K}}(\zb_{K,H-1})}{d\PP_{\Thetab_{2},\Pi_{K}}(\zb_{K,H-1})}d\PP_{\Thetab_{1},\Pi_{K}}(\zb_{K,H-1})\\
 & \vdots\\
 & \le\frac{\EE_{\Thetab_{1},\Pi_{K}}\big[\sum_{k=1}^{K}\sum_{h=1}^{H}\Delta_{k,h}^{2}\sigma_{k,h}^{2}]}{2}+\frac{\EE_{\Thetab_{1},\Pi_{K}}\big[\sum_{k=1}^{K}\sum_{h=1}^{H}\Delta_{k,h}^{3}\big]}{\sqrt{6\kappa}},
\end{split}
\]
which completes the proof. 
\end{proof}

\subsection{Covariance Inequality}

\begin{proposition} \label{prop:var_inequality}(Covariance inequality)
Let $\yb\in\RR^{d}$ denote a vector valued random variable. Then
for any deterministic vector $\ab\in\RR^{d}$ 
\[
\EE\Big[\big(\yb-\EE[\yb]\big)\big(\yb-\EE[\yb]\big)^{\top}\Big]\preceq\EE\Big[\big(\yb-\ab\big)\big(\yb-\ab\big)^{\top}\Big].
\]
\end{proposition}
\begin{proof}
\medspace{} For any $\ab\in\RR^{d}$, 
\[
\begin{split}\EE\Big[\big(\yb-\ab\big)\big(\yb-\ab\big)^{\top}\Big]-\EE\Big[\big(\yb-\EE[\yb]\big)\big(\yb-\EE[\yb]\big)^{\top}\Big]= & \ab\ab^{\top}-\EE[\yb]\ab^{\top}-\ab\EE[\yb]^{\top}+\EE[\yb]\EE[\yb]^{\top}\\
= & \big(\ab-\EE[\yb]\big)\big(\ab-\EE[\yb]\big)^{\top}\\
\succeq & \Ob,
\end{split}
\]
which completes the proof. 
\end{proof}

\subsection{Entropic Variational Principle}
\begin{proposition}[Entropic Variational Principle]
\label{lem:variation_principle}
Let $(\Omega, \mathcal{F}, P)$ be a probability space and $X: \Omega \to \mathbb{R}$ 
be a measurable function such that $\mathbb{E}_P[e^X] < \infty$. Let $\mathcal{Q}$ 
be the set of all probability measures on $(\Omega, \mathcal{F})$ that are 
absolutely continuous with respect to $P$. Then,
\begin{equation}
    \log \mathbb{E}_P[e^X] = \sup_{Q \in \mathcal{Q}} \left\{ \mathbb{E}_Q[X] - D_{KL}(Q \| P) \right\}
\end{equation}
where $KL(Q \| P)$ denotes the Kullback-Leibler divergence. Furthermore, 
the supremum is uniquely attained by the Gibbs measure $P^*$ defined by the 
Radon-Nikodym derivative:
\begin{equation}
    \frac{dP^*}{dP} = \frac{e^X}{\mathbb{E}_P[e^X]}
\end{equation}
\end{proposition}


\end{document}